\documentclass[journal, two column, 10pt]{IEEEtran}
\usepackage{amsmath,amssymb}
\usepackage{color}
\usepackage{flushend}
\usepackage[T1]{fontenc}
\usepackage{mathptmx}
\usepackage[mathscr]{euscript}
\usepackage[ruled,linesnumbered]{algorithm2e}
\usepackage[english]{babel}
\usepackage{amsthm}
\usepackage[font=footnotesize, labelfont=bf, justification=centering]{caption}
\usepackage[pdftex]{graphicx}
\usepackage{epstopdf}
\usepackage{cases}
\usepackage{amsmath}
\usepackage{mathtools, cuted}
\usepackage{ulem}

\usepackage{multirow}
\usepackage{booktabs}
\usepackage{balance}
\usepackage{subfig}
\usepackage{empheq}
\newcommand*\widefbox[1]{\fbox{\hspace{1em}#1\hspace{1em}}}
\usepackage{makecell, booktabs}
\usepackage{eqparbox}
\usepackage{threeparttable,booktabs}
\usepackage{etoolbox}
\appto\TPTnoteSettings{\small}
\usepackage{nomencl}
\makenomenclature

\newtheorem{thm}{Theorem}
\newtheorem{cor}{Corollary}
\newtheorem{lem}{Lemma}

\newtheorem{defn}{Definition}

\newtheorem{rem}{Remark}

\newcommand{\abs}[1]{\left\vert#1\right\vert}

\newcommand{\norm}[1]{\left\Vert#1\right\Vert}

\newcommand{\Mod}[1]{\ (\mathrm{mod}\ #1)}

\newcommand{\blue}{\color{blue}}

\usepackage{threeparttable,booktabs}
\usepackage{etoolbox}
\appto\TPTnoteSettings{\small}
\DeclareMathOperator{\atantwo}{atan2}

\begin{document}

\title{\small {\blue Robotics and Autonomous Systems}\\
\Huge{Rapid Path Planning for Dubins Vehicles\\ under Environmental Currents}}
\author{ \begin{tabular}{cccccccccc}
 \textbf{Khushboo Mittal${^\dag}{^\ast}$} & \ \ \  \textbf{Junnan Song${^\dag}{^\ast}$} & \ \ \  \textbf{Shalabh Gupta${^\dag}{^\star}$} & \ \ \ \textbf{Thomas A. Wettergren$^{\ddag}$}
\end{tabular}
\thanks{This work was supported by US Office of Naval Research under Award Number N000141613032. Any opinions or findings herein are those of the authors and do not necessarily reflect the views of the sponsoring agencies.}
\thanks{$^{\ast}$ These authors contributed equally to this work.}
\thanks{$^{\dag}$ Department of Electrical and Computer Engineering, University of Connecticut, Storrs, CT 06269, USA.}\
\thanks{$^{\ddag}$ Naval Undersea Warfare Center, Newport, RI 02841, USA.}\
\thanks{$^{\star}$ Corresponding Author (email id: shalabh.gupta@uconn.edu)}
\vspace{-18pt}
}

\maketitle
\begin{abstract}
This paper presents a rapid (real time) solution to the minimum-time path planning problem for Dubins vehicles under environmental currents (wind or ocean currents). Real-time solutions are essential in time-critical situations (such as replanning under dynamically changing environments or tracking fast moving targets). Typically, Dubins problem requires to solve for six  path types; however, due to the presence of currents, four of these path types require to solve the root-finding problem involving transcendental functions. Thus, the existing methods result in high computation times and their applicability for real-time applications is limited. In this regard, in order to obtain a real-time solution, this paper proposes a novel approach where only a subset of two Dubins path types ($LSL$ and $RSR$) are used which have direct analytical solutions in the presence of currents. However, these two path types do not provide full reachability. We show that by extending the feasible range of circular arcs in the $LSL$ and $RSR$ path types from $2\pi$ to $4\pi$: 1) full reachability of any goal pose is guaranteed, and 2) paths with lower time costs as compared to the corresponding $2\pi$-arc paths can be produced. Theoretical properties are rigorously established, supported by several examples, and evaluated in comparison to the Dubins solutions by extensive Monte-Carlo simulations.
\end{abstract}
\vspace{-3pt}

\begin{IEEEkeywords}
Dubins paths, path planning, environmental currents, curvature-constrained vehicles.
\end{IEEEkeywords}
\vspace{-3pt}

\thispagestyle{empty}
\vspace{0pt}

\nomenclature[1]{$\bf v$}{Vehicle velocity vector (m/s,~m/s)}
\nomenclature[2]{${\bf v}_w$}{Current velocity vector (m/s,~m/s)}
\nomenclature[3]{${\bf v}_{net}$}{Net velocity vector which is the vector sum of $\bf v$ and ${\bf v}_w$ (m/s,~m/s)}
\nomenclature[4]{$w_x$, $w_y$}{Components of ${\bf v}_w$ along the x and y axes (m/s)}
\nomenclature[5]{$\theta$}{Vehicle heading (rad)}
\nomenclature[6]{$\theta_w$}{Current heading (rad)}
\nomenclature[7]{$u$}{Turn rate (rad/s)}
\nomenclature[8]{$r$}{Turning radius (m)}
\nomenclature[9]{$(x_0,y_0,\theta_0)$}{Start pose (m,~m,~rad)}
\nomenclature[91]{$(x_f,y_f,\theta_f)$}{Goal pose (m,~m,~rad)}
\nomenclature[92]{$k$}{Feasible range parameter for an $LSL$ or $RSR$ path (-)}
\nomenclature{$\alpha$}{Turning angle of the first arc of an $LSL$ or $RSR$ path (rad)}
\nomenclature{$\beta$}{Length of the straight line segment of an $LSL$ or $RSR$ path (m)}
\nomenclature{$\gamma$}{Turning angle of the last arc of an $LSL$ or $RSR$ path (rad)}
\nomenclature{$T$}{Total travel time (s)}
\nomenclature{$(p^k_{LSL}, q^k_{LSL})$}{Center of rotation of the reachability ray for an $LSL$ path and a given $k$ (m,~m)}
\nomenclature{$(p^k_{RSR}, q^k_{RSR})$}{Center of rotation of the reachability ray for an $RSR$ path and a given $k$ (m,~m)}
\nomenclature{$\omega^k_{LSL[RSR]}(\alpha)$}{Rotation of a reachability ray for a given $\alpha$ and $k$ for an $LSL$[$RSR$] path (rad)}
\nomenclature{$\overline{\omega^k_{LSL[RSR]}}(\alpha)$}{Rotation of $\omega^k_{LSL[RSR]}(\alpha)$ by $\pi$ (rad)}
\nomenclature{$\alpha_{inf[sup]}^k$}{Infimum [supremum] of $\alpha$ for a given $k$ (rad)}
\nomenclature{$T_{Dubins}$}{Travel time of the optimal Dubins path (s)}
\nomenclature{$T_{2\pi[4\pi]}$}{Travel time of the $2\pi$-arc [$4\pi$-arc] path (s)}
\nomenclature{$\phi_{k_1, k_2}$}{Rotation of the line segment joining the centers of rotation of two path types with parameters $k_1$ and $k_2$, respectively (rad)}

\printnomenclature[55pt]

\section{Introduction}\label{sec:intro}
\subsection{Background}
A fundamental problem in robotics is to find the minimum-time path from a start pose to a goal pose while considering several constraints on vehicles such as bounded curvature~\cite{DM12}\cite{FS04}, bounded velocity~\cite{BM02}\cite{LK08} and bounded acceleration~\cite{RP94}\cite{bestaoui1989line}. In particular, bounded curvature implies that the vehicle's turning is  subject to a non-zero minimum turning radius corresponding to its speed and maximum turn rate.

Dubins~\cite{D57}\cite{shkel2001classification} used a geometrical approach to show that in absence of obstacles, the shortest path for a curvature-constrained vehicle between a pair of poses must be one of the following six path types (also known as the Dubins curves): $LSL$, $RSR$, $LSR$, $RSL$, $LRL$ and $RLR$, where $L(R)$ refers to a left (right) turn with the maximum curvature, and $S$ indicates a straight line segment. Since each path type is composed of three segments, it is uniquely determined by three path parameters, which describe the angles of the circular arcs and the length of the straight line segment. Recently, the authors proposed the T$^\star$ algorithm~\cite{SGW19} which extended the Dubins approach to variable speed vehicles in obstacle-rich environments  for time-optimal risk-aware motion planning. However, when environmental currents (e.g., wind or ocean currents) are present, the vehicle trajectory can be significantly distorted~\cite{MG19}, resulting in a minimum-time trajectory which is different from the minimum-distance trajectory.

\begin{figure}[t]
    \centering
    \subfloat[$2\pi$-arc paths in the inertial frame (IF) and the current frame (CF).]{
        \includegraphics[width=0.99\columnwidth]{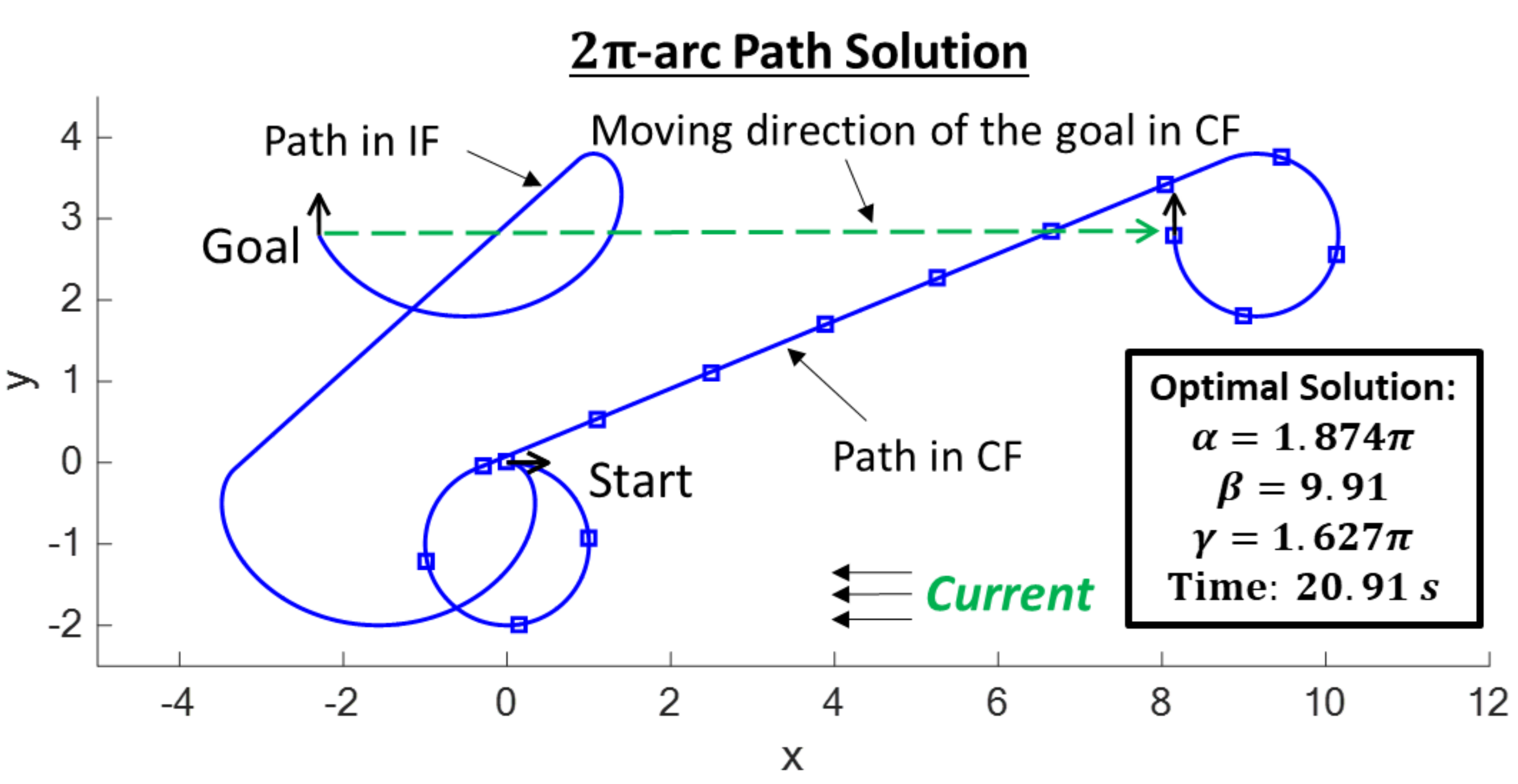}\label{fig:frontpage_fig1}}
        \\
    \subfloat[$4\pi$-arc paths in the inertial frame (IF) and the current frame (CF).]{
         \includegraphics[width=.99\columnwidth]{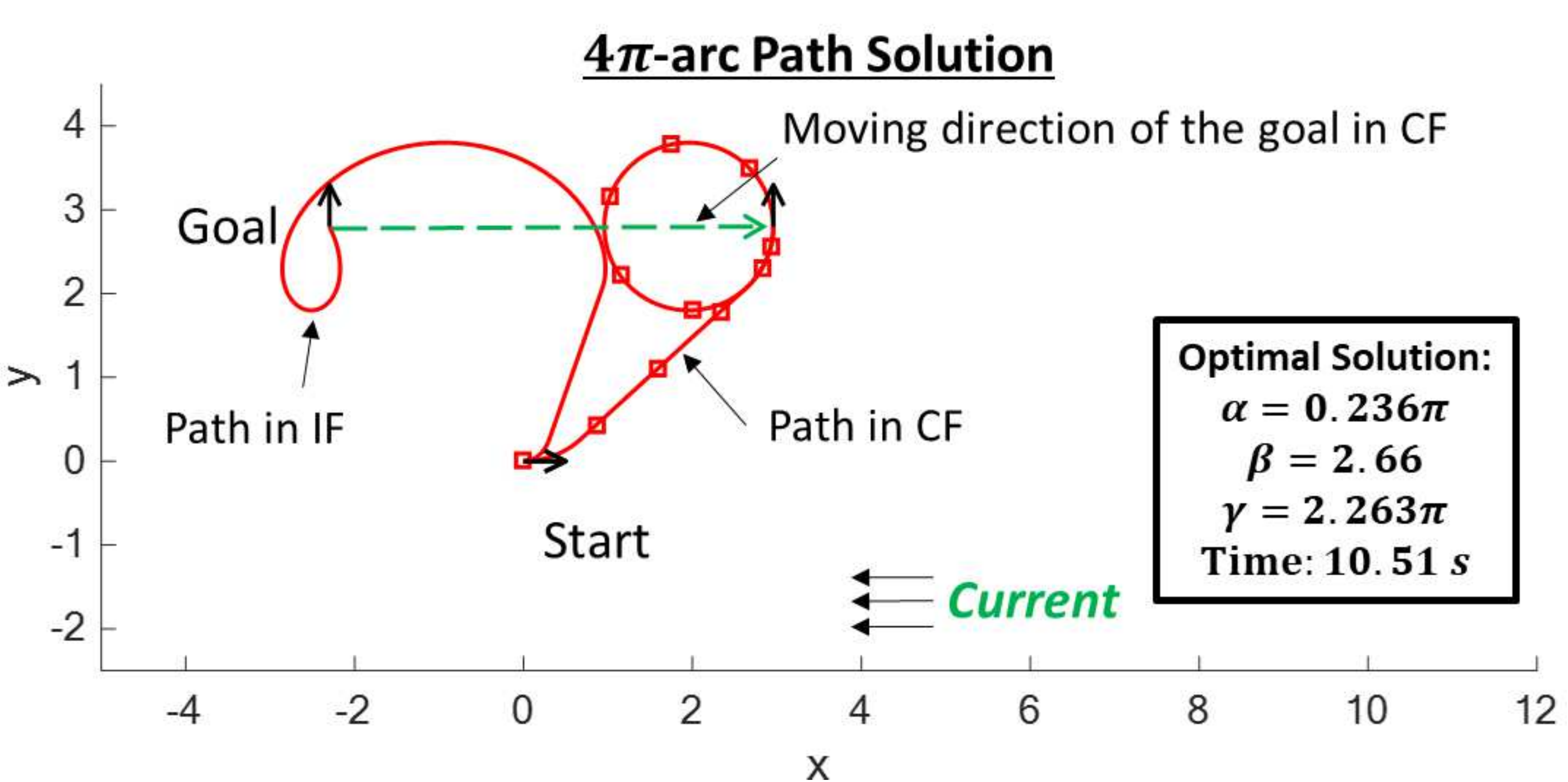}\label{fig:frontpage_fig2}}
         \vspace{-1pt}
    \caption{The minimum-time $2\pi$-arc paths vs. $4\pi$-arc paths. The current vector $(-0.5, 0)$, the start pose $(0,0,0)$ and the goal pose $(-2.3, 2.8, \pi/2)$.}\label{fig:frontpage_fig}
\vspace{-9pt}
\end{figure}

Along this line, the existing methods to compute the minimum-time trajectory for Dubins vehicles in the presence of  environmental currents can be categorized into two types: (1) solutions in the inertial frame (IF)~\cite{TW09} and (2) solutions in the current frame (CF)~\cite{MSH05}\cite{BT13}. The current frame is the inertial frame that moves at the speed and direction of the current. Fig.~\ref{fig:frontpage_fig1} shows the minimum-time Dubins path both in the IF and the CF. Due to the effect of current, the optimal Dubins path in the CF results in the distorted trochoidal path in the IF, therefore, the solutions in the IF have complex expressions~\cite{TW09}. A major advantage of using the CF is that the effect of current on the vehicle trajectory is completely encompassed by the motion of the reference frame, hence the path planning problem can be simplified to a moving-target interception problem using Dubins paths~\cite{MSH05}\cite{BT13}\cite{meyer2015dubins}\cite{ding2019curvature}. While details are discussed later, Fig.~\ref{fig:frontpage_fig2} shows the optimal paths obtained by our method in the CF and the IF.

\vspace{-12pt}
\subsection{The Real-time Challenge}
Although the above methods can produce the minimum-time trajectory for Dubins vehicles in the presence of static currents, their real-time application is limited due to their computational complexity.  As shown in~\cite{TW09}\cite{BT13}, the existing approaches require to solve for all six Dubins path types to find the minimum-time trajectory. Out of these six path types, only $LSL$ and $RSR$ paths have analytical solutions, while the remaining four path types require to solve a root-finding problem involving transcendental equations, which demand significant computational efforts. However, in dynamic situations (e.g., changing currents, adaptive exploration~\cite{SGH13}\cite{GRP09} and target tracking~\cite{HGW19}) it is critical to obtain a real-time solution for fast replanning, which is the focus of this paper.

\begin{figure}[t]
    \centering
    \includegraphics[width=0.90\columnwidth]{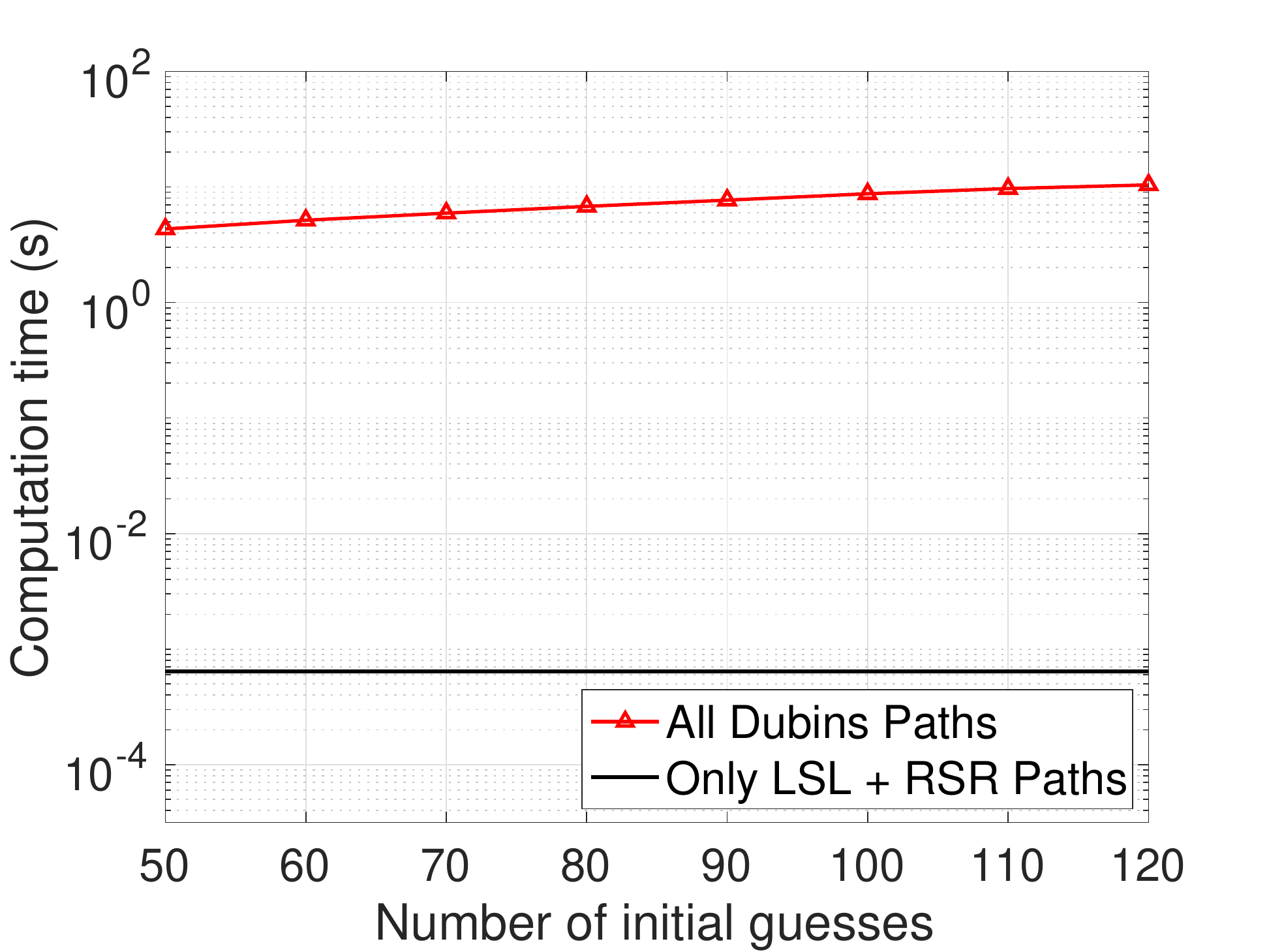}
    \caption{Mean computation times for $LSL$ and $RSR$ paths as compared to all six Dubins paths, over $1000$ randomly selected start and goal poses in a steady current environment, on a $2.4$~GHz CPU computer with $8$~GB RAM.}\label{fig:frontpage_computationtime} \vspace{-6pt}
\end{figure}

To motivate this further, we generated the computation time required to obtain the minimum-time path from all six Dubins path types, as shown in Fig.~\ref{fig:frontpage_computationtime}. Also, we compared this to the computation time required to get the minimum-time path from only the $LSL$ and $RSR$ path types. These computation times were obtained by averaging over $1000$ randomly selected start and goal poses in an environment with steady currents. The simulations were run in MATLAB on a computer with $2.4$~GHz CPU and $8$~GB RAM. It is seen that using only $LSL$ and $RSR$ paths takes $\sim6.4\times 10^{-4}$~s to get a solution. In contrast, using all six path types takes several orders of magnitude higher time to solve the transcendental equations. Furthermore, for practical applications, these numbers can become significantly larger for less powerful on-board processors. Moreover, these computation times depend on the non-linear solvers used.  In addition, the implementation of these optimization solvers on on-board processors is challenging as compared to a system of equations with analytical solutions.

\textit{Example}: The potential implications of computation times are shown with an example. Consider an underwater vehicle moving at 2.5~m/s in an environment with a time-varying current with a speed of $2$~m/s. Now, suppose the current changes direction towards that of the vehicle motion, then a new path needs to be computed. Suppose that it takes $\sim8.72$~s for the on-board processor to get a solution using all six path types. Then, the vehicle would drift by a distance of $ 8.72 \cdot (2+2.5)=39.24$~m before it could compute a new path. In comparison, if it uses only $LSL$ and $RSR$ path types, then this drift would be as little as $6.4\times 10^{-4}\cdot (2+2.5)= 0.0029$~m. Thus, computation time plays a crucial role in real-time path planning in dynamic environments.

\vspace{-6pt}
\subsection{Our Approach} \label{ourapproach}

Based on the above discussion, we propose a rapid (real-time) analytical solution as described below.

\vspace{6pt}
\subsubsection{\textbf{Proposed solution using $4\pi$-arc $LSL$ and $RSR$ paths}}
We propose a solution in the CF using only the $LSL$ and $RSR$ path types. However, the limitation of using only this subset of path types is the lack of full reachability, i.e., they cannot reach every goal pose in the presence of currents. To overcome the above limitation, we propose a simple yet powerful technique. Instead of using the regular $LSL$ and $RSR$ paths where the arc angles are within a range of $[0, 2\pi)$, we propose to extend their arc range to $[0, 4\pi)$~\cite{mittal2019real}. Accordingly, we define the concepts of $2\pi$-arc and $4\pi$-arc paths below, where the parameters $\alpha$ ($\gamma$) and $\beta$ refer to the turning angle of the first (second) arc and the length of the straight line segment, respectively.

\begin{defn}[\textbf{$2\pi$-arc Path}]\label{defn:conventional_path}
An $L^\alpha S^\beta L^\gamma$ or $R^\alpha S^\beta R^\gamma$ path is called a \textit{$2\pi$-arc} path, if $\alpha \in [0,2\pi)$ and $\gamma \in [0,2\pi)$.
\end{defn}

\begin{defn}[\textbf{$4\pi$-arc Path}]\label{defn:unconventional_path}
An $L^\alpha S^\beta L^\gamma$ or $R^\alpha S^\beta R^\gamma$ path is called a \textit{$4\pi$-arc} path, if $\alpha \in [0, 4\pi)$ and $\gamma \in [0, 4\pi)$.
\end{defn}

\begin{rem} The six Dubins path types use the $2\pi$-arcs.
\end{rem}

\begin{rem}It is shown that the $4\pi$-arc $LSL$ and $RSR$ paths provide full reachability along with reduced total time costs as compared to the $2\pi$-arc $LSL$ and $RSR$ paths.
\end{rem}

\textit{Example}: Figs.~\ref{fig:frontpage_fig1} and \ref{fig:frontpage_fig2} show the minimum-time $2\pi$-arc and $4\pi$-arc paths, respectively, in both the IF and the CF. Fig.~\ref{fig:frontpage_fig1} shows the optimal $2\pi$-arc path, which is a $RSR$ path with the total time cost of $20.91$~s. In comparison,  Fig.~\ref{fig:frontpage_fig2} shows the optimal $4\pi$-arc path, which is a $LSL$ path with $\gamma = 2.263\pi > 2\pi$ and the total time cost of $10.51$~s. Intuitively, this happens because instead of traveling against the current, the vehicle spends more time on arcs which allows the current to help it  to reach the goal in less time.
\vspace{6pt}

\subsubsection{\textbf{Theoretical analysis of $4\pi$-arc $LSL$ and $RSR$ paths}} We present a rigorous theoretical analysis of the properties of $4\pi$-arc $LSL$ and $RSR$ paths. First, we develop a comprehensive procedure for reachability analysis of the $2\pi$-arc $LSL$ and $RSR$ paths. We present the conditions for full reachability using these two path types with the support from Lemmas~\ref{lem:swipe}$-$\ref{lem:slopeproperty3}. The derivation of these conditions and the proofs of supporting lemmas are provided in Appendices~\ref{app:reachability} and~\ref{app:lemma_proofs}, respectively. Next, it is numerically validated that the $2\pi$-arc $LSL$ and $RSR$ paths fail to satisfy the reachability conditions under all goal poses and current velocities. Thus, we present  Theorem~\ref{claim1}, which provides a guarantee of full reachability using $4\pi$-arc $LSL$ and $RSR$ paths. Further, it is established through Theorem~\ref{claim2} and Corollary~\ref{claim2_cor}, that the computational complexity of both $2\pi$-arc and $4\pi$-arc path solutions is the same. Along with providing full reachability, another important benefit of $4\pi$-arc paths is their ability to generate faster, i.e., reduced time cost, paths in comparison to the $2\pi$-arc paths, which is highlighted in Theorem~\ref{claim3}. Finally, Theorem~\ref{rem:over_4pi} is presented to prove that $\alpha,\gamma\in[0, 4\pi)$ is sufficient for optimality using $LSL$ and $RSR$ path types and thus further increasing of range is not needed. For validation of our approach, extensive Monte Carlo simulations are performed to compare the performance of Dubins solutions and the proposed $4\pi$-arc path solutions.

\vspace{6pt}
\subsubsection{\textbf{Comparison of $4\pi$-arc $LSL$ and $RSR$ paths with Dubins}} The solution obtained from the $4\pi$-arc $LSL$ and $RSR$ paths might be sub-optimal for certain goal poses as compared to the one obtained from  the six Dubins path types; however, the longer convergence time of the Dubins path solution might render it unsuitable for real-time applications.

For offline applications in static current environments, one can use the Dubins path types to compute the minimum-time path. In this regard, Section~\ref{app:dubins_comparison} provides a detailed comparison of the solution quality (i.e., travel time cost) obtained for the $4\pi$-arc $LSL$ and $RSR$ solutions and the Dubins solutions. This analysis indicates that the advantage  of the Dubins solutions over the $4\pi$-arc $LSL$ and $RSR$ solutions in terms of travel time costs is not significant. Furthermore, upon adding the computation time costs, the advantage of Dubins solutions is further reduced. On the other hand, for time critical real-time applications (e.g., target tracking, planning under moving obstacles, and changing currents), $4\pi$-arc paths provide rapid and reliable solutions without causing any vehicle drift. In contrast, the high computation times for Dubins solutions can cause vehicle drifts, thereby, resulting in longer sub-optimal trajectories which sometimes do not even converge to the goal pose. Section~\ref{changingcurrents} presents a comparative analysis in the presence of dynamic currents, which highlights the benefits of the solutions obtained from the $4\pi$-arc $LSL$ and $RSR$ paths over the ones obtained from the six Dubins paths.

\vspace{-6pt}
\subsection{Our Contributions} The paper makes the following novel contributions:
\begin{itemize}
\item Provides an analytical solution of the path planning problem for Dubins vehicles under environmental currents, where the solution is based on a novel concept of $4\pi$-arc $LSL$ and $RSR$ paths and can be computed in  real-time. In this regard, the paper presents the following:
\begin{itemize}
\item A detailed analytical method to construct the reachability graphs of $LSL$ and $RSR$ paths.
\item A detailed derivation of the conditions under which $2\pi$-arc $LSL$ and $RSR$ paths provide full reachability.
\item A mathematical proof of full reachability of the $4\pi$-arc $LSL$ and $RSR$ paths under all conditions unlike the corresponding $2\pi$-arc paths (\textbf{Theorem~\ref{claim1}}).
\item A mathematical proof that a solution using $4\pi$-arc $LSL$ and $RSR$ paths can be obtained with the same computational workload as that needed for $2\pi$-arc paths (\textbf{Theorem~\ref{claim2} and  Corollary~\ref{claim2_cor}}),
\item  A mathematical proof that $4\pi$-arc $LSL$ and $RSR$ paths provide reduced travel time costs as compared to the corresponding $2\pi$-arc paths. (\textbf{Theorem~\ref{claim3}}).
\end{itemize}
\item Theoretical properties of $4\pi$-arc $LSL$ and $RSR$ paths are rigorously established and  evaluated  in  comparison  to  Dubins solutions  by  extensive  Monte-Carlo  simulations.
\end{itemize}

\vspace{-6pt}
\subsection{Organization}
The rest of the paper is organized as follows. Section~\ref{sec:litreview} reviews the existing literature. Section~\ref{sec:problemandsolution} presents the path planning problem and its analytical solution. Section~\ref{sec:reachabilityanalysis} presents a detailed analytical procedure for the reachability analysis of the 2$\pi$-arc $LSL$ and $RSR$ paths. Section~\ref{sec:main_contents} presents the theoretical properties of $4\pi$-arc paths and shows their advantages over the $2\pi$-arc paths. Section~\ref{sec:results} presents the comparative evaluation results. Finally, the paper is concluded in Section~\ref{sec:conclusion} with recommendations for future work. Appendices~\ref{app:reachability} and~\ref{app:lemma_proofs} provide proofs of reachability conditions and supporting lemmas.

\vspace{3pt}
\section{Literature Review}\label{sec:litreview}
Recently, several papers~\cite{zeng2016comparison} have addressed the path planning problem in the presence of  currents. Garau et al.~\cite{GBARP09}\cite{GAO05} studied the minimum-time path planning problem in marine environments with spatial current variability, where the time cost was defined as the sum of step-wise costs that are specified by the traveling distance over the vehicle speed in the presence of ocean currents. However, the drawback in their design is that infeasible paths are penalized rather than being prohibited. Petres et al.~\cite{PPPPEL07} presented the FM$^\star$ algorithm to find the minimum-time path for underwater vehicles, where the time cost is defined over the inner product of the distance function and the current field; however, their cost function still penalizes rather than restricts infeasible paths. In this regard, Soulignac et al.~\cite{STR09} proposed a time cost function that projects the speed vector to both axes as opposed to taking its norm as in~\cite{GAO05}. Accordingly, their method is restricted to feasible paths. In addition, energy based cost functions~\cite{ACO04}\cite{ZIOM08} have also been used for planning in the presence of ocean currents.

However, the above-mentioned methods ignore any kinematic motion constraints for vehicles. Along this line, Techy and Woolsey~\cite{TW09} addressed the minimum-time path planning problem for a curvature-constrained vehicle in constant wind, based on the fact that the circular arcs are distorted by the wind into the trochoidal curves in the inertial frame. They derived analytical solutions for $LSL$ and $RSR$ candidate paths, while for other paths of $LSR$, $RSL$, $LRL$ and $RLR$, they must solve certain transcendental equations to obtain solutions. However, as we show in Fig.~\ref{fig:frontpage_computationtime}, the root finding problem for transcendental equations can be computationally expensive.

In contrast, McGee et al.~\cite{MSH05} studied the minimum-time path planning problem in the current frame. They first used Pontryagin's Minimum Principle to demonstrate that the optimal path is comprised of straight line segments and curves of maximum turn rate. Then, they introduced the concept of a "virtual target" which starts at the goal state but moves in the opposite direction as the wind. In this setup, the minimum-time problem is simplified into a target interception problem, where the objective is to find the earliest interception point in the current frame so that the Dubins path can meet with the virtual target in minimum time. However, one must repeatedly check for the validity of possible interception points, which can be arbitrarily heavy to compute if the actual interception point lies far from the beginning search point.

In this regard, Bakolas et al.~\cite{BT13} directly solved for the interception point in the current frame by introducing an extra parameter of interception time. They also showed that when the wind speed is less than the vehicle speed, the vehicle has full reachability, i.e., the optimal path always exists for any given goal pose. However, their solution methodology still involves solving for the roots of multiple transcendental equations, which could lead to heavy computational burden, thus prohibiting it from real-time applications.

Some researchers used the Nonlinear Trajectory Generation (NTG) algorithm~\cite{inanc2005} based on spline curves to obtain the optimal trajectory of a glider with kinematic constraints in presence of dynamically varying ocean currents. The proposed algorithm relies on Sequential Quadratic Programming (SQP) approach to solve the nonlinear programming problem which might lead to sub-optimal solutions and high computational time. In comparison, this paper proposes a novel method which provides a rapid analytical solution to the path planning problem under currents with guaranteed full reachability.

\vspace{0pt}
\section{Problem Description and Solution}\label{sec:problemandsolution}
This section presents the minimum-time path planning problem for Dubins vehicles and its analytical solution.

\vspace{0pt}
\subsection{Problem Description}\label{sec:problem}
Consider a vehicle moving  at a velocity ${\bf v} =(v\cos{\theta}, v\sin{\theta})$, where $v\in \mathbb{R}^+$ is its speed and $\theta\in [0, 2\pi)$ is its heading. A steady current is assumed to be present in the environment with velocity ${\bf v}_w = (v_w\cos{\theta_w}, v_w\sin{\theta_w})\equiv (w_x, w_y) $, where $v_w \in \mathbb{R}^+$ is its speed and $\theta_w \in [0, 2\pi)$ is its direction. The current speed is assumed to be slower than the vehicle speed, i.e., $v_w < v$. Then, the motion of the vehicle can be described as:

\vspace{-3pt}
\begin{equation}
  \begin{cases}
      \dot{x}(t) &=  v \cdot \cos{\theta(t)} + w_x \\
      \dot{y}(t) &=  v \cdot \sin{\theta(t)} + w_y \\
      \dot{\theta}(t) &=  u(t) \label{eq:vehicle_model_currents}
  \end{cases},
\end{equation}
where $\mathbf{p} = (x,y,\theta) \in SE(2)$ is the vehicle pose and $u$ indicates its turn rate. By choosing a proper unit, the vehicle speed can be normalized to $v = 1$. The turn rate $u$ is symmetric and bounded, s.t., $u \in [-u_{\max}, u_{\max}]$, where $u_{\max} \in \mathbb{R}^+$ is the maximum turn rate and the $+$/$-$ sign indicates a left/right turn. These constraints imply that the vehicle is subject to the minimum turning radius of $r= 1/u_{\max}$ (for $v=1$).

Then, for a vehicle operating in a current environment, as described in~(\ref{eq:vehicle_model_currents}), the objective is to find the minimum-time path from a start pose $\mathbf{p}_{start} = (x_0, y_0, \theta_0)$ to a goal pose $\mathbf{p}_{goal} = (x_f, y_f, \theta_f)$. The state-of-the-art solutions~\cite{TW09}\cite{MSH05}\cite{BT13} to this problem require to solve for all six Dubins path types to find the minimum-time path. However, as shown in (34) and (39) of~\cite{BT13}, in order to obtain the path types of $LSR$, $RSL$, $LRL$ and $RLR$, one must solve a root-finding problem involving transcendental equations for numerical solutions. This inevitably requires significant computation resources and thus can seriously restrict their usage in real-time applications.

In this regard, in order to achieve a real-time solution, we address the above problem using only two path types which have direct analytical solutions. These are $L^{\alpha}S^{\beta}L^{\gamma}$ and $R^{\alpha}S^{\beta}R^{\gamma}$, where $\alpha$ and $\gamma$ are the turning angles of the first and second arc segments, respectively; and $\beta \geq 0$ denotes the length of the straight line segment. Thus, the solution for each path type is uniquely determined by the 3-tuple $\{\alpha, \beta, \gamma\}$ of path parameters. Since these parameters can be solved analytically, the solution is obtained very fast (in real-time).

However, due to using only a subset of the Dubins path types, there exist goal poses for which neither $LSL$ nor $RSR$ path can provide feasible solutions, i.e., $LSL$ and $RSR$ paths do not provide full reachability. To address this issue, we extend the feasible ranges of $\alpha$ and $\gamma$ from $[0,2\pi)$ to $[0,4\pi)$. It is shown later that the extended $LSL$ and $RSR$ path types guarantee full reachability, and can provide the solutions with even less time costs.

\vspace{-6pt}
\subsection{Solutions for the $LSL$ and $RSR$ Paths}\label{sec:method}
This section derives the analytical solutions for the parameters of the $LSL$ and $RSR$ path types using the CF, which moves with the same speed and direction as that of the current. In the  CF, the goal moves in the opposite direction with $(-w_x, -w_y)$. Thus, the problem is  simplified to a moving-target interception  problem. Therefore, the objective is to find the minimum interception time to meet with the moving goal using Dubins  $LSL$ and $RSR$ paths. Without loss of generality, we choose the start pose $(x_0, y_0, \theta_0) = (0,0,0)$.

\vspace{6pt}
\subsubsection{$L^{\alpha}S^{\beta}L^{\gamma}$ Path}\label{sec:LSL}
As seen in Fig.~\ref{fig:LSL_derivation}, in order to reach the goal $(x_f,y_f,\theta_f)$ in the CF, the following boundary constraints must be satisfied for an $LSL$ path~\cite{BT13}:
\begin{equation}\label{eq:LSL_conditions}
  \begin{cases}
      x_f - w_x T &=  r \sin\theta_f + \beta \cos\alpha \\
      y_f - w_y T &=  r (1-\cos\theta_f) + \beta \sin\alpha \\
      T &=  \big(r (\alpha + \gamma) + \beta \big)/v \\
      \alpha + \gamma &= 2 k \pi + \theta_f
  \end{cases},
\end{equation}
where $v = 1$ and $T \in \mathbb{R}^+$ is the total travel time.

\begin{figure}[t]
    \flushleft
    \hspace*{-1em} \subfloat[$L^\alpha S^\beta L^\gamma$ path]{
        \includegraphics[width=.520\columnwidth]{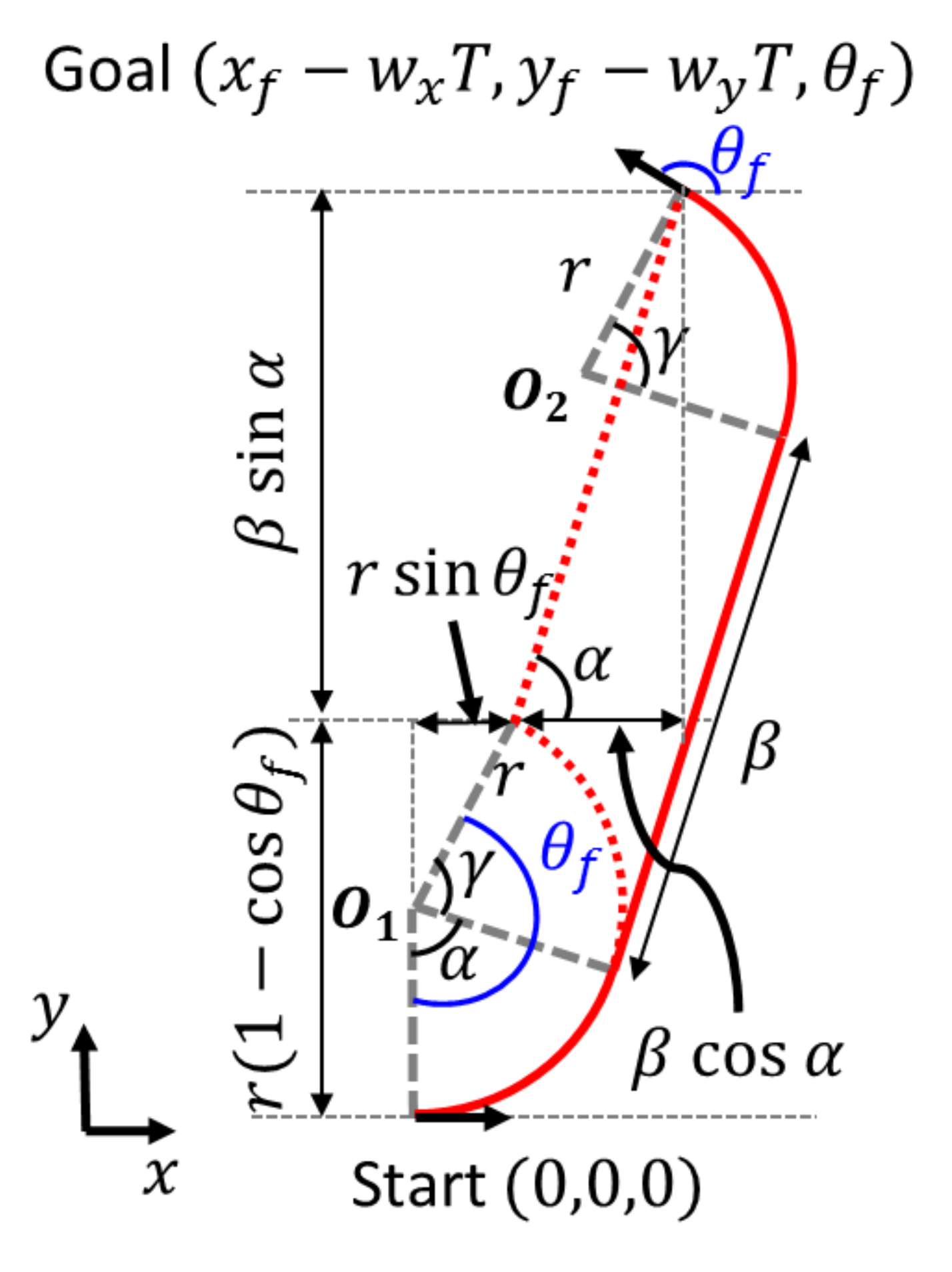}\label{fig:LSL_derivation}}
     \hspace*{-1.2em} \subfloat[$R^\alpha S^\beta R^\gamma$ path]{
         \includegraphics[width=.520\columnwidth]{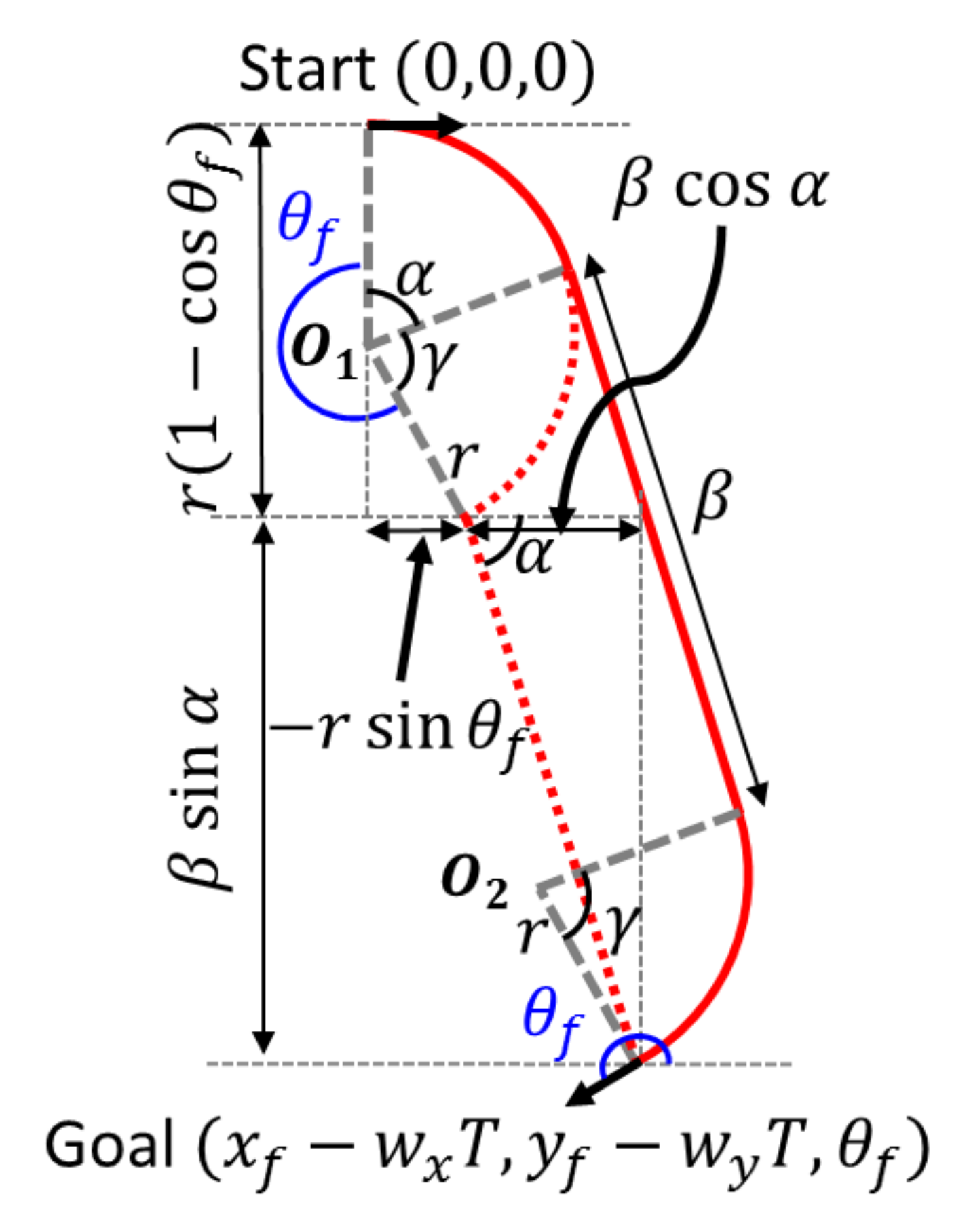}\label{fig:RSR_derivation}}
         \caption{Geometric illustration for $LSL$ and $RSR$ paths.} \label{fig:LSL_RSR_derivation}
         \vspace{-12pt}
\end{figure}

In addition, we introduce $k\in \mathbb{Z}$ to control the feasible ranges of $\alpha$ and $\gamma$. Specifically, for a $2\pi$-arc $LSL$ path, since $\theta_f \in [0, 2\pi)$ and $\alpha, \gamma \in [0, 2\pi)$, one has $k \in \{0, 1\}$. In contrast, for a $4\pi$-arc $LSL$ path, since $\alpha, \gamma \in [0, 4\pi)$, one has $k \in \{0,1,2,3\}$. Note: We show later that we need only $k \in \{0,1\}$ to find a feasible minimum-time $4\pi$-arc $LSL$ path.

Now, for a given $k$, define $A^k$ and $B^k$ as follows:
\begin{equation}\label{eq:LSL_AB}
  \begin{cases}
      A^k &=  x_f - r \sin \theta_f - w_x r (2k\pi + \theta_f) \\
      B^k &=  y_f -r (1- \cos \theta_f) - w_y r(2k\pi +  \theta_f)\\
  \end{cases},
\end{equation}
which are constants that can be computed given the current velocity, and the start and goal poses. Then, using (\ref{eq:LSL_conditions}) and (\ref{eq:LSL_AB}), we get:

\vspace{-3pt}
\begin{equation}\label{eq:LSL_AB_condition}
  \begin{cases}
      A^k &=  \beta \cos \alpha + w_x \beta \\
      B^k &=  \beta \sin \alpha + w_y \beta \\
  \end{cases}.
\end{equation}

Based on (\ref{eq:LSL_AB_condition}), we can compute $\beta$ by solving the quadratic equation $\big(A^k -w_x \beta \big)^2 + \big( B^k -w_y\beta \big)^2 = \beta^2$, such that

\begin{equation}\label{eq:LSL_beta}
\beta = \frac{\pm\sqrt{(A^k w_x + B^k w_y)^2 + ({A^k}^2 + {B^k}^2)(1-v_w^2)} - (A^k w_x + B^k w_y)}{1-v_w^2}.
\end{equation}

It is seen from (\ref{eq:LSL_beta}) that when $v_w < 1$, $\beta$ has valid solutions. Then, $\alpha$ can be computed as

\vspace{-3pt}
\begin{equation}
\alpha = \atantwo \big( B^k - \beta w_y, A^k -\beta w_x \big)  (\textrm{mod} \ \kappa),
\label{eq:LSL_alpha}
\end{equation}
where $\kappa = 2\pi$ for $2\pi$-arc paths, and $\kappa = 4\pi$ for $4\pi$-arc paths. Thereafter, $\gamma$ is computed as $\gamma = 2k\pi + \theta_f - \alpha$ (mod $\kappa$).

\vspace{10pt}
\subsubsection{$R^{\alpha}S^{\beta}R^{\gamma}$ Path}\label{sec:RSR}
As seen in Fig.~\ref{fig:RSR_derivation}, the following boundary constraints must be satisfied for an $RSR$ path:

\vspace{-3pt}
\begin{equation}\label{eq:RSR_conditions}
  \begin{cases}
      x_f - w_x T &=  -r \sin\theta_f + \beta \cos\alpha\\
      y_f - w_y T &=  -r (1-\cos\theta_f) - \beta \sin\alpha \\
      T &=  \big( r (\alpha + \gamma) + \beta \big) /v \\
      -\alpha - \gamma & = 2 k \pi + \theta_f
  \end{cases}.
\end{equation}

For a $2\pi$-arc $RSR$ path, since $\theta_f \in [0, 2\pi)$ and $\alpha, \gamma \in [0, 2\pi)$, one has $k \in \{-1,-2\}$; while for a $4\pi$-arc $RSR$ path, because $\alpha, \gamma \in [0,4\pi)$, one has $k \in \{-1, -2, -3, -4\}$. Note: We show later that we need only $k \in \{-1,-2\}$ to find a feasible minimum-time $4\pi$-arc $RSR$ path. Now, define
\vspace{-3pt}
\begin{equation}\label{eq:RSR_AB}
  \begin{cases}
      A^k &=  x_f + r \sin \theta_f + w_x r(2k\pi + \theta_f)  \\
      B^k &=  y_f + r (1- \cos \theta_f) + w_y r(2k\pi + \theta_f)  \\
  \end{cases},
\end{equation}
and using (\ref{eq:RSR_conditions}) and (\ref{eq:RSR_AB}), we get:

\vspace{-3pt}
\begin{equation}\label{eq:RSR_AB_condition}
  \begin{cases}
      A^k &=  \beta \cos \alpha + w_x \beta \\
      B^k &=  -\beta \sin \alpha + w_y \beta \\
  \end{cases}.
\end{equation}

Then, $\beta$ is solved using $\big(A^k - w_x\beta \big)^2 + \big( B^k - w_y\beta \big)^2 = \beta^2$, which results in the same expression as (\ref{eq:LSL_beta}). Similarly, when $v_w <1$, $\beta$ has valid solutions. Then, $\alpha$ can be computed as
\begin{equation}
\alpha = \atantwo(-B^k + \beta w_y, A^k -\beta w_x) \ (\text{mod} \ \kappa),
\label{eq:RSR_alpha}
\end{equation}
and $\gamma$ is computed as $\gamma = - 2k\pi -\theta_f - \alpha$  (mod $\kappa$).

\vspace{10pt}
\subsection{Feasible Ranges of Path Parameters}\label{sec:tight_bounds}
According to Defn.~\ref{defn:conventional_path} and Defn.~\ref{defn:unconventional_path}, the parameters $\alpha$ and $\gamma$ are defined over $[0, 2\pi)$ and $[0, 4\pi)$ for $2\pi$-arc paths and $4\pi$-arc paths, respectively. Given the direction $\theta_f \in [0, 2\pi)$ of the goal pose, we can obtain tighter feasible ranges for $\alpha$ and $\gamma$. Table~\ref{table:range_conventional}  shows the feasible ranges of path parameters for both $2\pi$-arc and $4\pi$-arc paths.  An example is provided below.

\vspace{6pt}
\textit{Example}: Consider a $4\pi$-arc $LSL$ path, where $\alpha \in [0, 4\pi)$ and $\gamma \in [0, 4\pi)$. There are four cases to study:

\begin{itemize}
\item $k = 0$ (i.e., $\alpha + \gamma = \theta_f < 2\pi$): Now, $\gamma \geq 0$ $\implies$ $\alpha \leq \theta_f$.
Similarly, $\alpha \geq 0$ $\implies$ $\gamma \leq \theta_f$. Thus, the feasible range for both $\alpha$ and $\gamma$ is $[0,\theta_f]$.

\item $k = 1$ (i.e., $\alpha + \gamma = 2\pi + \theta_f< 4\pi$): Again, $\gamma \geq 0$ $\implies$ $\alpha \leq 2\pi + \theta_f$. Similarly, $\alpha \geq 0$ $\implies$ $\gamma \leq  2\pi + \theta_f$. Thus, the feasible range for both $\alpha$ and $\gamma$ is $[0, 2\pi + \theta_f]$.

\item $k = 2$ (i.e., $\alpha + \gamma =  4\pi + \theta_f < 6\pi$): Here $\gamma < 4\pi$ $\implies$ $\alpha > \theta_f$. Similarly, $\alpha < 4\pi$ $\implies$ $\gamma > \theta_f$. Thus, the feasible range for both $\alpha$ and $\gamma$ is $(\theta_f, 4\pi)$.

\item $k = 3$ (i.e., $\alpha + \gamma = 6\pi + \theta_f < 8\pi$): Here $\gamma < 4\pi$ $\implies$ $\alpha > 2\pi + \theta_f$. Similarly, $\alpha < 4\pi$ $\implies$ $\gamma > 2\pi + \theta_f$. Thus, the feasible range for both $\alpha$ and $\gamma$ is $(2\pi + \theta_f, 4\pi)$.
\end{itemize}

Similarly, we can obtain the feasible range of path parameters for $4\pi$-arc $RSR$ path and for $2\pi$-arc $LSL$ and $RSR$ paths.

{
\begin{table}[t!]
\centering
\footnotesize
\caption{Feasible parameter ranges for $2\pi$-arc and $4\pi$-arc paths}\label{table:range_conventional}
\begin{tabular}{c|cc|c|cc}
\hline
\multicolumn{6}{c}{$2\pi$-arc Paths ($\alpha,\gamma$ ranges are up to mod $2\pi$)}\\ \hline
\multicolumn{3}{c|}{$LSL$ Path Type} & \multicolumn{3}{c}{$RSR$ Path Type} \\ \hline
$k$ & $\alpha$ and $\gamma$ & $\beta$ & $k$ & $\alpha$ and $\gamma$ & $\beta$ \\ \hline
$0$ & $[0, \theta_f]$ & $[0, \infty)$ & $-1$ & $[0, 2\pi-\theta_f]$  & $[0, \infty)$ \\ \hline
$1$ & $(\theta_f, 2\pi)$ & $[0, \infty)$ & $-2$ & $(2\pi-\theta_f, 2\pi)$ & $[0, \infty)$ \\ \hline
\multicolumn{6}{c}{$4\pi$-arc Paths ($\alpha,\gamma$ ranges are up to mod $4\pi$)}\\
\hline
\multicolumn{3}{c|}{$LSL$ Path Type} & \multicolumn{3}{c}{$RSR$ Path Type} \\ \hline
$k$ & $\alpha$ and $\gamma$ & $\beta$ & $k$ & $\alpha$ and $\gamma$ & $\beta$ \\ \hline
$0$ & $[0, \theta_f]$ & $[0, \infty)$ & $-1$ & $[0, 2\pi-\theta_f]$ & $[0, \infty)$ \\ \hline
$1$ & $[0, 2\pi + \theta_f]$ & $[0, \infty)$ & $-2$ & $[0, 4\pi -\theta_f]$ & $[0, \infty)$ \\ \hline
$2$ & $(\theta_f, 4\pi)$ & $[0, \infty)$ & $-3$ & $(2\pi-\theta_f, 4\pi)$ & $[0, \infty)$ \\ \hline
$3$ & $(2\pi+\theta_f, 4\pi)$ & $[0, \infty)$ & $-4$ & $(4\pi-\theta_f, 4\pi)$ & $[0, \infty)$ \\ \hline
\end{tabular}
\end{table}
}

\vspace{6pt}
\section{Reachability Analysis of $2\pi$-arc Paths}\label{sec:reachabilityanalysis}

This section derives the analytical expressions for generating the reachability graphs of $2\pi$-arc $LSL$ and $RSR$ path types and for finding the conditions of full reachability.

\subsection{Construction of Reachability Graphs}
\label{sec:reachability_graphs}
First, we show that for a given $\alpha$, the reachable goal points $(x_f,y_f)$ lie on a ray. Then, we show that by varying $\alpha$, this ray rotates to form the reachability graph.

\vspace{10pt}
$\bullet$ \textit{\textbf{$2\pi$-arc $LSL$ Paths:}} Let us denote
\begin{subequations}\label{eq:p_q_LSL}
\begin{align}
p^k_{LSL} & \equiv  r\sin{\theta_f} + w_{x}r(2k\pi + \theta_f), \\
q^k_{LSL} & \equiv r(1-\cos{\theta_f}) + w_{y}r(2k\pi + \theta_f),
\end{align}
\end{subequations}
which are constants for $k \in \{0,1\}$ given $\theta_f, w_x$ and $w_y$. Further, let us denote
\begin{subequations}\label{eq:def_a_c}
\begin{align}
    a(\alpha)\equiv \sin{\alpha} + w_y,\\
    c(\alpha) \equiv \cos{\alpha} + w_x.
\end{align}
\end{subequations}
Then, using (\ref{eq:LSL_AB}), (\ref{eq:LSL_AB_condition}), (\ref{eq:p_q_LSL}) and (\ref{eq:def_a_c}) we get:
\begin{subequations}\label{eq:LSL_beta_two_new}
\begin{align}
    x_f &= p^k_{LSL} + \beta \cdot c(\alpha), \label{eq:LSL_beta_1} \\
    y_f &= q^k_{LSL} + \beta \cdot a(\alpha). \label{eq:LSL_beta_2}
\end{align}
\end{subequations}
By performing $a(\alpha)\cdot$(\ref{eq:LSL_beta_1})$- c(\alpha)\cdot$(\ref{eq:LSL_beta_2}), (\ref{eq:LSL_beta_two_new}) is equivalent to the following:

\vspace{-6pt}
\begin{empheq}[box=\widefbox]{align}\label{eq:LSL_reachability}
  a(\alpha) x_f - c(\alpha) y_f - \big(a(\alpha) p^k_{LSL} - c(\alpha) q^k_{LSL} \big) = 0, \nonumber \\
  \text{s.t.: } x_f \geq p_{LSL}^k, y_f \geq q_{LSL}^k \text{, if } a(\alpha) \geq 0, c(\alpha) \geq 0,  \nonumber\\
   x_f < p_{LSL}^k, y_f \geq q_{LSL}^k \text{, if } a(\alpha) \geq 0, c(\alpha) < 0, \\
    x_f < p_{LSL}^k, y_f < q_{LSL}^k \text{, if } a(\alpha) < 0, c(\alpha) < 0, \nonumber \\
  x_f \geq p_{LSL}^k, y_f < q_{LSL}^k \text{, if } a(\alpha) < 0, c(\alpha) \geq 0. \nonumber
\end{empheq}

\vspace{6pt}
The constraints in (\ref{eq:LSL_reachability}) are obtained by using the feasible range of $\beta \geq 0$ in (\ref{eq:LSL_beta_1}) and (\ref{eq:LSL_beta_2}). As shown in Fig. \ref{fig:quadrants}, these constraints define the quadrants of the coordinate frame with center at $\left(p^k_{LSL},q^k_{LSL}\right)$. For a given $\alpha$, (\ref{eq:LSL_reachability}) represents a reachability ray and the goal $(x_f,y_f)$ is reachable if it lies on such ray. The rotation of (\ref{eq:LSL_reachability}), i.e., the angle it makes with the x-axis measured in the counterclockwise direction, is given as
\begin{empheq}[box=\widefbox]{align}\label{eq:LSL_slope}
\omega^{k}_{LSL}(\alpha)=\atantwo \big(a(\alpha),c(\alpha)\big)  \Mod{2\pi}, \ k \in \{0,1\}.
\end{empheq}

\begin{figure}[t]
    \centering
        \includegraphics[width=0.80\columnwidth]{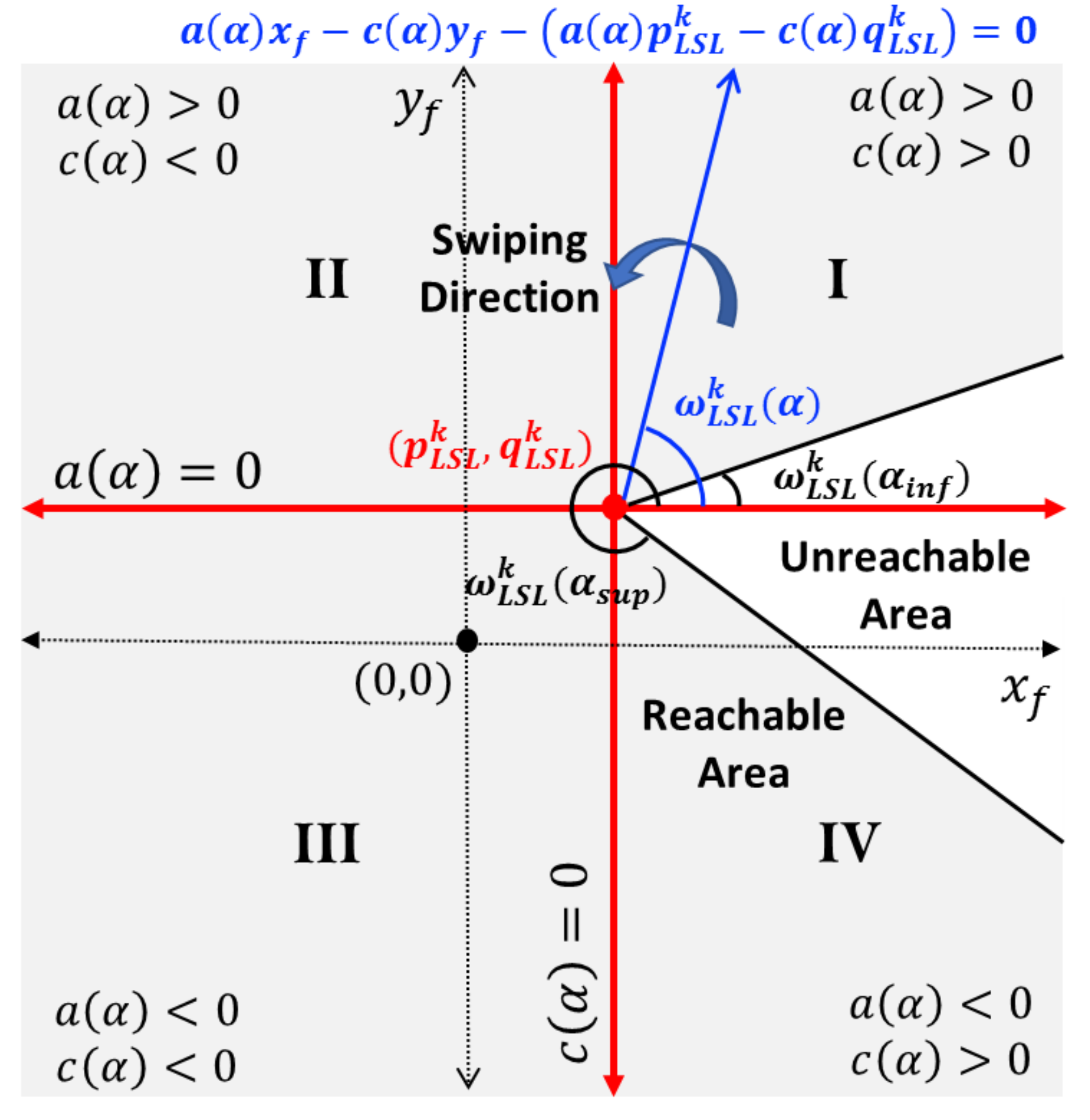}
     \caption{Reachability region of the $LSL$ path type obtained by anticlockwise rotation of~(\ref{eq:LSL_reachability}) about the center of rotation $(p^{k}_{LSL},q^{k}_{LSL})$.}
         \label{fig:quadrants} \vspace{-6pt}
\end{figure}

\begin{figure*}[t]
    \centering
        \includegraphics[width=1\textwidth]{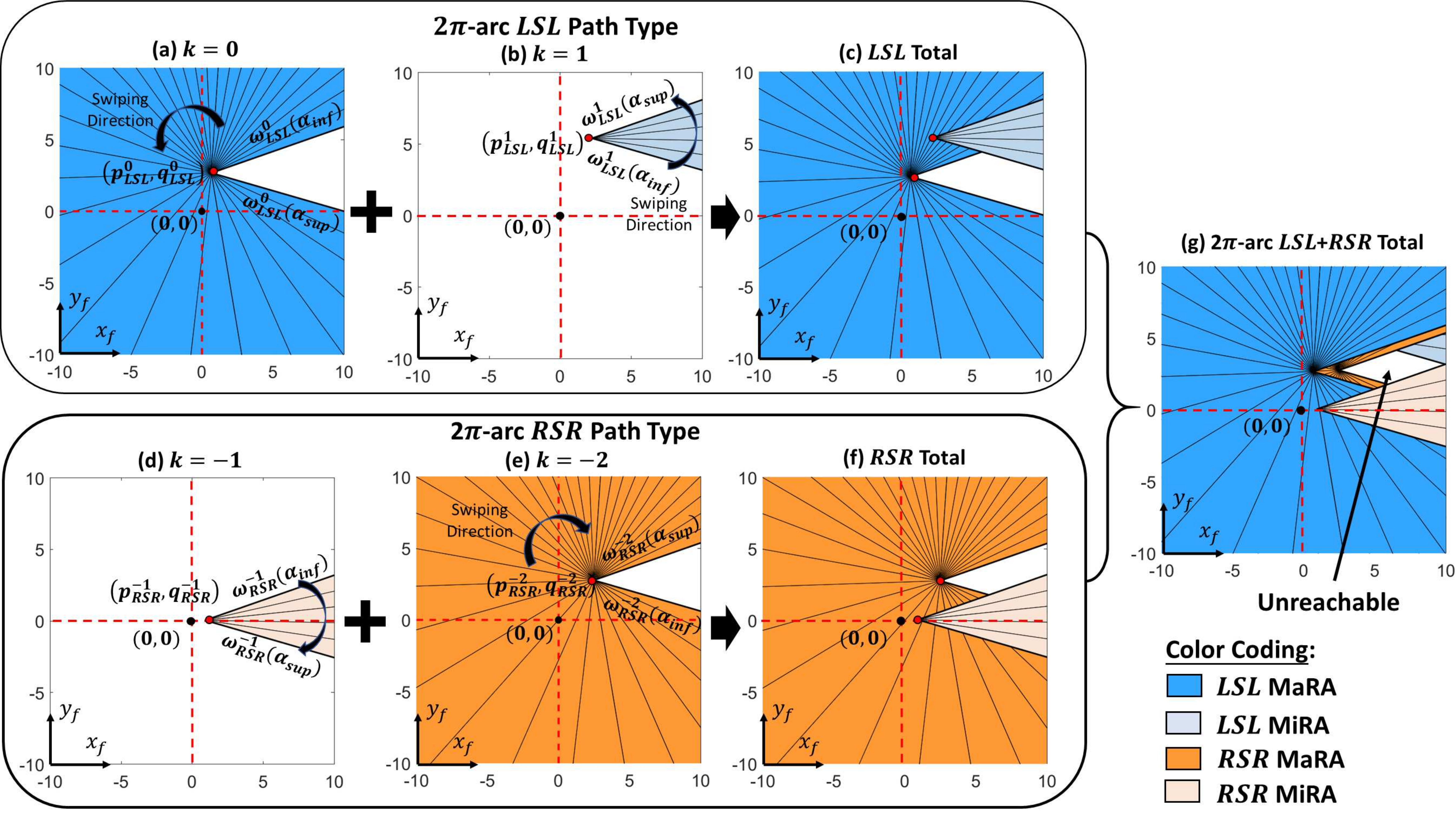}
         \caption{An example showing the construction of reachability graph for $2\pi$-arc $LSL$ and $RSR$ path types. (a) MaRA for $LSL$ with $k=0$, (b) MiRA for $LSL$ with $k=1$, (c) union of MaRA and MiRA for $LSL$ path, (d) MiRA for $RSR$ with $k=-1$, (e) MaRA for $RSR$ with $k=-2$, (f) union of MaRA and MiRA for $RSR$ path, (f) complete reachability graph obtained by taking union of both $LSL$ and $RSR$  path types.}\label{fig:reachability_example}
\end{figure*}

\vspace{6pt}
$\bullet$ \textit{\textbf{$2\pi$-arc $RSR$ Paths:}} Let us denote
\begin{subequations}\label{eq:p_q_RSR}
\begin{align}
p^k_{RSR} &\equiv -r\sin{\theta_f} - w_x r (2k\pi + \theta_f), \\
 q^k_{RSR} &\equiv -r(1-\cos{\theta_f}) - w_y r (2k\pi + \theta_f).
\end{align}
\end{subequations}
which are constants for $k\in \{-1,-2\}$ given $\theta_f, w_x$ and $w_y$. Further, let us denote
\begin{equation}\label{eq:def_b}
    b(\alpha) \equiv \sin{\alpha} - w_y.
\end{equation}
Then, using (\ref{eq:RSR_AB}), (\ref{eq:RSR_AB_condition}), (\ref{eq:p_q_RSR}) and (\ref{eq:def_b}) we get:
\begin{subequations}\label{eq:RSR_beta_two_new}
\begin{align}
    x_f &= p^k_{RSR} + \beta \cdot c(\alpha), \label{eq:RSR_beta_1} \\
    y_f &= q^k_{RSR} - \beta \cdot b(\alpha). \label{eq:RSR_beta_2}
\end{align}
\end{subequations}
By performing $b(\alpha) \cdot$(\ref{eq:RSR_beta_1})$+ c(\alpha) \cdot$(\ref{eq:RSR_beta_2}), (\ref{eq:RSR_beta_two_new}) is equivalent to the following:
\vspace{-8pt}
\begin{empheq}[box=\widefbox]{align}\label{eq:RSR_reachability}
  b(\alpha) x_f + c(\alpha) y_f - \big(b(\alpha) p^k_{RSR} + c(\alpha) q^k_{RSR} \big) = 0, \nonumber \\
  \text{s.t.: } x_f \geq p_{RSR}^k, y_f \geq q_{RSR}^k \text{, if } b(\alpha) \leq 0, c(\alpha) \geq 0,  \nonumber\\
  x_f < p_{RSR}^k, y_f \geq q_{RSR}^k \text{, if } b(\alpha) \leq 0, c(\alpha) < 0, \\
   x_f < p_{RSR}^k, y_f < q_{RSR}^k \text{, if } b(\alpha) > 0, c(\alpha) < 0, \nonumber \\
  x_f \geq p_{RSR}^k, y_f < q_{RSR}^k \text{, if } b(\alpha) > 0, c(\alpha) \geq 0. \nonumber
   \end{empheq}

 \begin{figure}[t]
\centering
\subfloat[3D reachable space for different parameters]{
    \includegraphics[width=0.70\columnwidth]{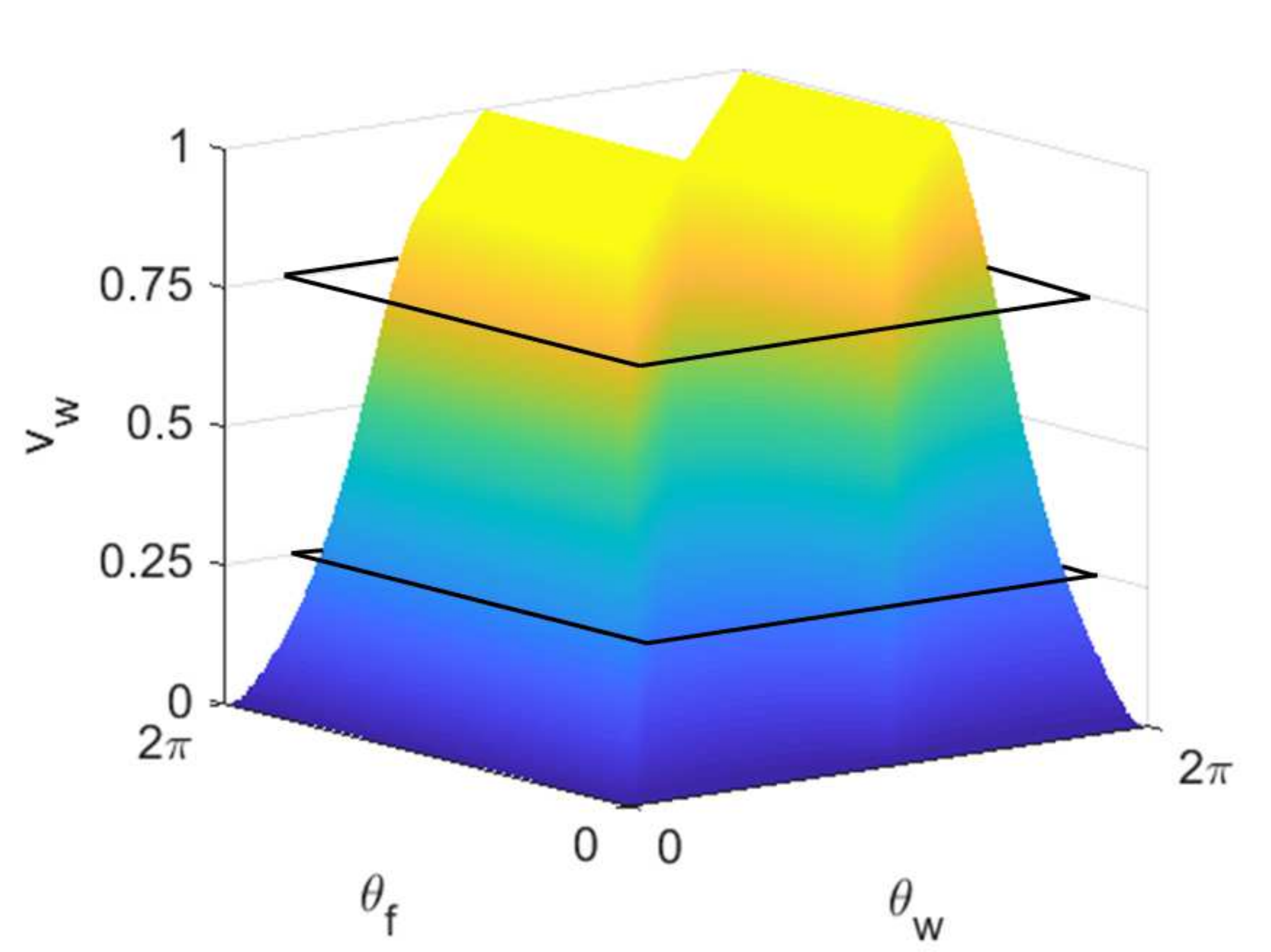}\label{fig:proposition1_3D_1}} \\
\subfloat[$v_w = 0.25$]{
     \includegraphics[width=.5\columnwidth]{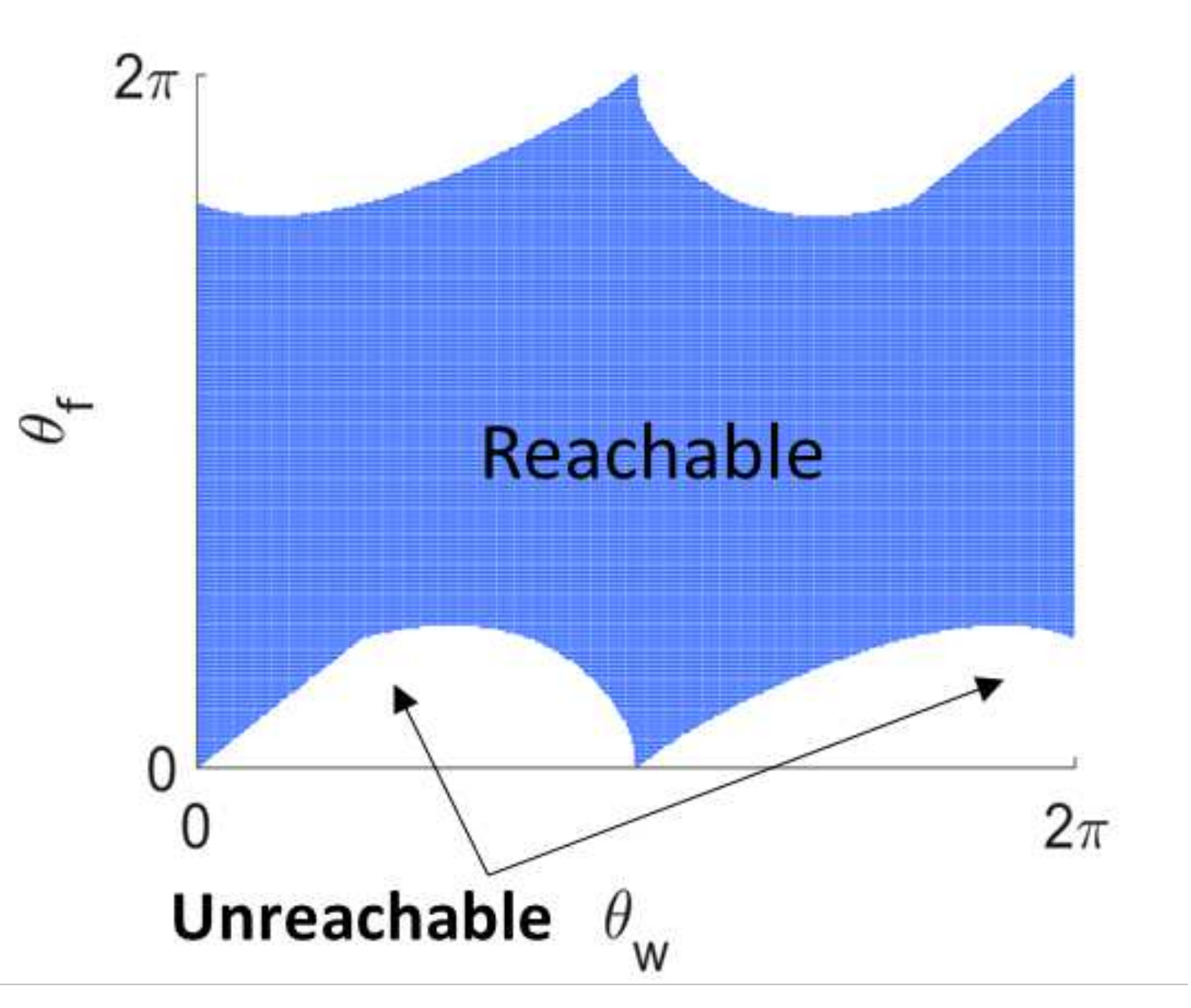}\label{fig:proposition1_3D_2}}
 \subfloat[$v_w = 0.75$]{
     \includegraphics[width=.5\columnwidth]{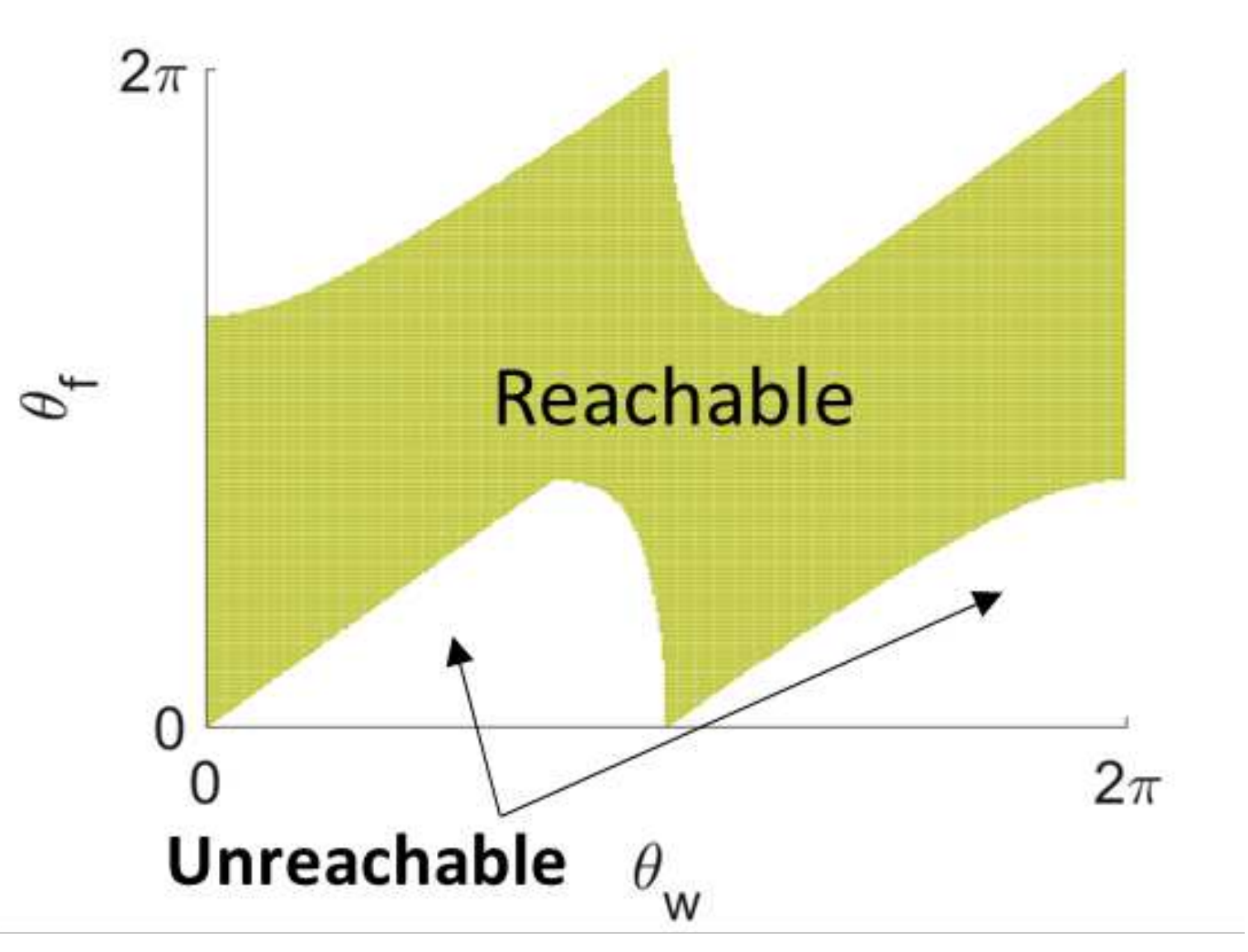}\label{fig:proposition1_3D_4}}
     \caption{The parameter space between $\theta_f \in [0, 2\pi)$, $\theta_w \in [0,2\pi)$ and $v_w = (0,1)$,  where full reachability is achieved.}\label{fig:proposition1_3D}
\end{figure}

\vspace{6pt}
The constraints in (\ref{eq:RSR_reachability}) are obtained by using the feasible range of $\beta \geq 0$ in (\ref{eq:RSR_beta_1}) and (\ref{eq:RSR_beta_2}). Again, these constraints define the quadrants of the coordinate frame with center at $\left(p^k_{RSR},q^k_{RSR}\right)$. For any given $\alpha$, (\ref{eq:RSR_reachability}) represents a reachability ray, and the goal $(x_f,y_f)$ is reachable if it lies on such ray. The  rotation of (\ref{eq:RSR_reachability}) is given as

\begin{empheq}[box=\widefbox]{align}\label{eq:RSR_slope}
    \omega^{k}_{RSR}(\alpha)=\atantwo \big(-b(\alpha),c(\alpha)\big)  \Mod{2\pi}, \ k \in \{-1,-2\}.
\end{empheq}

\vspace{12pt}
Now, we show a lemma that helps in constructing the reachability graphs using (\ref{eq:LSL_reachability}) and (\ref{eq:RSR_reachability}).

\vspace{0pt}
\begin{lem}
\label{lem:swipe}
As $\alpha$ increases from  $\alpha^k_{inf}$ to $\alpha^k_{sup}$, then for:
\begin{itemize}
    \item $LSL$ path type: ray (\ref{eq:LSL_reachability}) rotates anticlockwise about the center $\left(p_{LSL}^k,q_{LSL}^k\right)$, $\forall k \in \{0,1\}$.

\vspace{3pt}
\item  $RSR$ path type: ray (\ref{eq:RSR_reachability}) rotates clockwise about the center $\left(p_{RSR}^k,q_{RSR}^k\right)$, $\forall k \in \{-1,-2\}$.
\end{itemize}

\end{lem}
\begin{proof}
See Appendix~\ref{app:lemma1}.
\end{proof}

Lemma \ref{lem:swipe} implies that the reachable area for $LSL$ paths is obtained by rotating (\ref{eq:LSL_reachability}) about the center $\left(p_{LSL}^k,q_{LSL}^k\right)$, from $\omega^{k}_{LSL}(\alpha_{inf})$ to $\omega^{k}_{LSL}(\alpha_{sup})$, where $\alpha^k_{inf}$ to $\alpha^k_{sup}$ are the bounds of $\alpha$ (see  Table~\ref{table:range_conventional}) for a given $k$. Fig. \ref{fig:quadrants} shows the reachable area for $LSL$ paths obtained by this rotation. Note that there are different reachable areas for each $k$. Similarly, the reachable region for $RSR$ paths is obtained by rotating (\ref{eq:RSR_reachability}) from $\omega^{k}_{RSR}(\alpha_{inf})$ to $\omega^{k}_{RSR}(\alpha_{sup})$ for both its $k$ values.

\vspace{6pt}
\begin{rem}
Note that for simplicity of notation, we omit the superscript of $\alpha$ whenever it is used in the $\omega$ function, where it assumes the superscript of $\omega$.
\end{rem}

\vspace{6pt}
For further explanation, we introduce the concepts of \textit{Major Reachable Area} (MaRA) and \textit{Minor Reachable Area} (MiRA).

\begin{defn}[\textbf{MaRA}] For an $LSL$ ($RSR$) path type, MaRA is the larger of the reachable areas spanned by $k = 0$ or $1$ ($k = -1$ or $-2$).
\end{defn}
\begin{defn}[\textbf{MiRA}] For an $LSL$ ($RSR$) path type, MiRA is the smaller of the reachable areas spanned by $k = 0$ or $1$ ($k = -1$ or $-2$).
\end{defn}

\vspace{6pt}
\textit{Example}: Fig.~\ref{fig:reachability_example} shows an example of the construction of the reachability graph for $2\pi$-arc $LSL$ and $RSR$ path types. Here, the environment has a current of speed $v_w = 0.5$~m/s and direction $\theta_w = \pi/3$. The goal pose has the heading angle $\theta_f = 7\pi/4$, while its position $(x_f, y_f)$ is varied within $[-10, 10]$.

Figs.~\ref{fig:reachability_example}a and~\ref{fig:reachability_example}b show the MaRA ($k=0$) and MiRA ($k=1$) of the $LSL$ paths, respectively, which are obtained by rotating the  ray (\ref{eq:LSL_reachability}) by varying $\alpha$ from $\alpha_{inf}^k$ to $\alpha_{sup}^k$. The corresponding centers of rotation $(p_{LSL}^0,q_{LSL}^0) = (0.67, 2.67)$ and $(p_{LSL}^1,q_{LSL}^1) = (2.24, 5.39)$  are also shown. Fig.~\ref{fig:reachability_example}c shows the total reachable area of the $LSL$ paths obtained by combining the MaRA and MiRA from  Figs.~\ref{fig:reachability_example}a and \ref{fig:reachability_example}b, respectively. Clearly, the $LSL$ paths do not provide full reachability.

Similarly, Figs.~\ref{fig:reachability_example}d and \ref{fig:reachability_example}e show the MiRA ($k=-1$) and MaRA ($k=-2$)  of the $RSR$ paths, respectively, which are obtained by rotating the  ray (\ref{eq:RSR_reachability}) by varying $\alpha$ from $\alpha_{inf}^k$ to $\alpha_{sup}^k$. The corresponding centers of rotation $(p_{RSR}^{-1},q_{RSR}^{-1})= (0.90, 0.05)$ and $(p_{RSR}^{-2},q_{RSR}^{-2})= (2.47, 2.77)$ are also shown. Again, Fig.~\ref{fig:reachability_example}f shows the total reachable area of the $RSR$ path obtained by combining the MaRA and MiRA from Figs.~\ref{fig:reachability_example}d and \ref{fig:reachability_example}e, respectively. As seen, the $RSR$ paths also do not provide full reachability.

Finally, Fig.~\ref{fig:reachability_example}g shows the complete reachability graph using both $LSL$ and $RSR$ path types, which is obtained by  combining  Figs.~\ref{fig:reachability_example}c and \ref{fig:reachability_example}f. As seen in Fig.~\ref{fig:reachability_example}g, there is still some region that is unreachable, thus both $LSL$ and $RSR$ path types together also do not provide full reachability.

\begin{figure}[t]
    \centering
    \includegraphics[width=0.85\columnwidth]{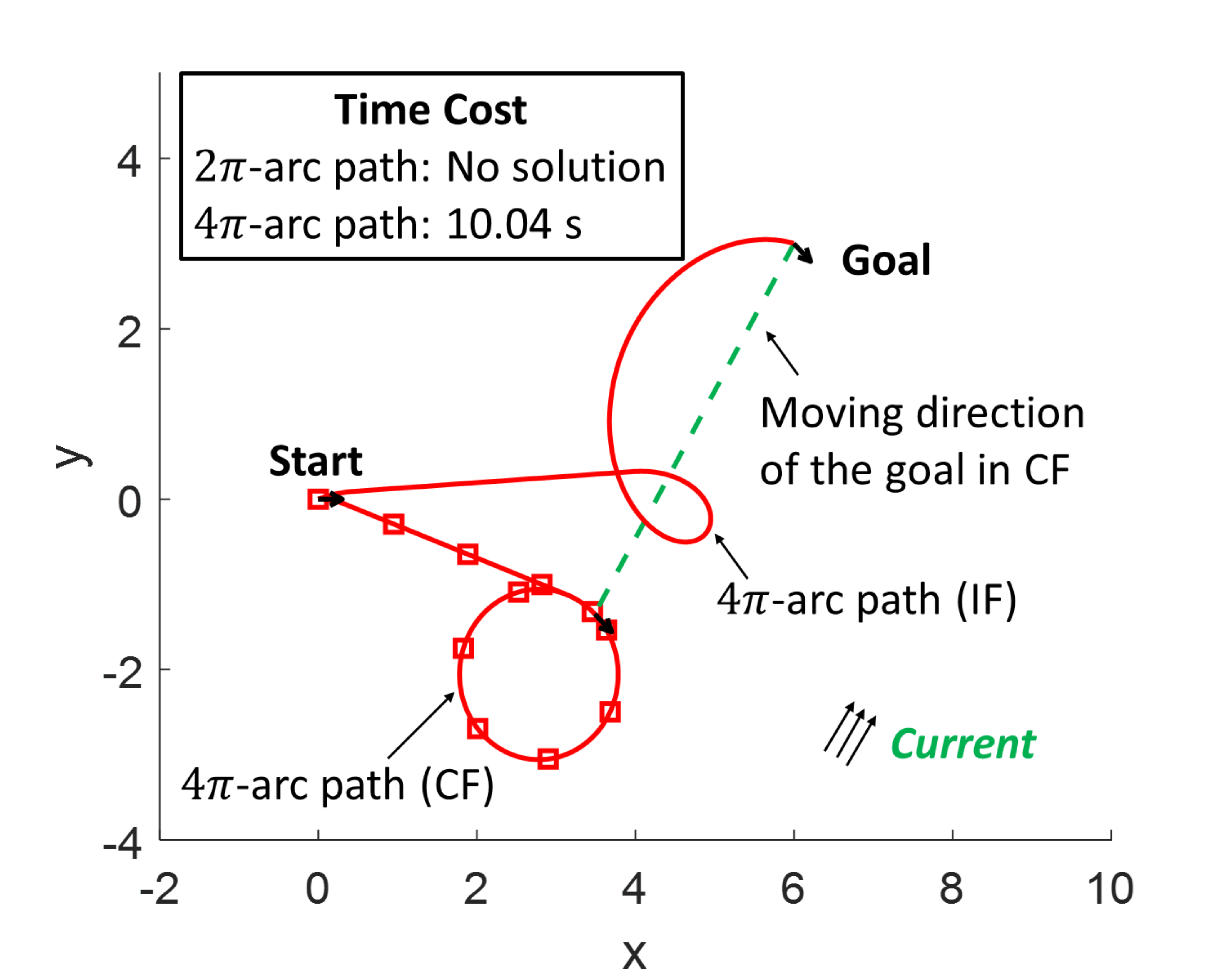}
    \caption{The optimal $4\pi$-arc paths in the IF and CF, while there is no feasible solution for $2\pi$-arc paths. The start pose $(x_0,y_0,\theta_0) = (0,0,0)$ and the goal pose $(x_f,y_f,\theta_f) = (6, 3, 7\pi/4)$. The optimal $4\pi$-arc path parameters are: $\alpha = 0.116\pi$, $\beta = 2.976$, $\gamma = 2.135\pi$.}\label{fig:prop1_example} \vspace{-6pt}
\end{figure}

\vspace{-6pt}
\subsection{Full Reachability Conditions for  the $2\pi$-arc Path Types}
\label{sec:full_reachability_conditions}
After acquiring the analytical expressions for generating the reachability graphs of the $2\pi$-arc $LSL$ and $RSR$ path types,  we now investigate the conditions under which these
paths provide full reachability.

Note that full reachability is achieved if the entire space is covered by atleast one of the following combinations:
\begin{enumerate}
\item Union of MaRA and MiRA of $LSL$, and/or
\item Union of MaRA and MiRA of $RSR$, and/or
\item Union of MaRA of $LSL$ and MiRA of $RSR$, and/or
\item Union of MaRA of $RSR$ and MiRA of $LSL$.
\end{enumerate}

\begin{rem} We show by Lemma \ref{lem:slopeproperty3} in Appendix \ref{app:reachability} that these four cases are sufficient for reachability analysis.
\end{rem}

For continuity of reading, the derivations of the full reachability conditions for the above four cases are presented in Appendix~\ref{app:reachability} and the results are summarized in Table~\ref{table:full_reachability}. If at some goal pose, all of the conditions in Table~\ref{table:full_reachability} are violated, then it is unreachable by $2\pi$-arc paths. Next, we visually verify the unreachable regions using a numerical validation.

\vspace{6pt}
\textbf{Numerical Validation}: The reachabilty conditions for $2\pi$-arc paths are shown in the last column of Table~\ref{table:full_reachability} in Appendix~\ref{app:reachability}. These reachability conditions only depend on parameters $\theta_f$, $\theta_w$ and $v_w$. Thus, we construct a 3D reachability graph by varying $\theta_f \in [0, 2\pi)$ and $\theta_w \in [0, 2\pi)$ in steps of $\pi/100$, and $v_w \in (0,1)$ in steps of $0.1$.
For any 3D parametric point, if at least one of the full reachability conditions is satisfied, then such point is colored, and the color varies with respect to  $v_w$, as shown in Fig.~\ref{fig:proposition1_3D_1}. In contrast, the white area indicates the parametric space where all the reachability conditions are violated, i.e., providing no feasible solutions. This validation illustrates that full reachability is not achieved by $2\pi$-arc $LSL$ and $RSR$ paths.

Figs.~\ref{fig:proposition1_3D_2} and ~\ref{fig:proposition1_3D_4} show the cross sections of Fig.~\ref{fig:proposition1_3D_1} at  $v_w= 0.25$~m/s and $v_w=0.75$~m/s, respectively. It is seen that a higher $v_w$ leads to a smaller reachable space.

Fig.~\ref{fig:prop1_example} shows a specific example where $2\pi$-arc path does not exist, but $4\pi$-arc path does. The start pose $(x_0,y_0,\theta_0) = (0,0,0)$, the goal pose $(x_f,y_f,\theta_f) = (6, 3, 7\pi/4)$, and the current moves at speed $v_w = 0.5$~m/s in the direction of $\theta_w = \pi/3$. It is seen that the turning angle of the second turn in the optimal $4\pi$-arc path has $\gamma = 2.135\pi>2\pi$, which drives the vehicle to circle around at the end so that it can meet with the exact goal heading with the help of external current.

\vspace{0pt}
\section{Theoretical Properties of $4\pi$-arc Paths}\label{sec:main_contents}
The previous section established that $2\pi$-arc $LSL$ and $RSR$ paths do not guarantee full reachability. This section presents the theoretical properties of $4\pi$-arc paths  which highlight their advantages over $2\pi$-arc paths in terms of: 1) full reachability, and 2) lower time costs, while requiring similar computational complexity. First, we present the concept of a dominant path type and show an example to motivate the above properties.

\begin{defn} [\textbf{Dominant Path Type}] For a given goal pose, a path type $LSL$ ($RSR$) is said to be dominant over $RSR$ ($LSL$), if it achieves a lower time cost to reach that goal pose.
\end{defn}

\vspace{6pt}
\textit{Example}: Figs.~\ref{fig:claim1_2pi} and \ref{fig:claim1_4pi} present the reachability plots of $2\pi$-arc and $4\pi$-arc paths, respectively. These are generated for an environment which has a current of speed $v_w = 0.5$~m/s and heading angle $\theta_w = \pi/3$. The coordinates of the goal pose $(x_f,y_f)$ are varied within $[-10,10]$. The two subplots of each figure correspond to two different goal pose directions $\theta_f \in \{5\pi/4, 7\pi/4\}$. A region is color-coded cyan (orange) if an $LSL$ ($RSR$) path exists and dominant over the $RSR$ ($LSL$) path type. The white color indicates that no feasible solution exists for either path type and the region is unreachable.

As seen in Fig.~\ref{fig:claim1_2pi}(2), for $\theta_f = 7\pi/4$, there exists a region which is unreachable for $2\pi$-arc paths. This implies that for any goal pose inside this region, no solutions exist for $\alpha$ and $\gamma$ within their feasible ranges defined in Table~\ref{table:range_conventional}. In contrast, as seen in Fig.~\ref{fig:claim1_4pi}(2), $4\pi$-arc paths achieve full reachability.

\begin{figure}[t]
    \centering
    \subfloat[Reachability graphs of the $2\pi$-arc paths for $\theta_f=5\pi/4$ and  $7\pi/4$.]{
        \includegraphics[width=0.90\columnwidth]{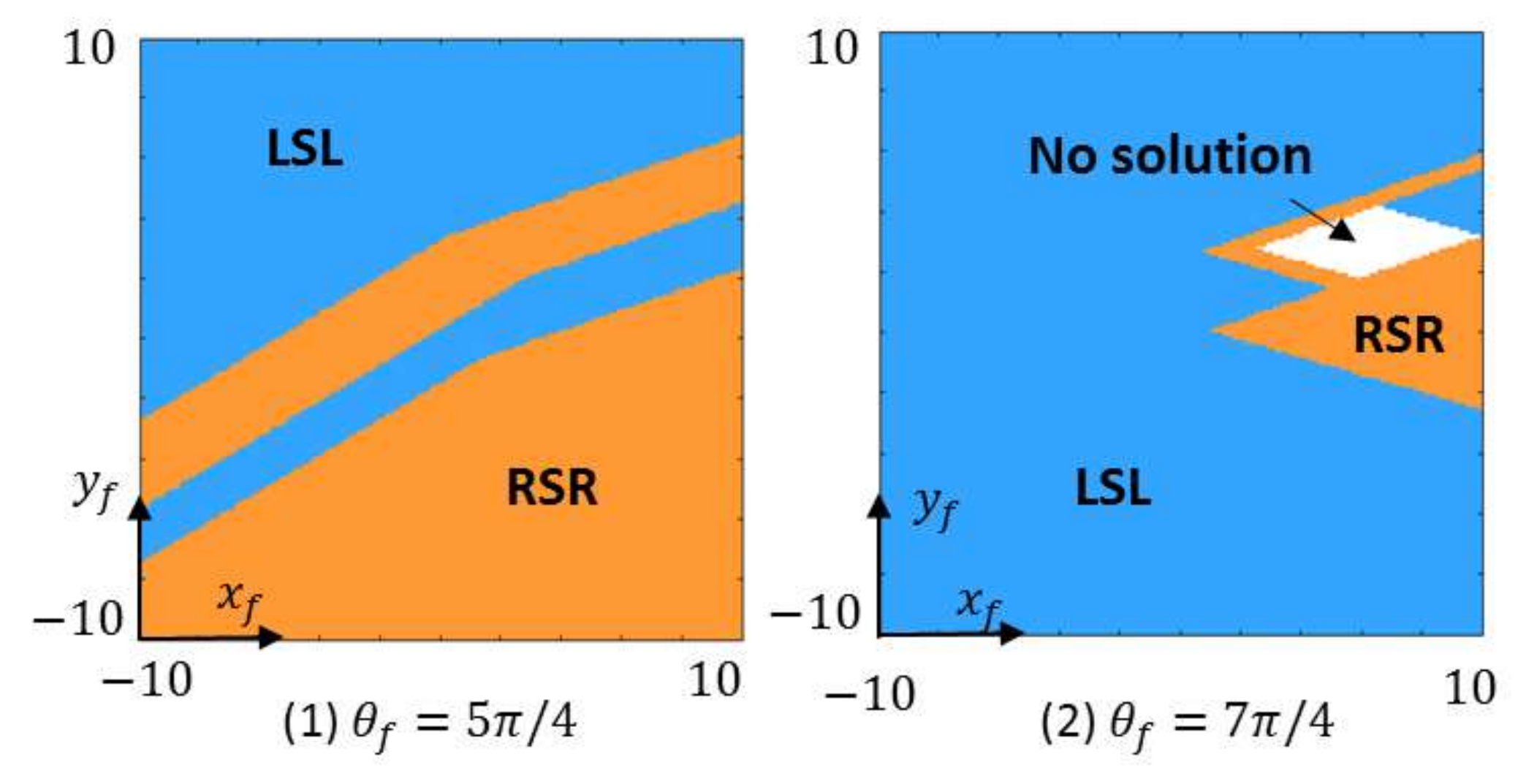}\label{fig:claim1_2pi}} \\ \vspace{-5pt}
    \subfloat[Reachability graphs of the $4\pi$-arc paths for $\theta_f=5\pi/4$ and  $7\pi/4$.]{
         \includegraphics[width=0.90\columnwidth]{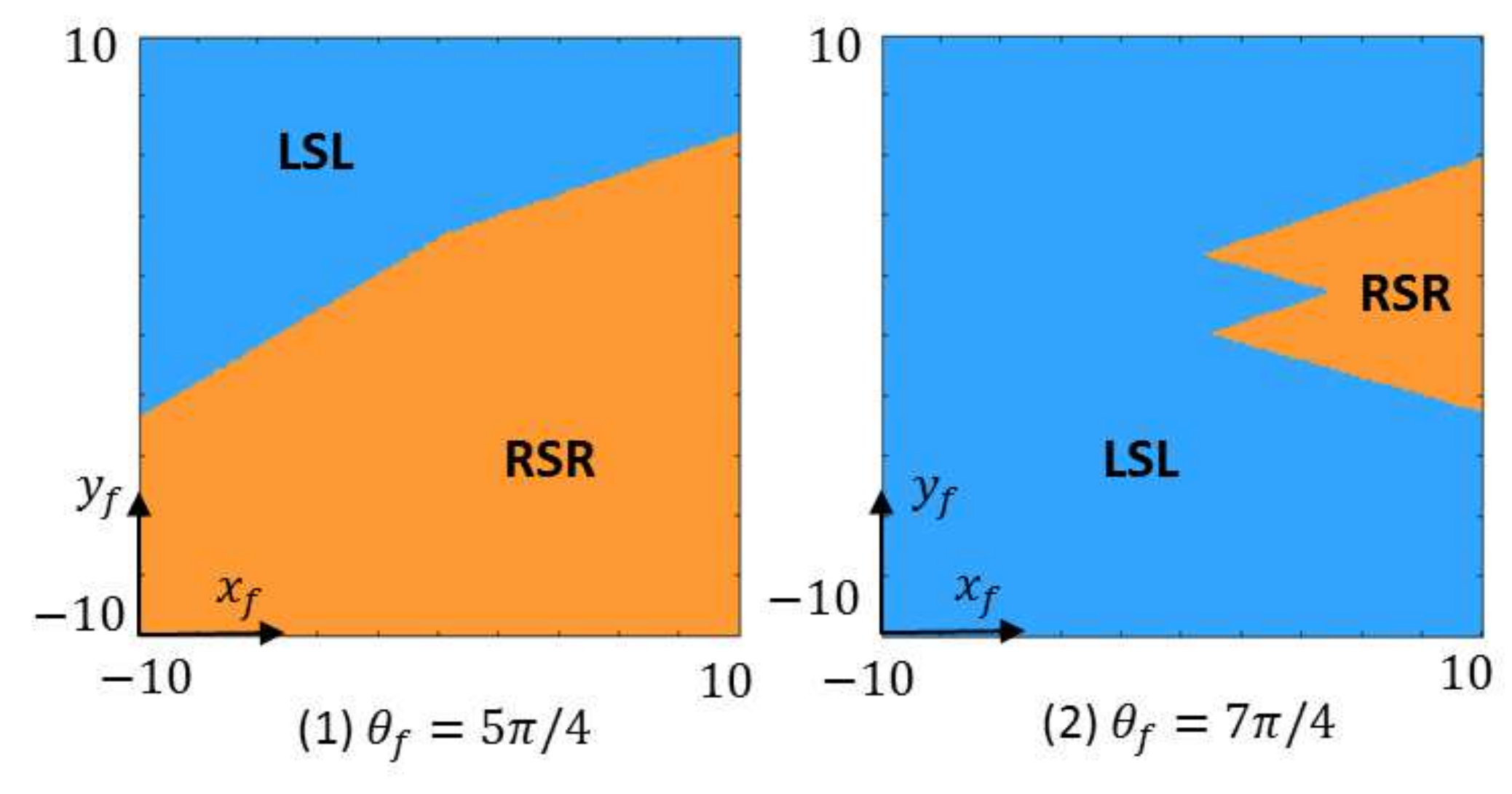}\label{fig:claim1_4pi}}
         \vspace{0pt}
         \caption{An example of reachability graphs for the $2\pi$-arc and $4\pi$-arc paths. The dominant of the $LSL$ (blue color) or $RSR$ (orange color) paths is shown in the corresponding area. White color indicates unreachable area.}
      \label{fig:claim1_both}\vspace{-6pt}
\end{figure}

Furthermore, the dominant path type (i.e., $LSL$ or $RSR$) for the same region could be different when using the $2\pi$-arc paths and $4\pi$-arc paths, as seen in Figs.~\ref{fig:claim1_2pi}(1) and \ref{fig:claim1_4pi}(1) corresponding to $\theta_f = 5\pi/4$. Since $4\pi$-arc solutions already include the $2\pi$-arc solutions, the above observation implies that there exist goal poses for which $4\pi$-arc paths can achieve even lower time costs as compared to the $2\pi$-arc paths.

\vspace{6pt}
\textbf{Roadmap of this Section:} In the following subsections, we present four theorems to highlight the theoretical properties of $4\pi$-arc $LSL$ and $RSR$ paths and compare them with the corresponding $2\pi$-arc paths. First, Theorem~\ref{claim1} proves that both the $LSL$ and $RSR$ $4\pi$-arc paths provide full reachability unlike the $2\pi$-arc paths. Then, Theorem~\ref{claim2} and Corollary~\ref{claim2_cor} show that the computation workload required to get a solution using the $4\pi$-arc paths is the same as that using the $2\pi$-arc paths. Next, Theorem~\ref{claim3} compares the optimality of $4\pi$-arc and $2\pi$-arc path solutions and shows that the optimal trajectory provided by $4\pi$-arc paths is either of shorter time or same as that provided by $2\pi$-arc paths. Finally, Theorem~\ref{rem:over_4pi} proves that $\alpha,\gamma\in[0,4\pi)$ is sufficient for optimality and increasing the range of these arc segments beyond $4\pi$ does not lead to a shorter time path.

\vspace{0pt}
\subsection{Full Reachability of $4\pi$-arc Paths}\label{4pifullreach}

The following theorem relates to the reachability of the $4\pi$-arc solutions for the $LSL$ and $RSR$ path types.

\begin{thm}[\textbf{Full reachability of $4\pi$-arc paths}]\label{claim1}
The $4\pi$-arc $LSL$ and $RSR$ paths individually provide full reachability.
\end{thm}
\begin{proof} Full reachability implies the existence of solution for any goal pose. We prove for $LSL$ and $RSR$ paths below.
\begin{itemize}
\item \textit{$4\pi$-arc $LSL$ paths}: Consider $k = 1$. From Table~\ref{table:range_conventional}, $\alpha_{inf} = 0$ and $\alpha_{sup} = 2\pi + \theta_f > 2\pi$.
Using Lemma~\ref{lem:swipe}, we construct the reachable space for $k = 1$ by rotating the  ray  (\ref{eq:LSL_reachability}) around $(p_{LSL}^1, q_{LSL}^1)$ by varying $\alpha$ from $0$ to $2\pi + \theta_f$. In this process, the ray (\ref{eq:LSL_reachability}) swipes in the anticlockwise direction from $\omega_{LSL}^1(0)$ to $\omega_{LSL}^1(2\pi + \theta_f)$. However, when $\alpha$ reaches $2\pi<2\pi + \theta_f$, the rotation of  ray (\ref{eq:LSL_reachability}) becomes $\omega_{LSL}^1(2\pi) = \omega_{LSL}^1(0) = \atantwo(w_y, 1+w_x) \Mod{2\pi}$, which implies that the ray comes back to the start again and continues swiping thereafter. This means that for $k = 1$, the whole space is covered and full reachability is obtained. Now consider $k = 2$. From Table~\ref{table:range_conventional}, $\alpha_{inf} = \theta_f$ and $\alpha_{sup} = 4\pi$. Following the same process as for the $k = 1$ case, one can  see that the swiped area for $k = 2$ also covers the whole area and full reachability is obtained. In summary, $4\pi$-arc $LSL$ paths guarantee full reachability. (Note: for $k = 0$ and $3$, the swiped area does not cover the whole space, hence they do not provide full reachability.)

\vspace{6pt}
\item \textit{$4\pi$-arc $RSR$ paths}: Consider $k = -2$. From Table~\ref{table:range_conventional},  $\alpha_{inf} = 0$ and $\alpha_{sup} = 4\pi - \theta_f > 2\pi$. Using Lemma~\ref{lem:swipe}, as $\alpha$ grows, the ray  (\ref{eq:RSR_reachability}) rotates around $(p_{RSR}^{-2}, q_{RSR}^{-2})$ in the clockwise direction from $\omega_{RSR}^{-2}(0)$ to $\omega_{RSR}^{-2}(4\pi - \theta_f)$. During this process, when $\alpha$ reaches $2\pi<4\pi - \theta_f$, the  rotation of  ray (\ref{eq:RSR_reachability}) becomes $\omega_{RSR}^{-2}(2\pi) = \omega_{RSR}^{-2}(0) = \atantwo(w_y, 1 + w_x) \Mod{2\pi}$, which implies that it comes back to the start again and continues swiping thereafter. This means that for $k = -2$, the whole space is covered and full reachability is obtained. Now consider $k = -3$. From Table~\ref{table:range_conventional}, $\alpha_{inf} = 2\pi-\theta_f$ and $\alpha_{sup} = 4\pi$. Following the same process as for the $k = -2$ case, one can see that the swiped area for $k = -3$ also covers the whole space and full reachability is obtained. In summary, $4\pi$-arc $RSR$ paths guarantee full reachability. (Note: for $k = -1$ and $-4$, the swiped area does not cover the whole space, hence they do not provide full reachability.)
\end{itemize}
Hence proved.
\end{proof}

\vspace{-12pt}

\subsection{Time Costs of $4\pi$-arc $LSL$ and $RSR$ Paths}\label{4pitimecost}
Now, we analyse the time costs of $4\pi$-arc $LSL$ and $RSR$ paths and compare them to the corresponding $2\pi$-arc paths.

Based on (\ref{eq:LSL_conditions}) and substituting $v = 1$, the time cost for an $LSL$ path type is given as

\begin{equation}
T = r(\alpha + \gamma) + \beta = 2k\pi r + r\theta_f + \beta.
\end{equation}

Similarly, based on (\ref{eq:RSR_conditions}), the time cost for an $RSR$ path type is given as

\begin{equation}
T = r(\alpha + \gamma) + \beta = -2k\pi r - r\theta_f + \beta.
\end{equation}

From this point on, let us denote $T_k$ and $\beta_k$ as the values of $T$ and $\beta$ for a given $k$, i.e., $T_k = 2k\pi r + r\theta_f + \beta_k$ for an $LSL$ path and $T_k = -2k\pi r - r\theta_f + \beta_k$ for an $RSR$ path.

\vspace{12pt}
\begin{thm}\label{claim2}
The following are true:
\begin{itemize}
    \item $T_0<T_{1}<T_2<T_{3}$, \ for $4\pi$-arc $LSL$ paths.
    \item $T_{-1}<T_{-2}<T_{-3}<T_{-4}$, \ for $4\pi$-arc $RSR$ paths.
\end{itemize}
\end{thm}
\begin{proof}
Let us denote $\Delta T_k$ as the difference in time cost $T_k$ between two consecutive $k$ values, i.e., for $LSL$ path type,
\begin{equation}\label{eq:LSL_delta_T}
\Delta T_k \triangleq T_{k+1} - T_k = 2\pi r + \beta_{k+1} - \beta_k, \ k = 0, 1, 2,
\end{equation}
and for $RSR$ path type,
\begin{equation}\label{eq:RSR_delta_T}
\Delta T_k \triangleq T_{k-1} - T_k =  2\pi r + \beta_{k-1} - \beta_k, \ k = -1, -2, -3.
\end{equation}

Consider $4\pi$-arc $LSL$ paths. To prove the theorem, we show that $\Delta T_k > 0, \forall k=0,1,2$. Fig.~\ref{fig:claim2_LSL} shows the feasible $4\pi$-arc $LSL$ paths in the CF, corresponding to $k$ (shown in solid blue) and $k+1$ (shown in solid red), to reach the goal pose $(x_f, y_f, \theta_f)$. These paths have the time costs $T_k$ and $T_{k+1}$, respectively. While these two paths share the same start pose, due to different travel times, the corresponding goal poses in the CF become $G_k = (x_f - w_x T_k, y_f - w_y T_k, \theta_f)$ and $G_{k+1} = (x_f - w_x T_{k+1}, y_f - w_y T_{k+1}, \theta_f)$, where $\norm{G_{k+1} - G_k} = \sqrt{w_x^2 \Delta T_k^2 + w_y^2 \Delta T_k^2} = v_w \abs{\Delta T_k}$.

Since an $LSL$ path is comprised of an $\alpha$ arc, a straight line and a $\gamma$ arc, one can equivalently combine the two arcs followed by the straight line to reach the same goal pose, as shown by the dotted line paths in Fig.~\ref{fig:claim2_LSL}, corresponding to $k$ (shown in dotted blue) and $k+1$ (shown in dotted red). According to (\ref{eq:LSL_conditions}), $\alpha + \gamma = 2k \pi + \theta_f$, so if $k$ is increased by $1$, it adds a full $2\pi$ rotation to this combined $\alpha$ and $\gamma$ arc. This implies that after combining these arcs, the blue and red dotted straight lines share the same start point $O_k \in \mathbb{R}^2$. Note that the solid straight lines are parallel to the corresponding dotted straight lines, with lengths $\beta_k$ and $\beta_{k+1}$, respectively.

\begin{figure}[!t]
    \centering
    \subfloat[$LSL$ path type]{
        \includegraphics[width=.48\columnwidth]{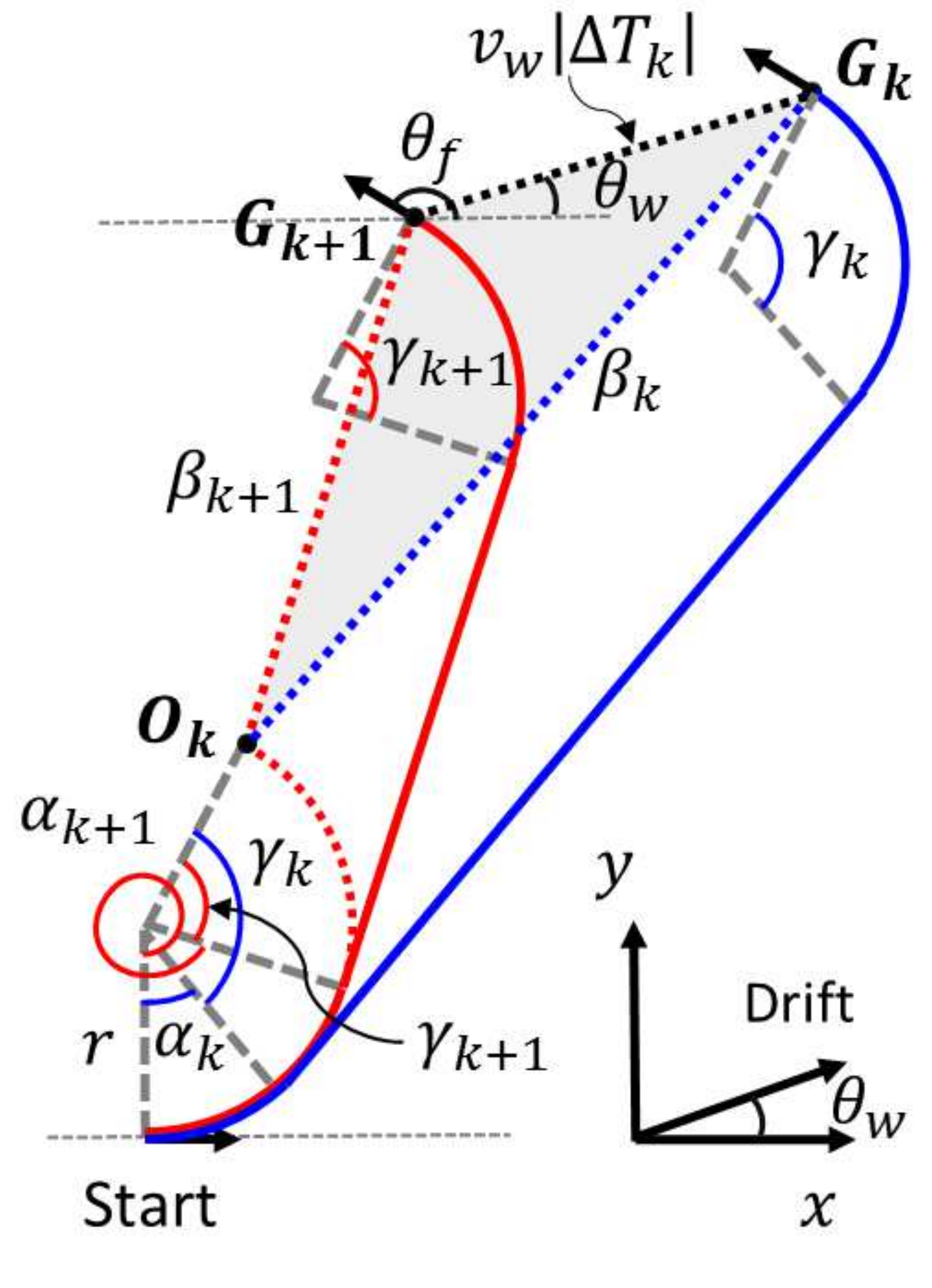}\label{fig:claim2_LSL}}
    \subfloat[$RSR$ path type]{
         \includegraphics[width=.5\columnwidth]{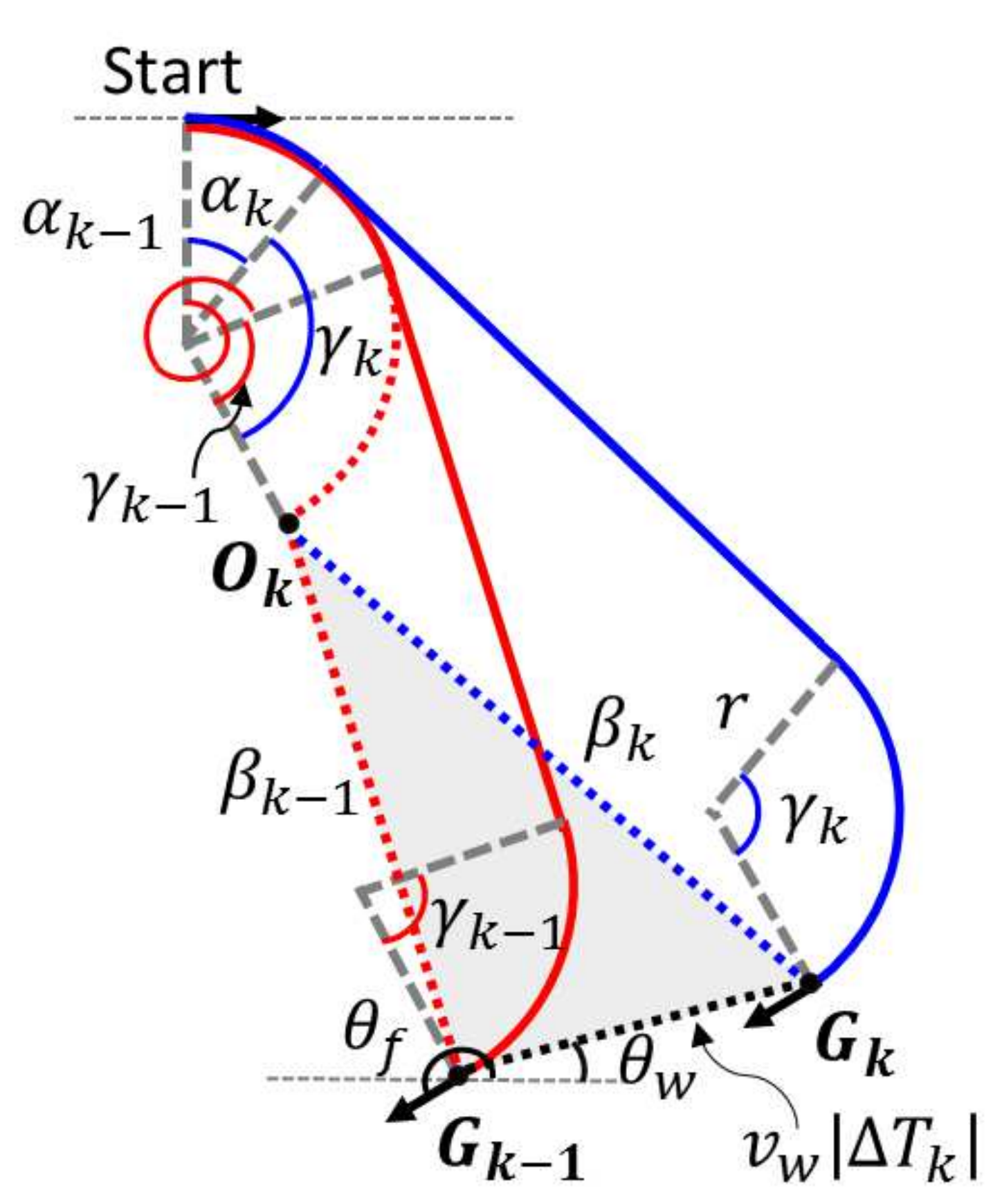}\label{fig:claim2_RSR}
         }\\
     \caption{Illustrative figures to show $\Delta T_k > 0, \forall k$ in Theorem~\ref{claim2}.}\label{fig:claim2}\vspace{-6pt}
\end{figure}

Now consider the triangle formed by $O_k, G_k$ and $G_{k+1}$, shown by the shaded region in Fig.~\ref{fig:claim2_LSL}, where
$\norm{O_k - G_{k}} = \beta_k$ and $\norm{O_k - G_{k+1}} = \beta_{k+1}$. Next, we consider three cases:

\vspace{6pt}
\begin{enumerate}
    \item $\Delta T_k > 0$: In this case, $\norm{G_{k+1} - G_k} = v_w \Delta T_k$. Using the triangle inequalities, we get $\abs{\beta_{k+1} - \beta_k} < v_w \Delta T_k$. By (\ref{eq:LSL_delta_T}), $\beta_{k+1} - \beta_k = \Delta T_k - 2\pi r$. Hence, $\abs{\Delta T_k - 2\pi r} < v_w \Delta T_k$ $\implies$ $\frac{2\pi r}{1+v_w} < \Delta T_k < \frac{2\pi r}{1-v_w}$. Note that if $O_k, G_k$ and $G_{k+1}$ fall on one line, then $\abs{\beta_{k+1} - \beta_k} = v_w \Delta T_k$, then $\Delta T_k = \frac{2\pi r}{1+v_w}$ or $\frac{2\pi r}{1-v_w}$. Therefore, the feasible range of $\Delta T_k$ is

    \vspace{-3pt}
    \begin{equation}\label{eq:Delta_Tk_range}
        \boxed{\Delta T_k \in \bigg[\frac{2\pi r}{1+v_w}, \frac{2\pi r}{1-v_w} \bigg].}
    \end{equation}

\vspace{6pt}
\item $\Delta T_k < 0$: In this case, $\norm{G_{k+1} - G_k} = -v_w \Delta T_k$. Then, based on the triangle inequalities, $\abs{\beta_{k+1} - \beta_k} < -v_w \Delta T_k$. Again substituting $\beta_{k+1} - \beta_k = \Delta T_k - 2\pi r$ from (\ref{eq:LSL_delta_T}), we get $\frac{2\pi r}{1-v_w} < \Delta T_k < \frac{2\pi r}{1+v_w}$. However, since $0 < v_w < 1$, this inequality is invalid. Thus, $\Delta T_k <0$ is impossible.

\vspace{6pt}
\item $\Delta T_k = 0$: In this case, $\norm{G_{k+1} - G_k} = 0$. Then, $\abs{\beta_{k+1} - \beta_k} = 0$ $\implies$ $\Delta T_k - 2\pi r=0$ $\implies$ $\Delta T_k = 2\pi r$, which is a contradiction, hence $\Delta T_k = 0$ is impossible.
\end{enumerate}

\vspace{6pt}
Thus, $\Delta T_k  > 0, \forall k$, and its bounds are given in (\ref{eq:Delta_Tk_range}). Similarly, for $4\pi$-arc $RSR$ paths, the bounds of $\Delta T_k$ can be derived using Fig.~\ref{fig:claim2_RSR}, leading to the same bounds and the derivation is omitted here. Hence proved.
\end{proof}

The following corollary shows that in order to obtain the minimum-time solutions using $4\pi$-arc paths, it is sufficient to use $k=\{0, 1\}$ for $LSL$ path type and $k=\{-1, -2\}$ for $RSR$ path type and the remaining $k$ values are not needed.

\vspace{6pt}
\begin{cor}\label{claim2_cor}
A minimum-time solution for the $4\pi$-arc paths can be obtained by using
\begin{itemize}
\item $k \in \{0,1\}$ for $LSL$ paths and
\item $k \in \{-1,-2\}$ for $RSR$ paths. \end{itemize}
\end{cor}
\begin{proof}
Theorem~\ref{claim2} implies that based on time costs, the preferred solutions follow the order $k=0,1,2,3$ for $LSL$ paths and $k=-1,-2,-3,-4$ for $RSR$ paths. Theorem~\ref{claim1} suggests that for $LSL$ paths, $k = 0$  solutions do not provide full reachability; however full reachability can be achieved by $k=1$ solutions. Similarly, for $RSR$ paths, $k = -1$ solutions do not provide full reachability; however full reachability can be achieved by $k=-2$ solutions. Thus, in order to get full reachability and to obtain minimum-time paths, one must solve only for $k \in \{0,1\}$ for $LSL$ paths, and $k \in \{-1,-2\}$ for $RSR$ paths. Hence proved.
\end{proof}

\begin{rem}Corollary~\ref{claim2_cor} implies that the computation workload required to get a solution using the $4\pi$-arc paths is the same as that using the $2\pi$-arc paths.
\end{rem}

\begin{cor}\label{claim2_cor2}
A minimum-time $4\pi$-arc $LSL$ or $RSR$ solution must satisfy $\alpha$ + $\gamma$ < 4$\pi$.
\end{cor}
\begin{proof}
Using Corollary~\ref{claim2_cor} and that $\theta_f < 2\pi$, substitute $k=1$ into (\ref{eq:LSL_conditions}) and $k=-2$ into (\ref{eq:RSR_conditions}), one can easily get the result. Hence proved.
\end{proof}

\begin{rem}\label{rem:parameters}
As seen from Table \ref{table:range_conventional}, the feasible ranges of parameters $\alpha$ and $\gamma$ for the $4\pi$-arc $LSL$ ($RSR$) paths for $k = 0$ ($k = -1$) are the same as those of the corresponding $2\pi$-arc paths. However, for $k = 1$ ($k = -2$), the parameter ranges for $4\pi$-arc $LSL$ ($RSR$) paths form supersets of the corresponding ranges of the $2\pi$-arc paths.
\end{rem}

\vspace{6pt}
\begin{thm}\label{claim3}
The time costs of $4\pi$-arc path solutions are lower than or same as those of the $2\pi$-arc path solutions.
\end{thm}
\begin{proof}
First, consider the case when both $2\pi$-arc $LSL$ and $RSR$ solutions exist for a given goal pose.  Remark~\ref{rem:parameters} indicates that any valid $2\pi$-arc path solution is also a valid $4\pi$-arc path solution. Hence, in this case the time cost of  $4\pi$-arc path solution is the same as that of the $2\pi$-arc path solution.

Second, consider the case when neither of the $2\pi$-arc $LSL$ and $RSR$ solutions exist for a given goal pose. In this case, Theorem~\ref{claim1} guarantees that $4\pi$-arc $LSL$ and $RSR$ solutions exist for that goal pose.

Third, consider the case when only one of the $2\pi$-arc $LSL$ or $RSR$ path solution exists for a given goal pose, i.e., the other path type does not provide a solution. Thus, the dominant solution is the only existing path type. However, from Theorem~\ref{claim1}, for $4\pi$-arc paths both $LSL$ and $RSR$ paths exist and the dominant solution is selected from these two path types with the minimum time cost. Thus, due to the existence of an extra solution provided by the $4\pi$-arc paths, the time cost of the dominant path could be better than or same as that of the single solution provided by the $2\pi$-arc paths. The examples below validate this case. Hence proved.
\end{proof}

\begin{figure}[t]
    \centering
    \subfloat[Cost map of $2\pi$-arc paths.]{
        \includegraphics[width=0.215\textwidth]{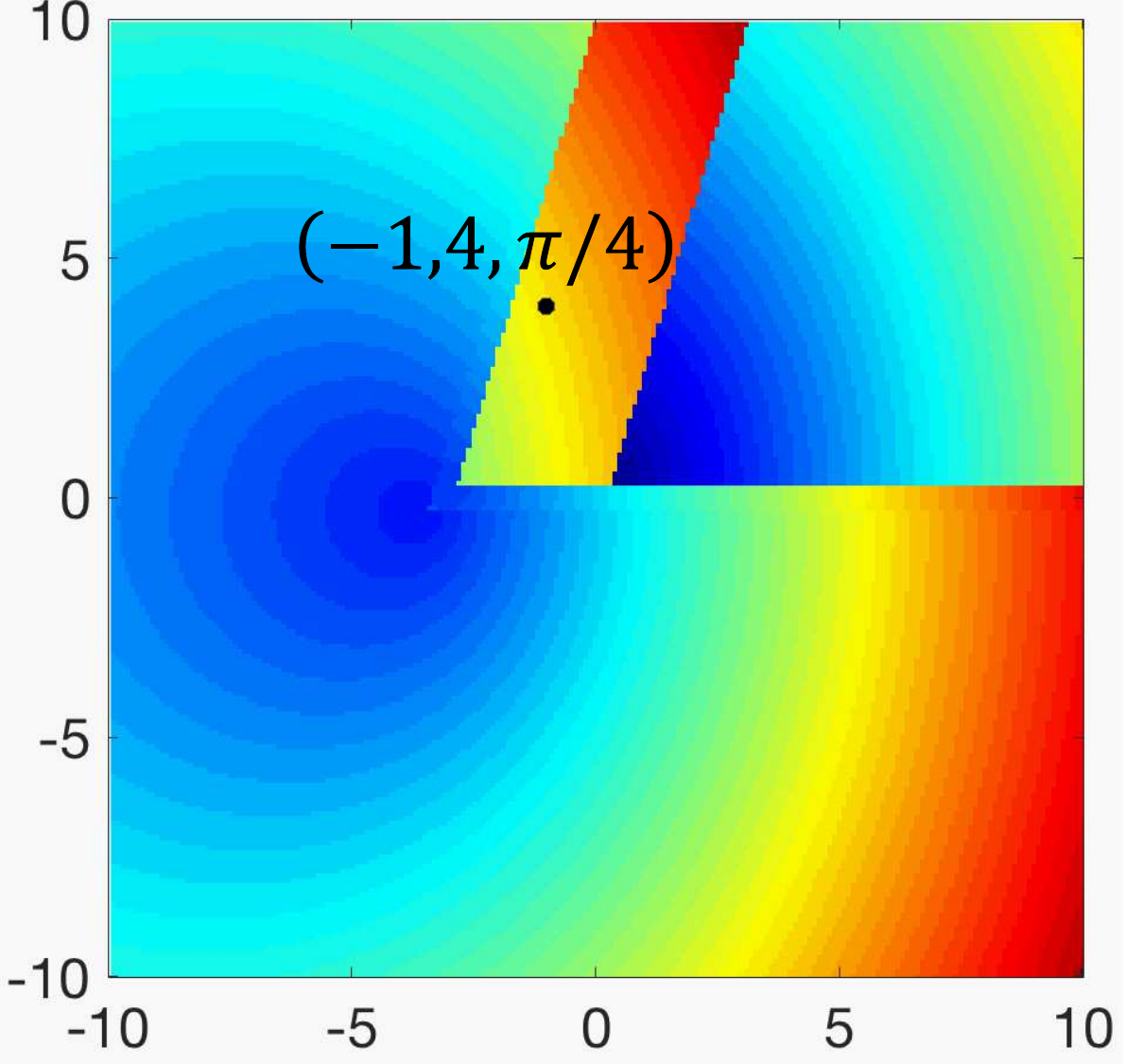}\label{fig:example_2pi}}
    \subfloat[Cost map of $4\pi$-arc paths.]{
        \includegraphics[width=0.245\textwidth]{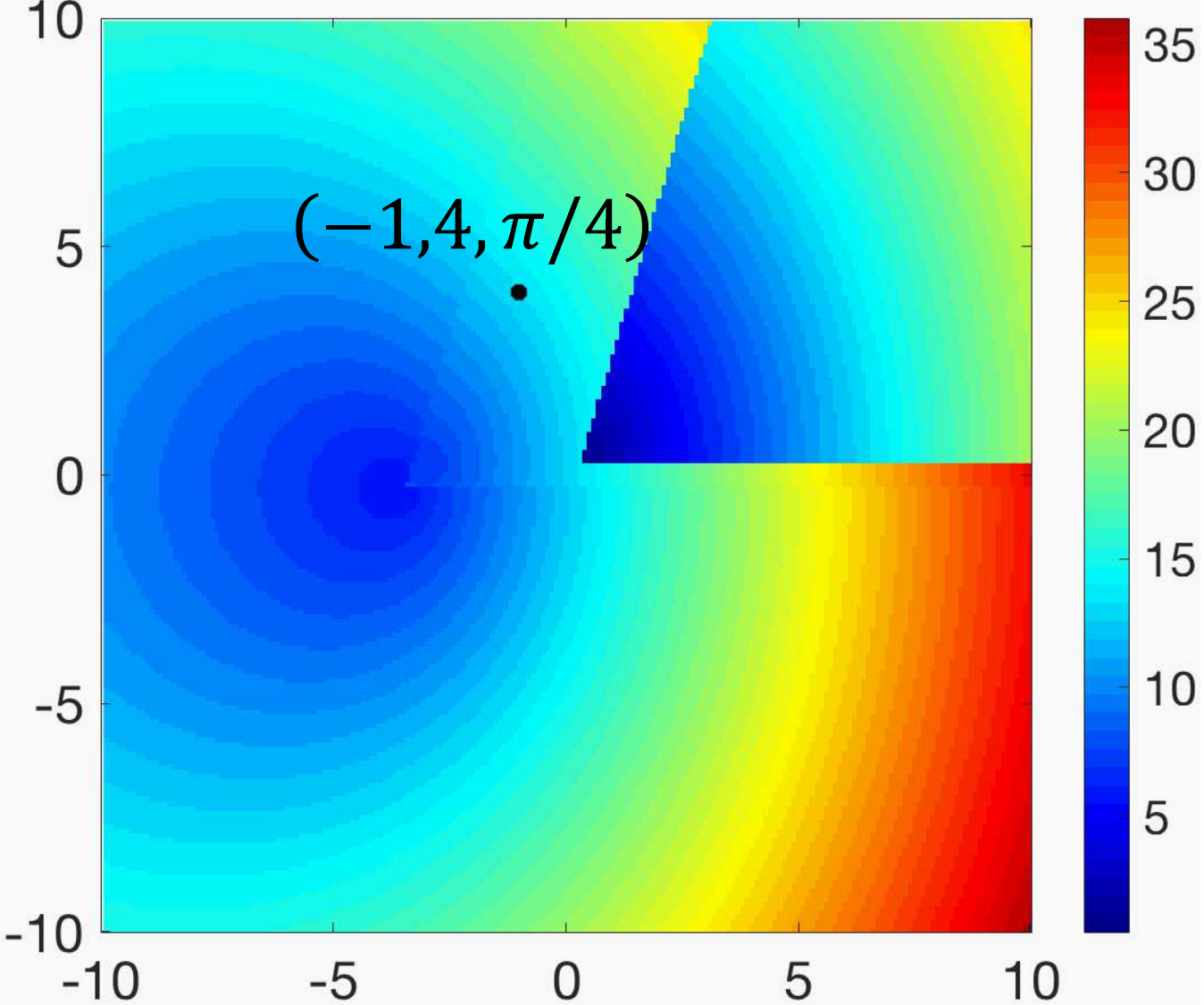}\label{fig:example_4pi}} \\
    \subfloat[The $2\pi$-arc and $4\pi$-arc path solutions in the IF and  CF. The start pose $(x_0, y_0,\theta_0)=(0,0,0)$ and the goal pose $(x_f, y_f, \theta_f)=(-1, 4, \pi/4)$. The optimal $2\pi$-arc path has: $\alpha=1.890\pi, \beta=12.691$ and $\gamma=1.860\pi$; and the optimal $4\pi$-arc path has: $\alpha=0.206\pi, \beta=6.143$ and $\gamma=2.044\pi$.]{
        \includegraphics[width=0.48\textwidth]{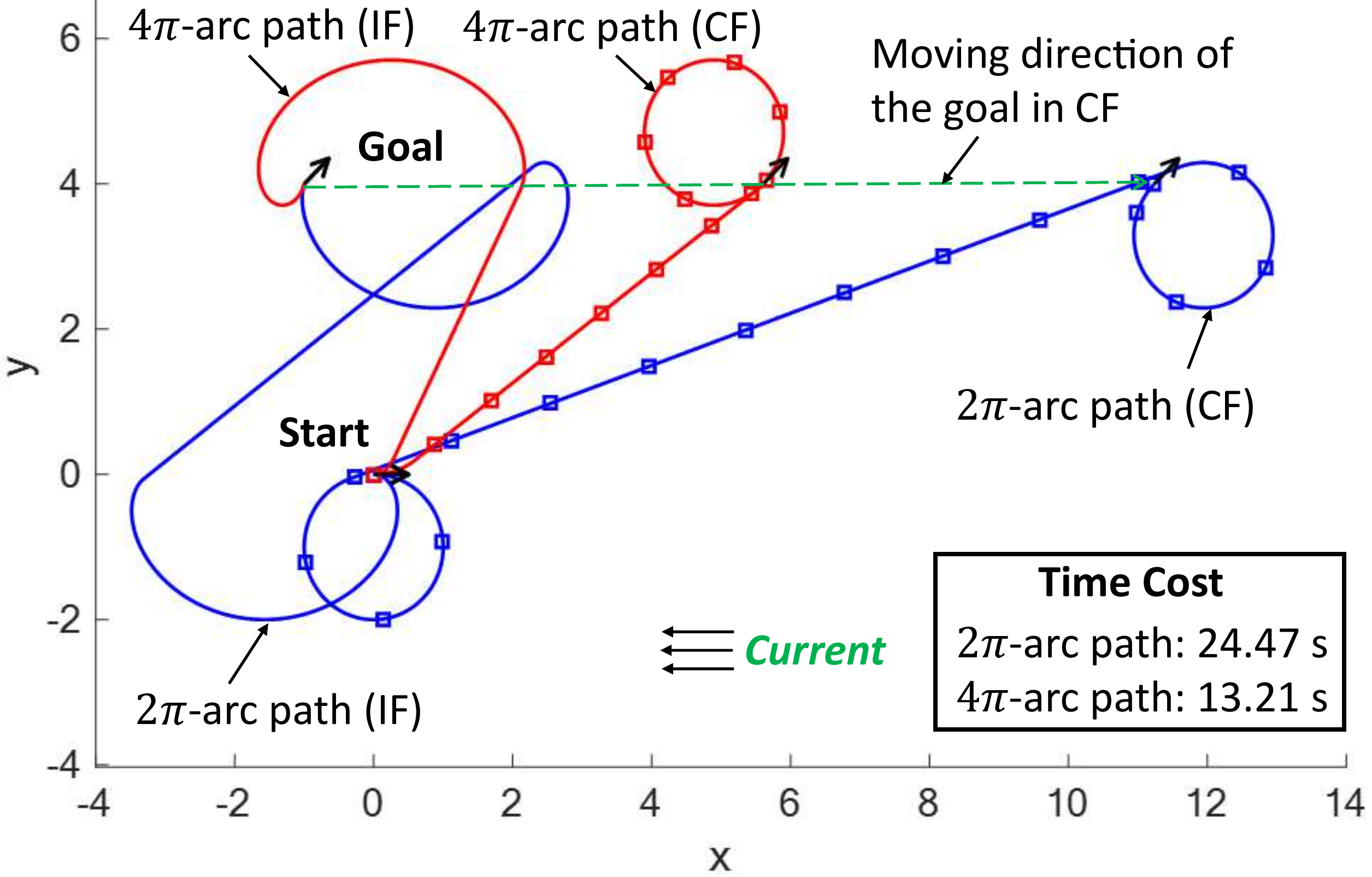}\label{fig:example_paths}}
    \caption{An example to illustrate the result of Theorem~\ref{claim3} that the $4\pi$-arc paths provide faster solutions than the $2\pi$-arc paths.} \vspace{-6pt}
     \label{fig:claim3_example1}
\end{figure}

\vspace{0pt}
\textit{Example}:
We show an example where the $4\pi$-arc paths provide faster (i.e., lower time cost) solutions as compared to the $2\pi$-arc paths. We first construct the time cost map for a fixed set of $\theta_f$, $v_w$ and $\theta_w$, where each $(x_f, y_f)$ is assigned the time cost of the dominant path between $LSL$ and $RSR$ paths.

Fig.~\ref{fig:claim3_example1} shows the example generated for an environment with current of $v_w = 0.5$~m/s and $\theta_w = \pi$. For constructing the time cost map, the goal poses are varied within $x_f, y_f \in [-10,10]$~m with a fixed heading angle $\theta_f = \pi/4$. Figs.~\ref{fig:example_2pi} and \ref{fig:example_4pi} show the time cost maps for $2\pi$-arc paths and $4\pi$-arc paths, respectively. The color code indicates the value of the time cost. Clearly, there exist many goal poses where $4\pi$-arc paths provide significantly lower time costs.

Next, we pick a goal pose where $4\pi$-arc paths provide a lower time cost, say $(x_f,y_f,\theta_f) = (-1, 4, \pi/4)$. Then, we draw the optimal $2\pi$-arc and $4\pi$-arc paths in the IF and the CF, as shown in Fig.~\ref{fig:example_paths}. The $2\pi$-arc path follows the $RSR$ path type, and requires a total time cost of $24.47$~s. In comparison, the $4\pi$-arc path follows the $LSL$ path type and the total time cost is reduced to $13.21$~s. This is because on the $2\pi$-arc path, the vehicle has to travel a longer straight-line segment that is almost in an opposite direction to the current, hence its actual speed in the inertial frame becomes slower. On the other hand, the $4\pi$-arc path first makes a small left turn, followed by a much shorter straight-line segment; then, it starts circling for over $2\pi$ while letting the current help it to reach the goal.

\vspace{6pt}
\begin{thm}\label{rem:over_4pi}
The time cost $T$ cannot be reduced further by extending the ranges of arc segments ($\alpha$ and $\gamma$) over $4\pi$.
\end{thm}
\begin{proof}
Suppose the ranges of $\alpha$ and $\gamma$ are defined over $[0, 2n\pi)$, where $n > 2$ and $n \in \mathbb{N}^+$. Then, using the same procedure as described in Section~\ref{sec:method}, we get a larger set of feasible values of $k$, s.t. for $LSL$ paths, $k \in \{0,1,\ldots,2n-1\}$, and for $RSR$ paths, $k \in \{-1, -2,\ldots, -2n\}$.

Then, one can derive the feasible ranges for $\alpha$ and $\gamma$. Consider a $2n\pi$-arc $LSL$ path, where $\alpha \in [0, 2n\pi)$ and $\gamma \in [0, 2n\pi)$. We examine only $k=0,1$ cases as necessary.

\begin{itemize}
\item $k = 0$ (i.e., $\alpha + \gamma = \theta_f < 2\pi$): Now, $\gamma \geq 0$ $\implies$ $\alpha \leq \theta_f$.
Similarly, $\alpha \geq 0$ $\implies$ $\gamma \leq \theta_f$. Thus, the feasible range for both $\alpha$ and $\gamma$ is $[0,\theta_f]$.

\item $k = 1$ (i.e., $\alpha + \gamma = 2\pi + \theta_f< 4\pi$): Again, $\gamma \geq 0$ $\implies$ $\alpha \leq 2\pi + \theta_f$. Similarly, $\alpha \geq 0$ $\implies$ $\gamma \leq  2\pi + \theta_f$. Thus, the feasible range for both $\alpha$ and $\gamma$ is $[0, 2\pi + \theta_f]$.
\end{itemize}

The above analysis indicates that for  $2n\pi$-arc $LSL$ paths, if $n > 2$, the feasible ranges of $\alpha$ and $\gamma$ for $k = 0, 1$ are the same to the corresponding ones for $4\pi$-arc $LSL$ paths, as presented in Table~\ref{table:range_conventional}. Similarly, one can verify that for $2n\pi$-arc $RSR$ paths, if $n > 2$, the feasible ranges of $\alpha$ and $\gamma$ for $k = -1, -2$ are also the same to the corresponding ones for $4\pi$-arc $RSR$ paths.

Since the feasible ranges of $\alpha$ and $\gamma$ for $2n\pi$-arcs are the same as those for $4\pi$-arc paths, by Theorem~\ref{claim1} full reachability is achieved using $k = 0, 1$ for $LSL$ paths and $k = -1, -2$ for $RSR$ paths. Further, by Theorem~\ref{claim2}, $\Delta T_k > 0, \forall k$. Therefore, for $n > 2$, we only need to search over $k = 0,1$ for $LSL$ paths and $k = -1,-2$ for $RSR$ paths to get the minimum-time path. This implies that the time cost $T$ is not reduced by extending the feasible ranges of $\alpha$ and $\gamma$ over $4\pi$. Hence proved.
\end{proof}

\vspace{0pt}
\section{Results and Discussion}\label{sec:results}

This section presents the results of the proposed approach, which uses the $4\pi$-arc $LSL$ and $RSR$ paths, in comparison to the Dubins approach, which uses the six $2\pi$-arc paths. We discuss the performance of these two approaches first in an environment with static current and then in an environment with dynamically changing current. We conduct Monte Carlo simulations as needed for statistical performance evaluation. The simulations were done on a computer with $2.4$~GHz and $8$~GB RAM. In order to obtain a solution using the Dubins approach, the transcendental functions are solved using the function \textit{fsolve} in MATLAB.  On average, the Dubins approach took $\sim8.72$~s to get a solution with $100$ initial guesses, while the $4\pi$-arc paths approach took only $\sim0.64$~ms which is orders of magnitude faster than that of the Dubins computation.

\subsection{Comparison of $4\pi$-arc $LSL$ and $RSR$ solutions with Dubins solutions in a static current environment}
\label{app:dubins_comparison}
\vspace{6pt}
First, we considered an environment with a static current where the planning is done offline. This comparative study is presented using two metrics: a) the solution quality (i.e., the travel time cost) and b) the total time cost (i.e., the offline computation time cost plus the travel time cost).

\textit{Simulation Setup}: The start pose is fixed at $(x_0, y_0, \theta_0)=(0,0,0)$. Then, $80$ different goal positions are distributed uniformly on the boundaries of concentric squares at a distance of $R=\{5, 10, 50, 100, 200\}$~m around the origin. For each goal position, $6$ different heading angles $\theta_f \in \{ \frac{m\pi}{3}, m = 0,\ldots 5\}$ are considered. This leads to a total of $480$ goal poses. The vehicle and  current speeds are taken to be $v=1$~m/s and $v_w=0.5$~m/s, respectively, where $6$ different current heading angles $\theta_w \in \{ \frac{m\pi}{3}, m = 0,\ldots 5\}$ are considered, thus leading to a total number of $2880$ runs.

For each run, the travel time cost and computation time cost are obtained for the two approaches. Fig~\ref{fig:comparison_Dubins} shows the savings obtained with the proposed $4\pi$-arc path solutions as compared to the Dubins solutions. Fig~\ref{fig:comparison_Dubins_travel} shows the savings in travel time, computed as $T_{Dubins}$ - $T_{4\pi}$, where $T_{Dubins}$ and $T_{4\pi}$ refer to the travel time costs of Dubins paths and $4\pi$-arc paths, respectively. As seen in the figure, in more than $50\%$ of the cases, the travel time costs of $4\pi$-arc path solutions match those of the Dubins solutions. Although the performance of Dubins paths is better than the $4\pi$-arc paths for the remaining cases, the travel time cost difference is not that significant.

Fig.~\ref{fig:comparison_Dubins_total} shows the total time cost obtained by adding the computation time costs taken by the two approaches to their respective travel time costs. It is seen that in more than $90\%$ of the cases the total time of the $4\pi$-arc solutions is lower than that of the Dubins solutions; thus, $4\pi$-arc solutions yield a superior performance upon considering the computation times.

Based on these trends, it is observed that although Dubins solutions are suitable for applications requiring offline planning, they do not provide significant advantage over the $4\pi$-arc $LSL$ and $RSR$ solutions in terms of travel time costs. Furthermore, when computation times are added then Dubins solutions provide worse total time costs in a significant majority of cases. Moreover, as discussed in Section \ref{changingcurrents}, for applications requiring online planning in dynamic current environments, the high computation times of Dubins solutions cause significant vehicle drifts, thus, resulting in longer sub-optimal trajectories which sometimes do not even converge to the goal pose. In such situations, $4\pi$-arc paths lead to faster and reliable solutions with negligible drifts allowing the vehicle to reach the goal pose precisely in shorter times.

\begin{figure}[t]
    \centering
    \subfloat[Savings in travel time: $T_{Dubins}$ - $T_{4\pi}$.]{
        \includegraphics[width=0.50\columnwidth]{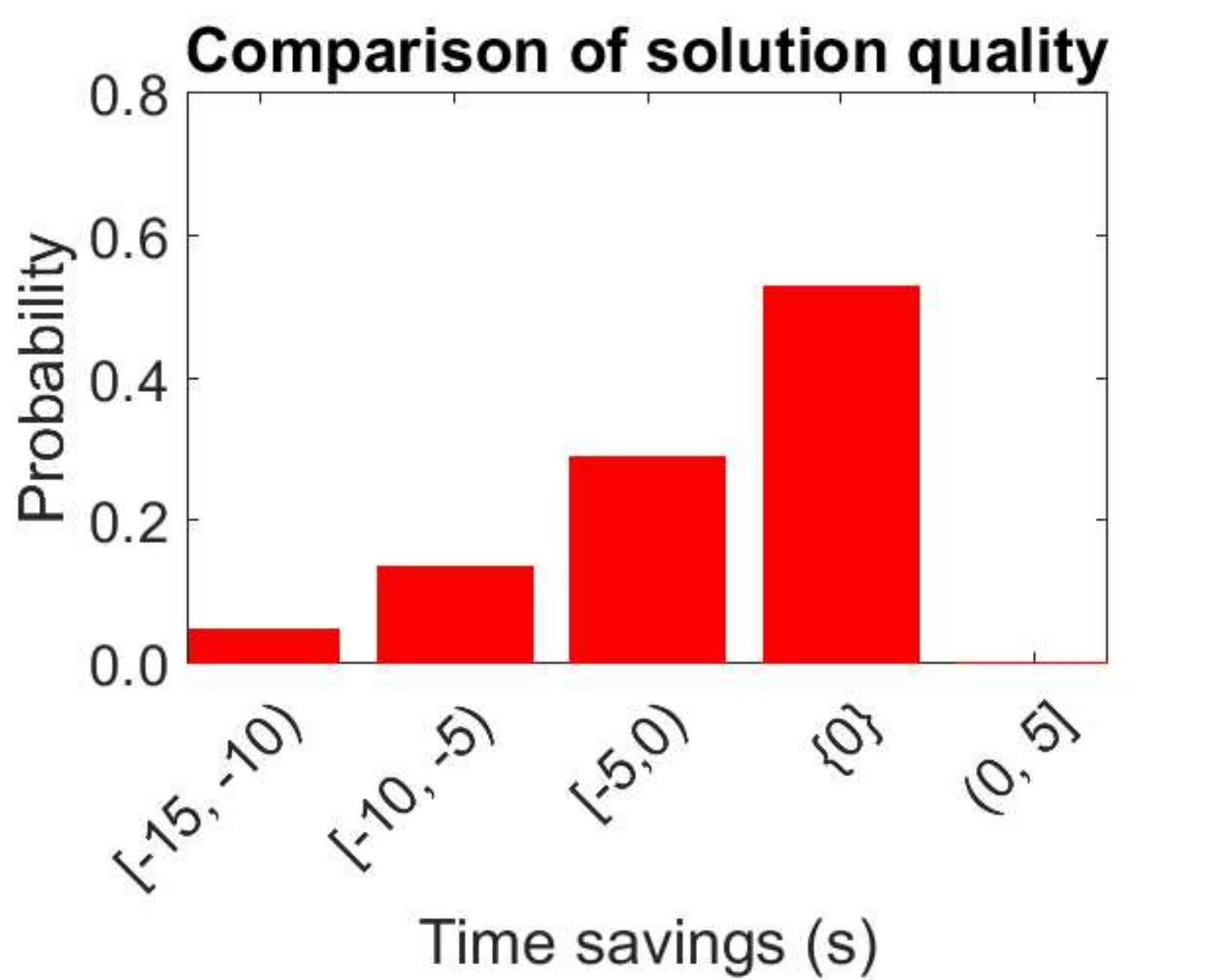}\label{fig:comparison_Dubins_travel}}
    \subfloat[Savings in total time after including computation time.]{
        \includegraphics[width=0.50\columnwidth]{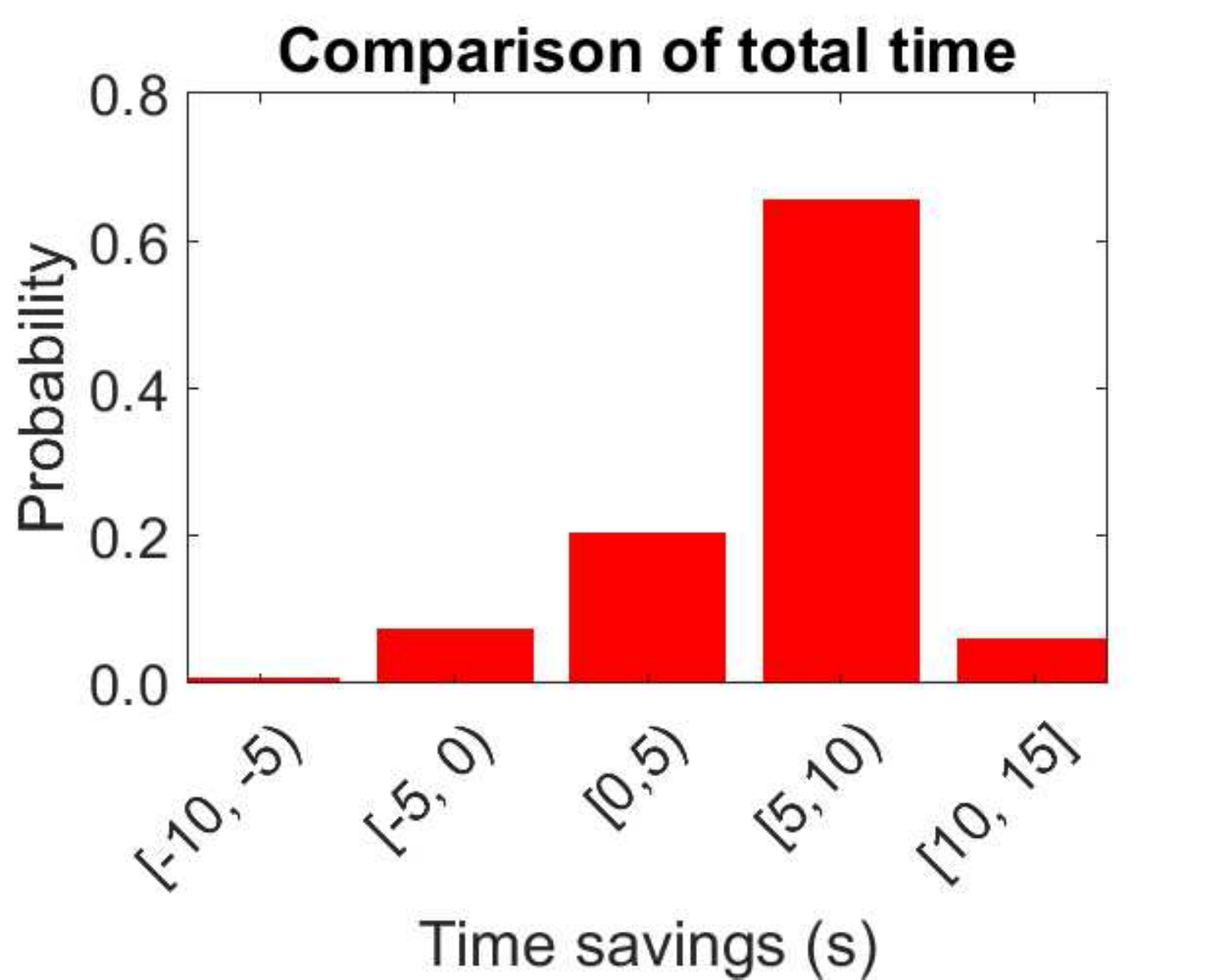}\label{fig:comparison_Dubins_total}}
        \caption{Time savings of the $4\pi$-arc solutions w.r.t. the Dubins solutions over $2880$ different simulation runs in a static current environment.}
     \label{fig:comparison_Dubins}
\end{figure}

\begin{figure*}[!t]
    \centering
    \subfloat[An example of path replanning under changing current. Start pose $(x_0,y_0,\theta_0)=(0,0,0)$ and goal pose $(x_f,y_f,\theta_f)=(5, 8.5, 3\pi/4)$. Initially, the current has $v_w = 0.5$~m/s and $\theta_w = \pi$, which changed at time $3.2$~s to a new current with $v_w = 0.75$~m/s and $\theta_w = 3\pi/2$. The radius of precision circle is $1$~m.]{
        \includegraphics[width=1\textwidth]{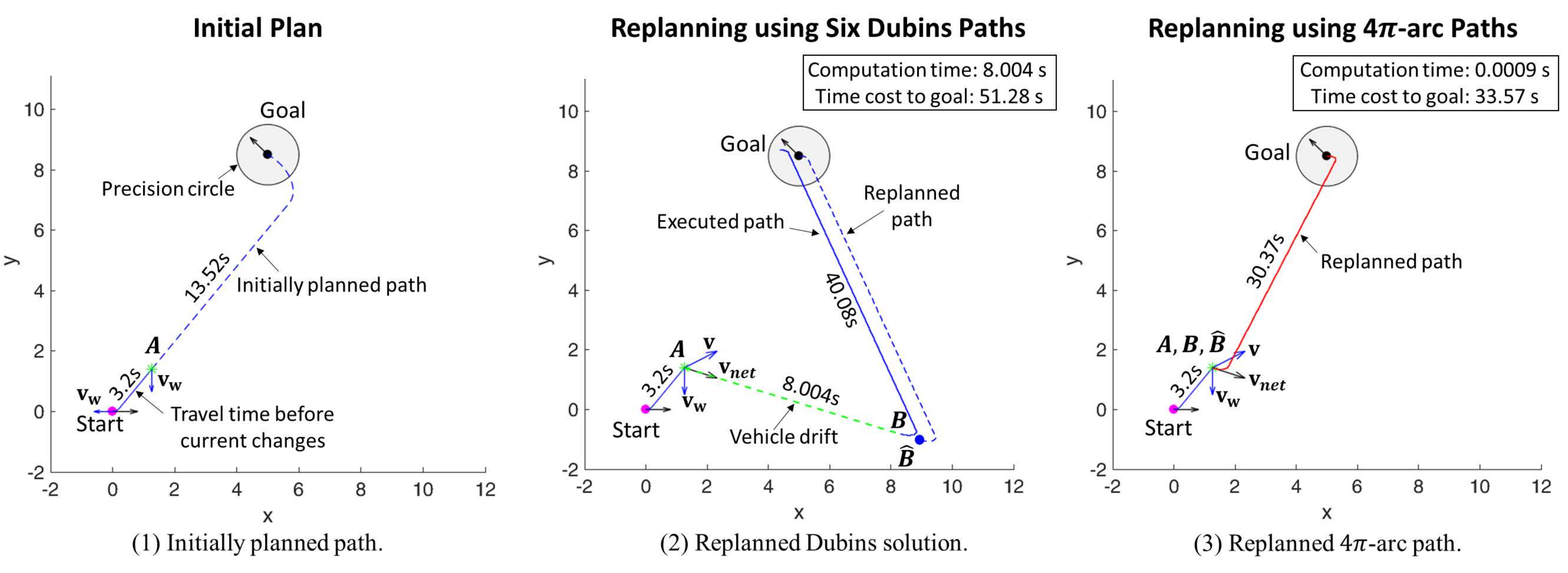}\label{fig:replan_results}}\\
    \subfloat[An example to show the effect of the net velocity of the vehicle drift. Start pose $(x_0,y_0,\theta_0)=(0,0,0)$ and goal pose $(x_f,y_f,\theta_f)=(5, 8.5, 3\pi/4)$. Initially, the current has $v_w = 0.5$~m/s and $\theta_w = 3\pi/2$, which changed at time $3.72$~s to a new  current.]{
        \includegraphics[width=1\textwidth]{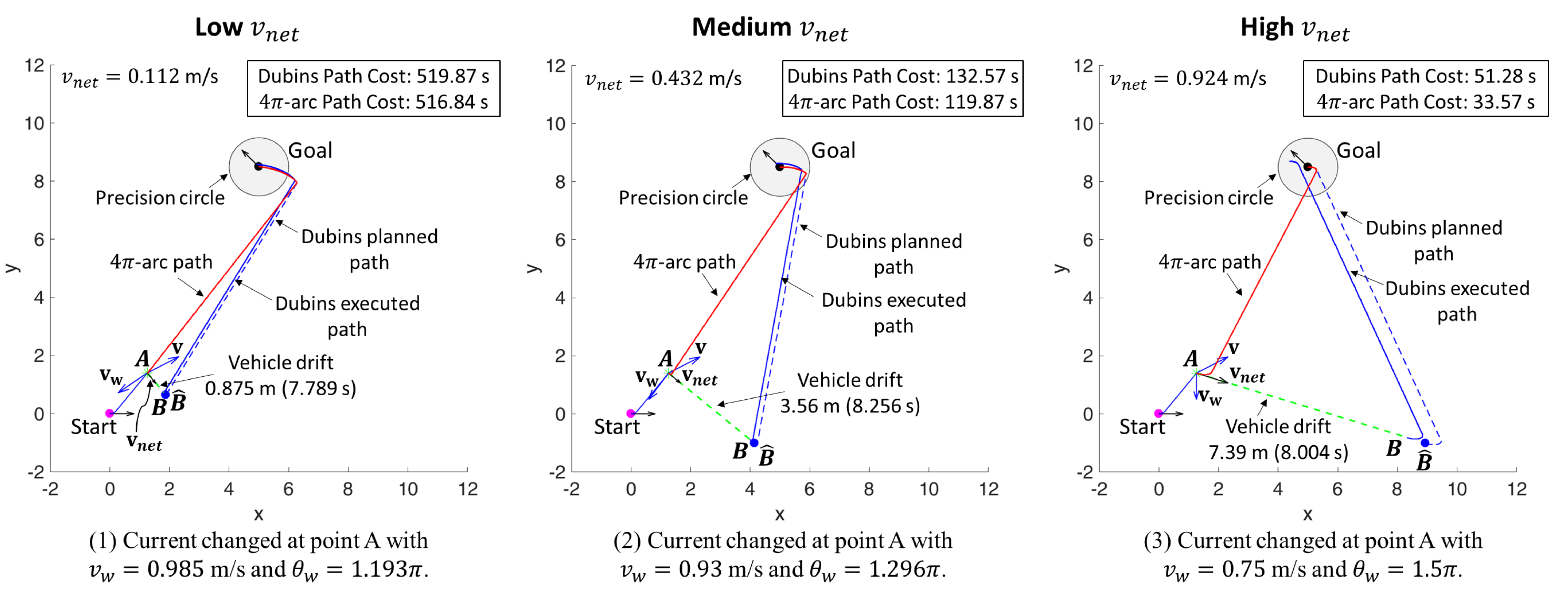}\label{fig:drift_results}}\\
        \vspace{3pt}
        \includegraphics[width=1\textwidth]{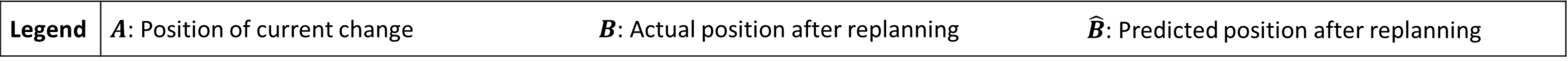}\label{fig:legend_results}
        \caption{Illustrative examples of replanning under changing current, and the effect of $v_{net}$ on the vehicle drift.}
     \label{fig:traj_results}
\end{figure*}

\subsection{Effect of a Change in Current}\label{sec:res_realtime}
During path execution, a change in the current's speed or heading could deviate the vehicle from its original path if left unattended. Hence, it is necessary to replan online upon detection of a change in current. However, as explained in Section~\ref{sec:intro}, using Dubins solution to regenerate the path to reach the goal pose requires considerable amount of computation time to solve the transcendental functions, during which the vehicle can drift noticeably. In particular, the vehicle drift would be along the direction of the net velocity of the vehicle and the current at that moment. To account for such drifts, the replanning is done by using a predicted position of the vehicle after the drift as the new start pose. This predicted position is computed by adding a translation (i.e., the product of the average computation time of $\sim8.72$~s and the net velocity) to the vehicle pose. Note that the predicted position is needed only for the Dubins solution, while it is unnecessary for the $4\pi$-arc path solution due to its negligible computation time.

\begin{figure*}[t]
    \centering
    \includegraphics[width=1\textwidth]{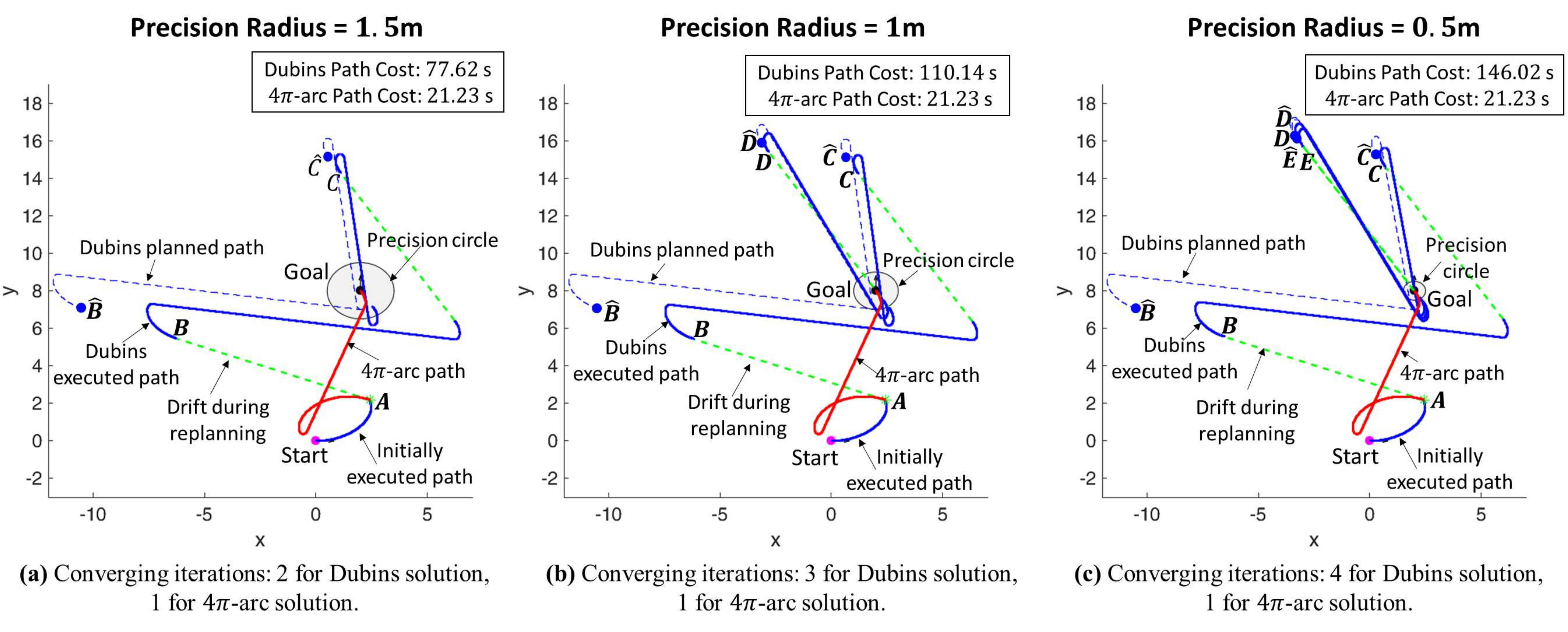}\\
    \vspace{3pt}
    \includegraphics[width=1\textwidth]{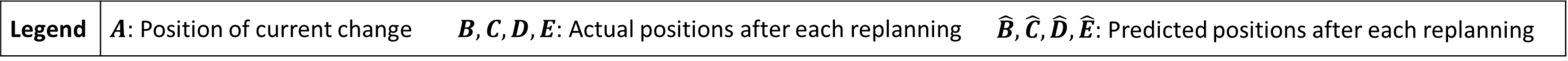}
    \caption{An example to show the effect of precision on planning time. Start pose $(x_0,y_0,\theta_0)=(0,0,0)$ and goal pose $(x_f,y_f,\theta_f)=(2, 8, \pi/2)$. Initially, the  current has $v_w = 0.75$~m/s and $\theta_w = 0$, which changed at time $3.72$~s to a new current with $v_w = 0.65$~m/s and $\theta_w = \pi$.}
    \label{fig:prec_results}
\end{figure*}

The vehicle is considered to be successful in reaching the goal if it: 1) arrives within a precision circle of radius $1$~m centered at the goal, and 2) achieves a heading within $\theta_f\pm5^{\circ}$.

Fig.~\ref{fig:replan_results} shows an illustrative example of the effect of current on replanning and the resulting total travel times using both approaches. Fig.~\ref{fig:replan_results}(1) shows the initially planned path using the Dubins approach from the start pose $(x_0, y_0, \theta_0) = (0,0,0)$ to the goal pose $(x_f, y_f, \theta_f) = (5,8.5,3\pi/4)$. The environment was considered to have an initial current of speed $v_w=0.5$~m/s and direction $\theta_w=\pi$. After the vehicle traveled for $3.2$~s and reached a point $A$, the current speed changed to $v_w=0.75$~m/s and its direction changed to $\theta_w=3\pi/2$, which forced the vehicle to replan a new path \textit{in situ}.

Fig.~\ref{fig:replan_results}(2) shows the replanning process using the Dubins approach. During replanning, the vehicle is drifted along the net velocity $\mathbf{v}_{net} = \mathbf{v} + \mathbf{v}_w$, where $\mathbf{v} = (v\cos{\theta}, v\sin{\theta})$ and $\mathbf{v}_w = (v_w \cos{\theta_w}, v_w \sin{\theta_w})$. The vehicle drift is shown by the green dashed line in the figure. The points $B$ and $\hat{B}$ denote the actual and the predicted position of the vehicle after replanning is over, respectively. Due to the difference between the predicted and the actual position, instead of executing the replanned path from the predicted position $\hat{B}$, marked by the blue dotted line, the vehicle actually traveled from point $B$, marked by the solid blue line. The vehicle then converged to the goal with its end-point lying inside the precision circle with an acceptable heading error. The total time taken by the vehicle to reach the goal is obtained by adding the initial execution time of $\sim3.2$~s before the change of  current, the replanning time of $\sim8$~s, and the execution time of $\sim40.08$~s along the replanned path, which leads to the total travel time of $\sim51.28$~s.

In comparison, Fig.~\ref{fig:replan_results}(3) shows the replanning process using the $4\pi$-arc $LSL$ and $RSR$ paths approach. Due to the negligible computation time, the points $A$, $B$ and $\hat{B}$ coincided, thus resulting in a much faster total travel time of $\sim33.57$~s. Also, the goal pose was achieved more accurately as compared to the Dubins solution. This example clearly highlights the benefits of the proposed rapid solution using the $4\pi$-arc paths over the Dubins approach.

\vspace{-6pt}
\subsection{Effect of $\mathbf{v}_{net}$}

During replanning, the vehicle is drifted along the direction of $\mathbf{v}_{net}$, with a magnitude of $v_{net}\in\mathbb{R}^+$ times the computation time. To examine the effect of $\mathbf{v}_{net}$ over the vehicle drift, we tested three scenarios over a range of $v_{net}$ and the results are shown in Fig.~\ref{fig:drift_results}(1)$-$(3). The start pose, the goal pose and the initial environmental current are set to be the same as those in Section~\ref{sec:res_realtime}; and the replanning occurs due to a change of current after $3.2$~s, when the vehicle has reached point $A$.

As seen in Fig.~\ref{fig:drift_results}(1)$-$(3), the $4\pi$-arc path solution generates trajectories with negligible drifts, while the Dubins solution results in significant vehicle drifts of lengths $0.875$~m for low $v_{net}=0.112$~m/s, $3.56$~m for medium $v_{net}=0.432$~m/s and $7.39$~m for high $v_{net}=0.924$~m/s. In all cases, since the Dubins solution incurs high computation time, it leads to a higher overall execution time. In particular, even for the scenario with low $v_{net}$ as shown in Fig.~\ref{fig:drift_results}(1), where the drift is very close to the vehicle's initial state and within its turning radius, $4\pi$-arc paths provide a faster solution than the Dubins solution because of the high computation time of the latter.

\begin{figure*}[!t]
    \centering
    \subfloat[Savings for naval application]{
        \includegraphics[width=0.40\textwidth]{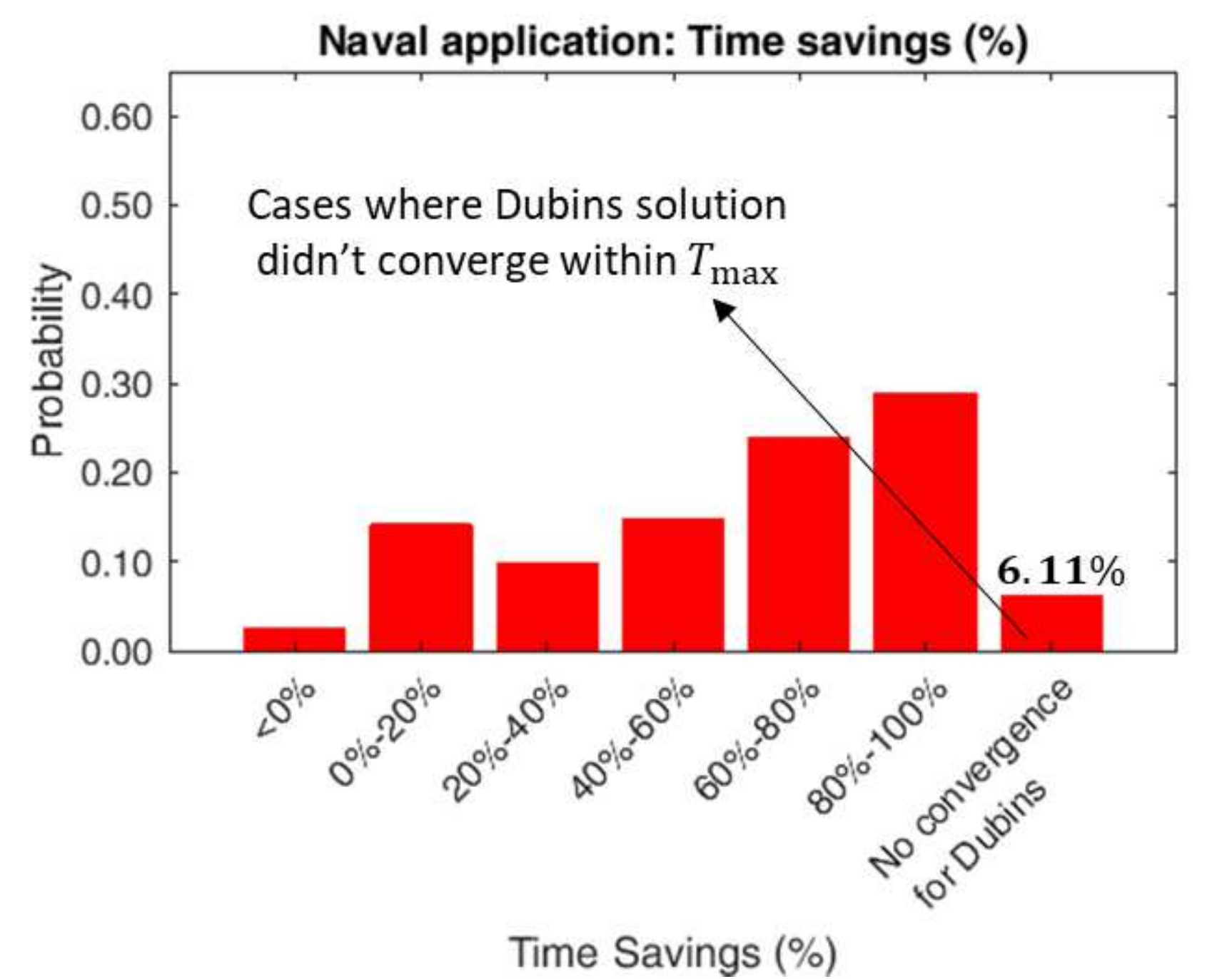}\label{fig:naval_results}}
    \subfloat[Savings for aerial application]{
        \includegraphics[width=0.40\textwidth]{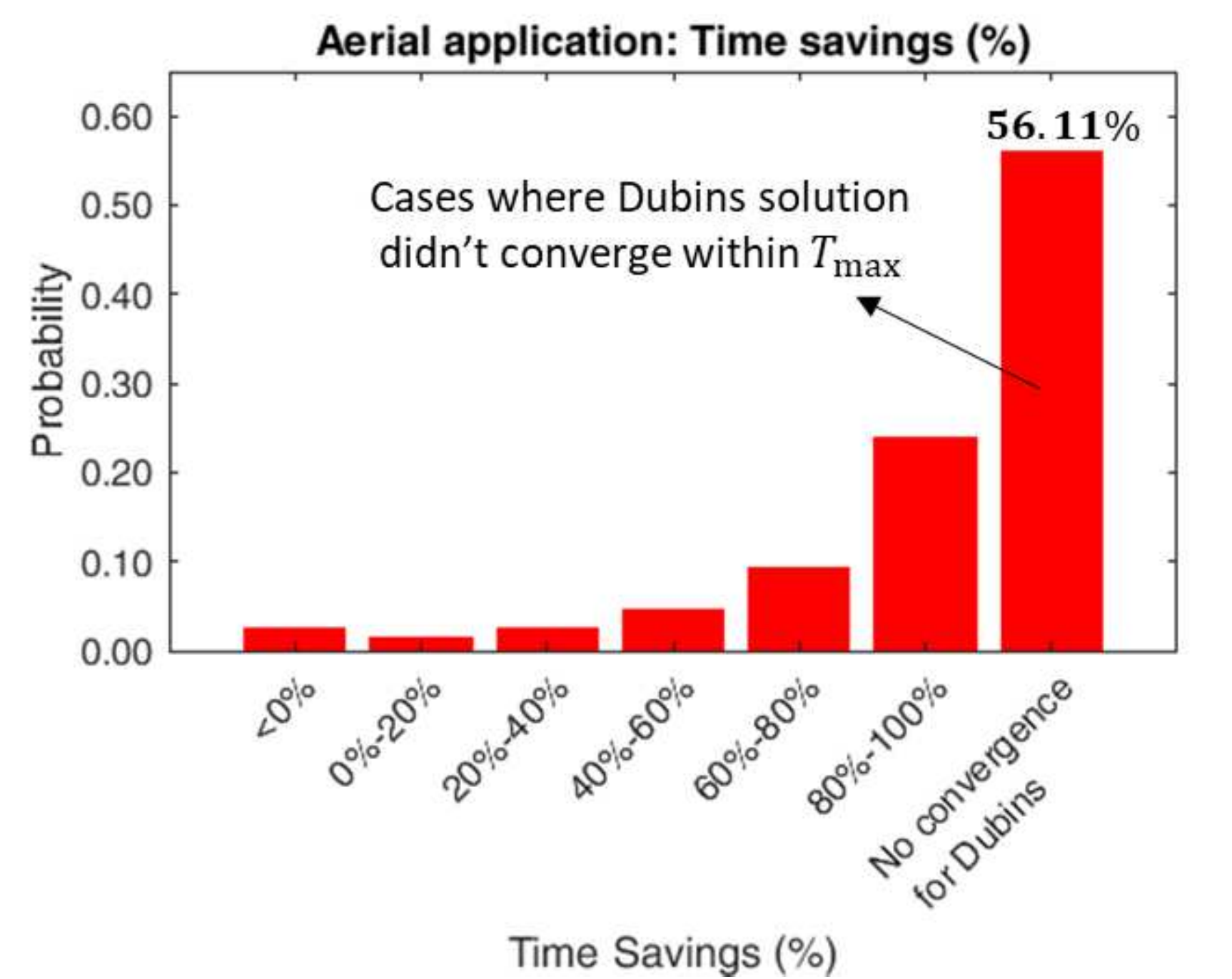}\label{fig:aerial_results}}

        \caption{Monte Carlo simulation results: Time savings of the $4\pi$-arc solutions w.r.t. the Dubins solutions.}
     \label{fig:mc_results}
     \vspace{-6pt}
\end{figure*}

\subsection{Effect of the Size of Precision Circle} Next, we study the effect of the size of precision circle, centered at the goal, on the total travel time using the two approaches. The vehicle is assumed to keep replanning until it converges inside the precision circle with an acceptable heading error. Fig.~\ref{fig:prec_results} shows the results obtained by varying the radii of the precision circle as: $1.5$~m, $1$~m and $0.5$~m. The start pose is $(x_0, y_0, \theta_0) = (0,0,0)$ and the goal pose is $(x_f, y_f, \theta_f) = (2,8,\pi/2)$. The environment was considered to have an initial current of speed $v_w=0.75$~m/s and direction $\theta_w=0$, which changed to $v_w=0.65$~m/s and $\theta_w=\pi$ at time $3.72$~s. As seen in Fig.~\ref{fig:prec_results}, after the change of current, the Dubins approach faces serious difficulty in convergence to the goal requiring several replannings as the precision radius decreases, while the $4\pi$-arc approach converged easily every time in a single replanning. Specifically, for precision radius of $1.5$~m, $1$~m and $0.5$~m, the Dubins approach required $2$, $3$ and $4$ replannings before convergence to the goal; accordingly, the total travel times to reach the goal were $77.62$~s, $110.14$~s and 146.02~s, respectively. As expected, the total travel time of $4\pi$-arc solution was $21.23$~s which is much smaller than the Dubins solution and was unaffected by the shrinking precision radius. This is due to the significantly less replanning time of the $4\pi$-arc paths, which allows them to reach the goal with high accuracy in shorter times.

\vspace{-6pt}
\subsection{Comparison of $4\pi$-arc $LSL$ and $RSR$ solutions with Dubins solutions in a dynamic current environment} \label{changingcurrents}
Now, we present a comparative evaluation of the $4\pi$-arc $LSL$ and $RSR$ solutions with Dubins solutions in a dynamic current environment. The performance of the two approaches is evaluated statistically using Monte Carlo simulations which cover a wide range of environmental conditions, considering realistic vehicle properties and sensing capabilities. The simulation setup is described as follows.

\vspace{6pt}
\textit{Sampled Goal Poses}:  The start pose is fixed at $(x_0, y_0, \theta_0)=(0,0,0)$. Then, six different goal positions are chosen located at a distance of $R=100$~m from the origin. For each goal position, six different heading angles $\theta_f \in \{ \frac{m\pi}{3}, m = 0,\ldots 5\}$ are considered, which leads to a total number of $36$ start and goal pose pairs. Due to noise (discussed later), $10$ Monte Carlo simulation runs were conducted for each goal pose, thus leading to a total number of $360$ runs.

\vspace{6pt}
\textit{Changing Environment}: To validate the effectiveness of the proposed method, the current with speed $v_w$ is set to change its direction with a random heading angle $\theta_w \in \{\frac{m\pi}{6}, m = 0,\ldots, 11\}$. This change happens after a random time interval $T_0 \in \{30, 45, 60\}$~s. Specifically, for each simulation run, the current heading $\theta_w$ and its time period $T_0$ are randomly generated from their corresponding sets. Then, after $T_0$, the updated current heading $\theta_w$ and its time period $T_0$ are randomly chosen again and the process is repeated. Thus, the vehicle has to replan its path based on the updated $\theta_w$ every time the current changes. Since the measurements of $\theta_w$ include noise (discussed later), the vehicle estimates its value using a Maximum Likelihood Estimator (MLE)~\cite{BLK04}, which utilizes measurements of $\theta_w$ within a period of $T_1=12$~s.

\vspace{6pt}
\textit{Termination Conditions}:  The vehicle is assumed to successfully reach the goal pose if: (1) it arrives within a precision circle of radius $1.5$~m centered at the goal, and (2) its heading falls between $\theta_f \pm 5^o$. However, if the vehicle cannot
converge to the goal pose in $T_{\max}=1000$~s, then the solution is considered to be not convergent.

\vspace{6pt}
\textit{Performance Metric:} The performance of the proposed $4\pi$-arc solution is evaluated in comparison to the Dubins solution based on the percentage of savings in the total travel time:
\begin{equation}
    Savings (\%)= \frac{T_{Dubins} - T_{4\pi}}{T_{Dubins}}\cdot 100,
\end{equation}
where $T_{Dubins}$ and $T_{4\pi}$ denote the total time cost using Dubins solution and the proposed $4\pi$-arc solution, respectively.

\vspace{6pt}
\textit{Applications:} Since sensing capabilities can vary significantly for different vehicles and in different operation environments, we evaluated the performance for two different applications: 1) naval (unmanned underwater vehicles (UUVs)) and 2) aerial (unmanned aerial vehicles (UAVs)).

\vspace{6pt}
\subsubsection{Naval Application}
Consider a typical UUV that travels at a speed of $v = 2.5$~m/s. The ocean environment is assumed to have currents that move at a speed of $v_w = 2$~m/s with an initial heading of $\theta_w = 0$. Regarding the sensing systems, the ocean current speed and heading are usually measured using an Acoustic Doppler Current Profiler (ADCP)~\cite{ADCP} with a sampling rate of $1$~Hz. On the other hand, the location and heading of UUV can be measured using Long Baseline (LBL) localization system~\cite{PSSL14} and compass, respectively. The sensor uncertainties are modeled using Additive White Gaussian Noise (AWGN) with parameters listed in Table~\ref{table:noise}.

\begin{table}[b!]
\centering
\caption{The specifics in Monte Carlo simulations}\label{table:noise}
\begin{tabular}{lll}
  \toprule
  {Application} & {Naval} & {Aerial}\\
  \midrule
  \eqparbox{Col}{Vehicle speed} & \eqparbox{Col}{$v = 2.5$~m/s} & \eqparbox{Col}{$v = 10$~m/s}\\
  \addlinespace[0.1cm]
  \eqparbox{Col}{External current} & \eqparbox{Col}{Ocean currents \\ $v_w = 2$~m/s} & \eqparbox{Col}{Wind \\ $v_w = 8$~m/s}\\
  \addlinespace[0.2cm]
  \eqparbox{Col}{Noise in vehicle\\state measurement} & \eqparbox{Col}{$\sigma_{GPS} = 0.3$~m \\ $\sigma_{compass} = 0.5^o$} & \eqparbox{Col}{$\sigma_{GPS} = 0.01$~m\\ $\sigma_{compass} = 0.5^o$}\\
  \addlinespace[0.2cm]
  \eqparbox{Col}{Noise in current state\\ measurement} & \eqparbox{Col}{$\sigma_{v_w} = 0.75\%\cdot{v_w}$ \\ $\sigma_{\theta_w} = 0.67^o$} & \eqparbox{Col}{$\sigma_{v_w} = 1.25\%\cdot{v_w}$ \\ $\sigma_{\theta_w} = 4^o$}\\
  \addlinespace[0.1cm]
  \bottomrule
\end{tabular}
\end{table}

Fig.~\ref{fig:naval_results} shows the distribution of percentage savings in time for the $4\pi$-arc path solutions in comparison to the corresponding Dubins solutions over all Monte Carlo runs. While $4\pi$-arc path solutions always converged, Dubins solutions could not converge within the precision circle in $T_{\max}$ time for $6.11\%$ of the runs. As explained in Section~\ref{sec:res_realtime}, this happens mainly due to their significantly high computation times during replanning which makes them keep replanning due to errors caused by the vehicle drift. For the remaining runs where both methods converged, the proposed $4\pi$-arc path solutions achieved an average of $57.62\%$ time savings, thus showing their superiority over Dubins solutions in a dynamic naval environment. This implies that the $4\pi$-arc path solutions can guide the UUV to successfully reach the goal pose in significantly less time cost as compared to the Dubins solutions. Furthermore, we note that only a very small fraction of all test cases result in negative time savings, which could be perhaps when the vehicle drift directly took the vehicle to the goal.

\vspace{6pt}
\subsubsection{Aerial Application}
Consider a typical UAV that travels at a speed of $v = 10$~m/s. The environment is assumed to have wind that moves at a speed of $v_w = 8$~m/s with an initial heading $\theta_w = 0$. As for the sensing systems, the wind profile can be measured using the Acoustic Resonance Wind Sensor system of FT 205~\cite{FT205}, which has a sampling rate of $10$~Hz. For localization of the UAV, a Real-Time Kinematic (RTK) GPS is used~\cite{RTK}. The sensor uncertainties are modeled using AWGN, with parameters listed in Table~\ref{table:noise}.

Fig.~\ref{fig:aerial_results} shows the distribution of percentage savings in time for the $4\pi$-arc path solutions in comparison to the corresponding Dubins solutions over all Monte Carlo runs. While $4\pi$-arc path solutions always converged, Dubins solutions could not converge within the precision circle in $T_{\max}$ time for $56.11\%$ of the runs. This number is higher than that of the naval applications due to the much higher uncertainties in current state measurements using wind sensors. The significantly increased number of non-converging runs shows the poor performance of Dubins approach in severe environments, thus highlighting the benefits of $4\pi$-arc path solutions. For the remaining runs where both methods converged, the proposed $4\pi$-arc path solutions achieved an average of $68.47\%$ time savings, thus showing their superiority over the Dubins solutions in a dynamic aerial environment. Furthermore, we note that only a very small fraction of all test runs result in negative time savings, while a significant majority have faster $4\pi$-arc path solutions.

\vspace{0pt}
\section{{Summary and Future Work}}\label{sec:conclusion}

\vspace{0pt}
\subsection{Summary}
The paper presents a rapid (real-time) solution to the minimum-time path planning problem for Dubins vehicles in the presence of environmental currents. The standard Dubins solution is obtained by solving for six path types ($LSL, RSR, LSR, RSL, LRL, RLR$); however, due to the presence of currents, four of these path types require solving of the root-finding problem involving transcendental functions. Thus, the existing Dubins solution results in high computation times which are not suitable for real-time applications.

Therefore, to obtain a real-time solution, this paper proposed a novel approach which utilizes only the $LSL$ and $RSR$ path types from the Dubins solution set which have direct analytical solutions; however they lack full reachability.

In this regard, the paper established the following properties for $LSL$ and $RSR$ paths:

\begin{enumerate}
\item Full reachability is guaranteed by extending their arc ranges from $2\pi$ to $4\pi$;

\item $4\pi$-arc paths yield superior or same performance in terms of time costs as compared to the corresponding $2\pi$-arc paths;
\item $4\pi$-arc paths require the same computational load to obtain a solution as needed for $2\pi$-arc paths.
\end{enumerate}

Based on the above, it is established that for real-time applications, the planner should consider the $4\pi$-arc $LSL$ and $RSR$ path solutions, while $2\pi$-arc solutions are not needed.

Furthermore, the performance of the proposed approach was evaluated against the Dubins solution with all six path types. For this purpose, two applications were considered: i) naval and ii) aerial, where extensive Monte Carlo simulations were conducted for statistical analysis under stochastic uncertainties in dynamically changing environments. The results showed that the $4\pi$-arc solutions converged to the goal pose in all runs as opposed to the Dubins solutions which failed to converge in a significant portion of runs. For the cases where Dubins solutions converged, the $4\pi$-arc solutions yielded superior performance and achieved significantly lower time costs to reach the goal poses with high precision.

\subsection{Future Work}
Future research will consider the following challenging problems for Dubins vehicles: 1) minimum-time path planning under spatio-temporally varying currents, 2) complete coverage in unknown environments~\cite{SG18}~\cite{SG19}, and 3) Dubins orienteering problem in dynamic environments~\cite{PFVS17}.

\vspace{0pt}
\appendix

\vspace{0pt}
\subsection{Derivation of conditions under which $2\pi$-arc $LSL$ and $RSR$ path types provide full reachability}
\label{app:reachability}
\vspace{6pt}

From (\ref{eq:LSL_slope}) and  (\ref{eq:RSR_slope}), we note that the boundaries of the reachable areas have the following rotations:

\begin{itemize}
\item $\omega_{LSL}^k(\alpha_{inf})$ and $\omega_{LSL}^k(\alpha_{sup})$,  for $k=0,1$,
\item $\omega_{RSR}^k(\alpha_{inf})$ and $\omega_{RSR}^k(\alpha_{sup})$, for $k=-1,-2$.
\end{itemize}

\vspace{6pt}
Now, we present a lemma related to these boundary rotations, which helps us in deriving the reachability conditions.
\vspace{0pt}

\begin{lem}\label{lem:parallelslope}
The following are true:
\begin{itemize}
    \item $\omega^0_{LSL}(\alpha_{inf})=\omega^1_{LSL}(\alpha_{sup})= \omega_{RSR}^{-1}(\alpha_{inf}) = \omega_{RSR}^{-2}(\alpha_{sup})$
    \item $\omega^0_{LSL}(\alpha_{sup})=\omega^1_{LSL}(\alpha_{inf})= \omega_{RSR}^{-1}(\alpha_{sup}) = \omega_{RSR}^{-2}(\alpha_{inf})$
\end{itemize}
\end{lem}

\begin{proof}
See Appendix~\ref{proof:parallelslope}
\end{proof}

By Lemma~\ref{lem:parallelslope}, the boundary lines of certain reachability regions of $LSL$ and $RSR$ path types are parallel to each other. This fact is explored to derive the full reachability conditions.

\begin{figure*}[t]
    \centering
    \subfloat[Case 1: Union of MaRA and MiRA of $LSL$ paths. The two plots show the two conditions for the subcase, where $k = 0$ forms MaRA and $k = 1$ forms MiRA. The titles show the corresponding conditions.]{
        \includegraphics[width=.48\textwidth]{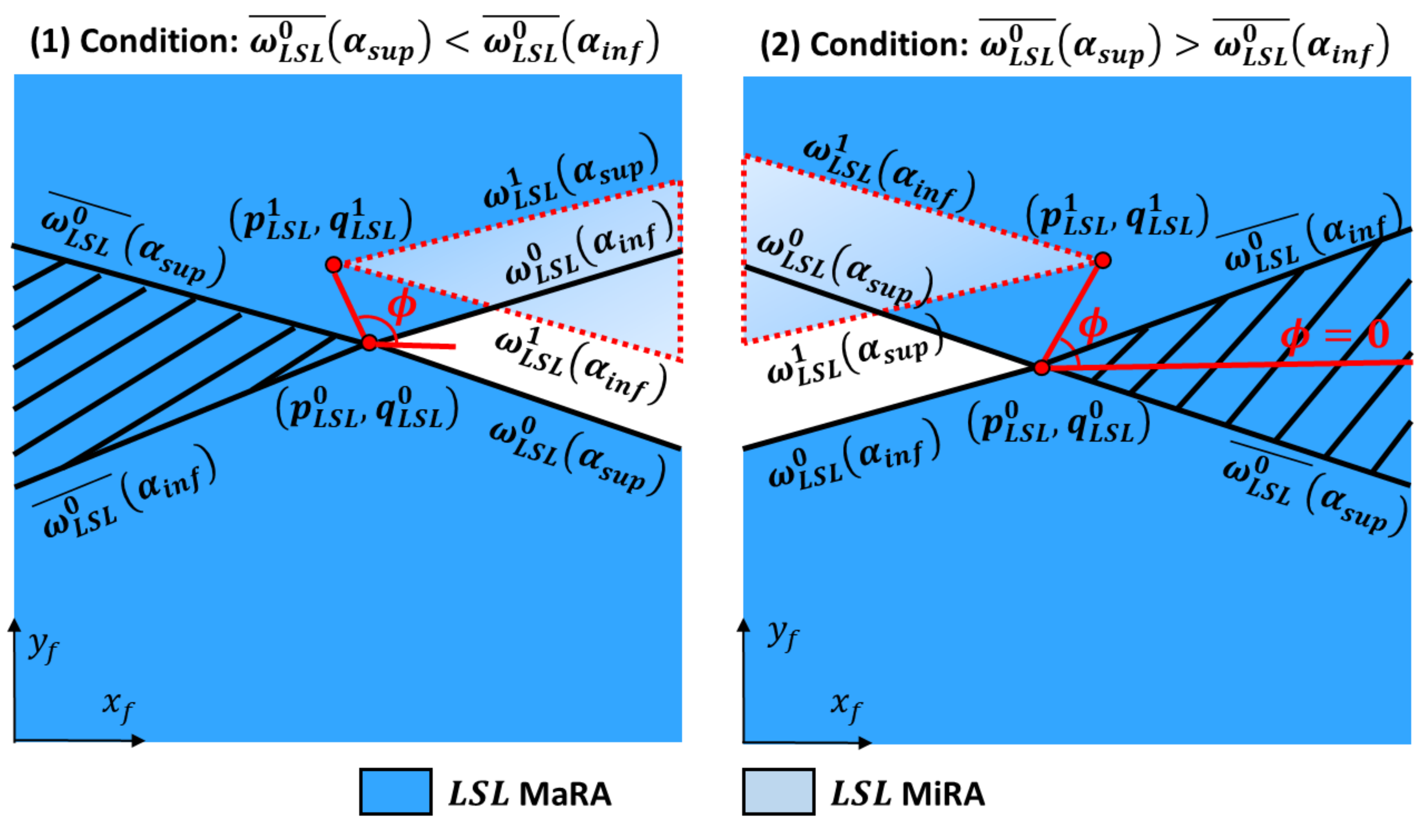}\label{fig:same_LSL}}
    \subfloat[Case 2: Union of MaRA and MiRA of $RSR$ paths. The two plots show the two conditions for the subcase, where $k=-1$ forms MaRA and $k=-2$ forms MiRA. The  titles show the corresponding conditions.]{
        \includegraphics[width=.48\textwidth]{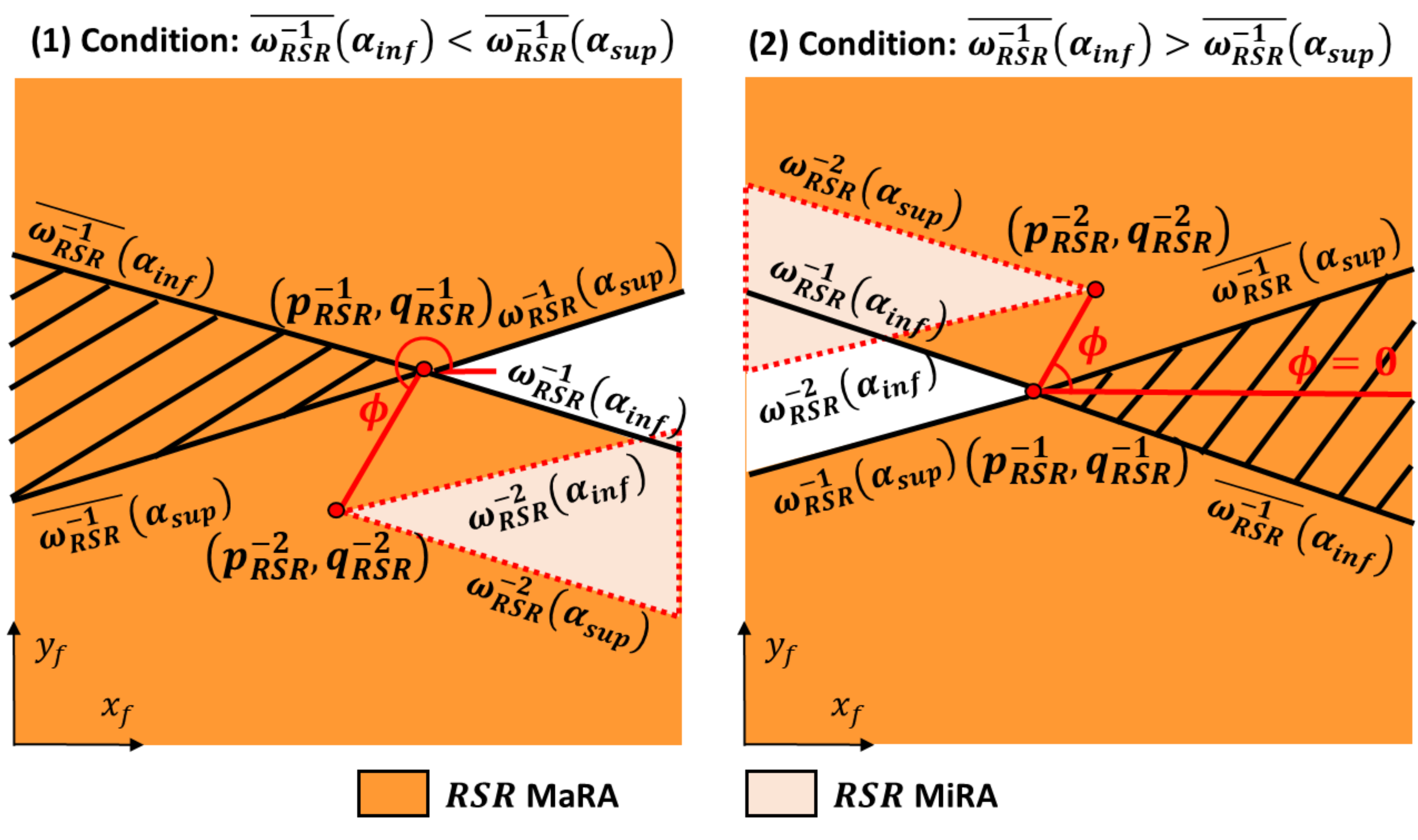}\label{fig:same_RSR}}\\
    \subfloat[Case 3: Union of MaRA of $LSL$ and MiRA of $RSR$ paths. The two plots show the two conditions for the subcase, where $k=0$ forms MaRA and $k=-1$ forms MiRA. The  titles show the corresponding conditions.]{
        \includegraphics[width=.48\textwidth]{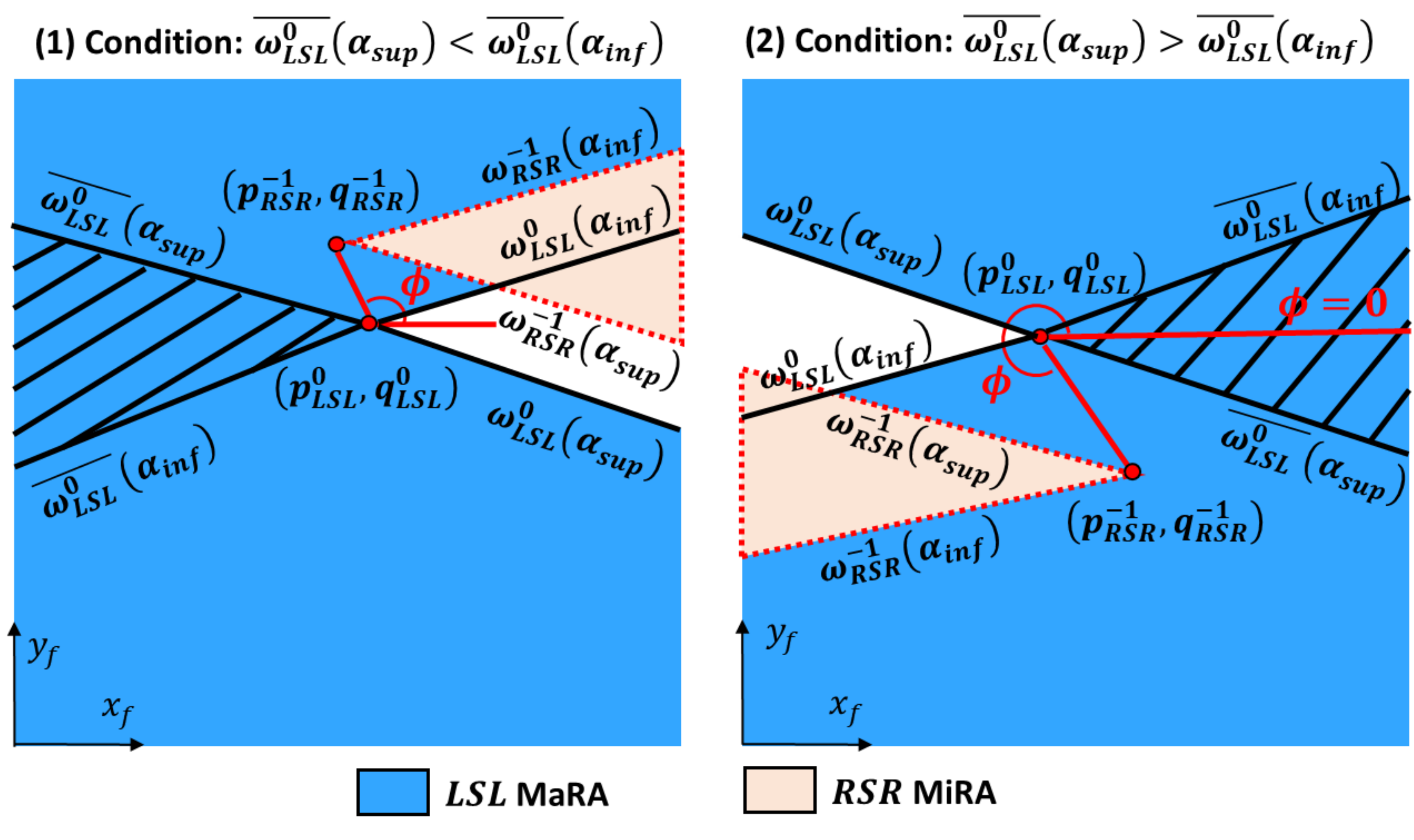}\label{fig:different_LSL}}
    \subfloat[Case 4: Union of MaRA of $RSR$ and MiRA of $LSL$ paths. The two plots show the two conditions for the subcase, where $k=-1$ forms MaRA and $k=0$ forms MiRA. The  titles show the corresponding conditions.]{
        \includegraphics[width=.48\textwidth]{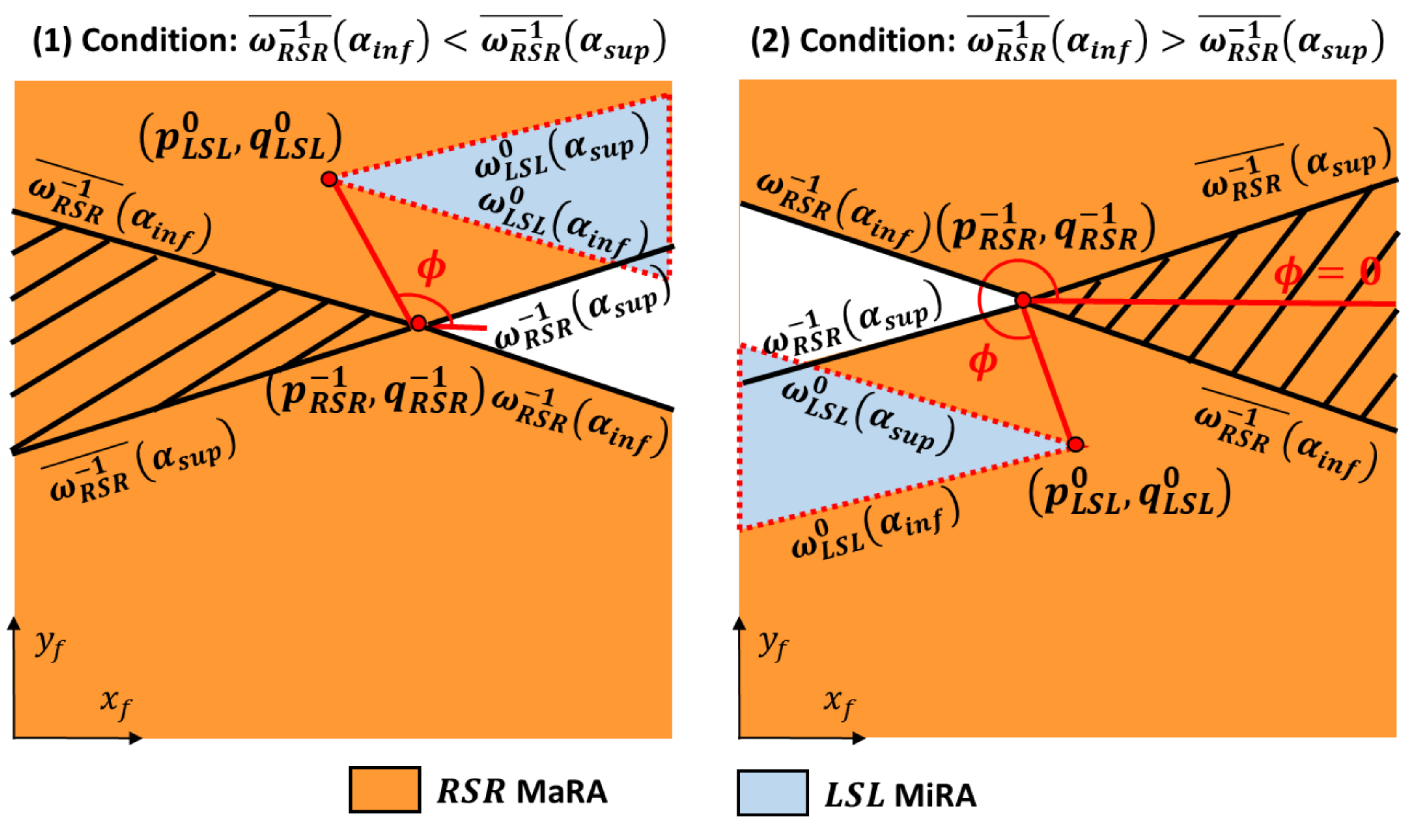}\label{fig:different_RSR}}\\ \vspace{6pt}
    \caption{Illustration of full reachability conditions using the $2\pi$-arc paths}\label{fig:reachability_conditions}  \vspace{6pt}
\end{figure*}

\vspace{6pt}
Before we start with the detailed analysis of the reachability conditions, we present a useful notation. Let $\delta \in [0,2\pi)$ be the rotation of a ray, then denote
\begin{equation}\label{pirotation}
\boxed{
\overline{\delta} \equiv \delta + \pi \Mod{2\pi}},
\end{equation}
to be the rotation of the ray in its opposite direction.

As discussed in Section~\ref{sec:full_reachability_conditions}, full reachability is achieved by $2\pi$-arc $LSL$ and $RSR$ paths, if the entire space is covered by atleast one of the following cases:

\begin{enumerate}
\item Union of MaRA and MiRA of $LSL$, and/or
\item Union of MaRA and MiRA of $RSR$, and/or
\item Union of MaRA of $LSL$ and MiRA of $RSR$, and/or
\item Union of MaRA of $RSR$ and MiRA of $LSL$.
\end{enumerate}

Now, we derive the full reachability conditions for Case $1$, while the derivation of the rest of the cases are similar.

\vspace{12pt}
\textbf{Case 1: Conditions under which the union of $\mathbf{LSL}$ MaRA and $\mathbf{LSL}$ MiRA provide full reachability}

\vspace{6pt}
Consider the centers  $(p_{LSL}^0, q_{LSL}^0)$ and $(p_{LSL}^1, q_{LSL}^1)$, as described in (\ref{eq:p_q_LSL}), for $k=0$ and $k=1$, respectively. There are two subcases:

\vspace{3pt}
\begin{itemize}
\item [1.1]\textit{$k=0$ forms $LSL$ MaRA and  $k=1$ forms $LSL$ MiRA}:

An illustrative example is shown in Fig. \ref{fig:same_LSL}. Note that the boundaries of $LSL$ MaRA are formed by  rays with rotations $\omega^0_{LSL}(\alpha_{inf})$ and $\omega^0_{LSL}(\alpha_{sup})$. Similarly, the boundaries of $LSL$ MiRA are formed by  rays with  rotations $\omega^1_{LSL}(\alpha_{inf})$ and $\omega^1_{LSL}(\alpha_{sup})$.

\vspace{6pt}
Now, using the notation in (\ref{pirotation}), we define $\overline{\omega_{LSL}^0}(\alpha_{sup})$  and  $\overline{\omega_{LSL}^0}(\alpha_{inf})$ to denote the rotations of the boundaries of $LSL$ MaRA by $\pi$ about the center $(p_{LSL}^0, q_{LSL}^0)$.

\vspace{3pt}
Further,  since $\overline{\omega_{LSL}^0}(\alpha_{sup}) \neq \overline{\omega_{LSL}^0}(\alpha_{inf})$, we can have:

\begin{itemize}
    \item [1)] $\overline{\omega_{LSL}^0}(\alpha_{sup}) < \overline{\omega_{LSL}^0}(\alpha_{inf})$, as shown in Fig.~\ref{fig:same_LSL}(1) or
    \vspace{3pt}
    \item [2)] $\overline{\omega_{LSL}^0}(\alpha_{sup})$ > $\overline{\omega_{LSL}^0}(\alpha_{inf})$, as shown in Fig.~\ref{fig:same_LSL}(2).
\end{itemize}

\vspace{6pt}
The region enclosed within the $\pi$ rotations of $LSL$ MaRA boundaries is shown as the shaded area in Fig.~\ref{fig:same_LSL}.

\vspace{6pt}
For full reachability, $LSL$ MiRA should cover the unreachable area of $LSL$ MaRA. From Lemma \ref{lem:parallelslope}, we know that $\omega^0_{LSL}(\alpha_{inf})=\omega^1_{LSL}(\alpha_{sup})$ and $\omega^0_{LSL}(\alpha_{sup})=\omega^1_{LSL}(\alpha_{inf})$, thus the respective boundaries of $LSL$ MaRA and $LSL$ MiRA are parallel. This fact implies that, to achieve full reachability, the center of rotation $(p_{LSL}^1, q_{LSL}^1)$ of $LSL$ MiRA should lie within the shaded area of $LSL$ MaRA (see  Fig.~\ref{fig:same_LSL}).

\begin{table*}[t]
\caption{\label{table:full_reachability} Full reachability conditions using $2\pi$-arc $LSL$ and $RSR$ paths.}
\begin{tabular}{|c|c|c|c|c|c|c|c|}
\hline
\multirow{2}{*}{Case} & \multicolumn{2}{c|}{MaRA} & \multicolumn{2}{c|}{MiRA} & \multirow{2}{*}{\begin{tabular}[c]{@{}c@{}}Rotation of the line segment joining\\ the centers of MaRA and MiRA\end{tabular}} & \multicolumn{2}{c|}{Full Reachability Conditions} \\ \cline{2-5} \cline{7-8}
 & Path Type & $k$ & Path Type & $k$ &  & If & Reachability Condition \\ \hline
\multirow{4}{*}{$1$} & \multirow{4}{*}{$LSL$} & \multirow{2}{*}{$0$} & \multirow{4}{*}{$LSL$} & \multirow{2}{*}{$1$} & \multirow{2}{*}{$\phi_{0,1} = \atantwo \big(w_y, w_x \big) \Mod{2\pi}$} & $\overline{\omega_{LSL}^0}(\alpha_{sup}) < \overline{\omega_{LSL}^0}(\alpha_{inf})$ & $\overline{\omega_{LSL}^0}(\alpha_{sup}) \leq \phi_{0,1} \leq \overline{\omega_{LSL}^0}(\alpha_{inf})$ \\ \cline{7-8}
 &  &  &  &  &  & $\overline{\omega_{LSL}^0}(\alpha_{sup}) > \overline{\omega_{LSL}^0}(\alpha_{inf})$ & \begin{tabular}[c]{@{}c@{}}$\overline{\omega_{LSL}^0}(\alpha_{sup}) \leq \phi_{0,1} < 2\pi$, \\ \text{ or } $0 \leq \phi_{0,1} \leq \overline{\omega_{LSL}^0}(\alpha_{inf})$\end{tabular} \\ \cline{3-3} \cline{5-8}
 &  & \multirow{2}{*}{$1$} &  & \multirow{2}{*}{$0$} & \multirow{2}{*}{$\phi_{1,0} = \atantwo \big(-w_y, -w_x \big) \Mod{2\pi}$} & $\overline{\omega_{LSL}^1}(\alpha_{sup}) < \overline{\omega_{LSL}^1}(\alpha_{inf})$ & $\overline{\omega_{LSL}^1}(\alpha_{sup}) \leq \phi_{1,0} \leq \overline{\omega_{LSL}^1}(\alpha_{inf})$ \\ \cline{7-8}
 &  &  &  &  &  & $\overline{\omega_{LSL}^1}(\alpha_{sup}) > \overline{\omega_{LSL}^1}(\alpha_{inf})$ & \begin{tabular}[c]{@{}c@{}}$\overline{\omega_{LSL}^1}(\alpha_{sup}) \leq \phi_{1,0} < 2\pi$,\\        \text{ or } $0 \leq \phi_{1,0} \leq  \overline{\omega_{LSL}^1}(\alpha_{inf})$\end{tabular} \\ \hline
\multirow{4}{*}{$2$} & \multirow{4}{*}{$RSR$} & \multirow{2}{*}{$-1$} & \multirow{4}{*}{$RSR$} & \multirow{2}{*}{$-2$} & \multirow{2}{*}{$\phi_{-1,-2}=\atantwo \big(w_y, w_x \big) \Mod{2\pi}$} & $\overline{\omega_{RSR}^{-1}}(\alpha_{inf}) < \overline{\omega_{RSR}^{-1}}(\alpha_{sup})$ & $\overline{\omega_{RSR}^{-1}}(\alpha_{inf}) \leq \phi_{-1,-2} \leq \overline{\omega_{RSR}^{-1}}(\alpha_{sup})$ \\ \cline{7-8}
 &  &  &  &  &  & $\overline{\omega_{RSR}^{-1}}(\alpha_{inf}) > \overline{\omega_{RSR}^{-1}}(\alpha_{sup})$ & \begin{tabular}[c]{@{}c@{}}$\overline{\omega_{RSR}^{-1}}(\alpha_{inf}) \leq \phi_{-1,-2} < 2\pi$,\\ \text{ or  } $0 \leq \phi_{-1,-2} \leq \overline{\omega_{RSR}^{-1}}(\alpha_{sup})$\end{tabular} \\ \cline{3-3} \cline{5-8}
 &  & \multirow{2}{*}{$-2$} &  & \multirow{2}{*}{$-1$} & \multirow{2}{*}{$\phi_{-2,-1}=\atantwo \big(-w_y, -w_x \big) \Mod{2\pi}$} & $\overline{\omega_{RSR}^{-2}}(\alpha_{inf}) < \overline{\omega_{RSR}^{-2}}(\alpha_{sup})$ & $\overline{\omega_{RSR}^{-2}}(\alpha_{inf}) \leq \phi_{-2,-1} \leq \overline{\omega_{RSR}^{-2}}(\alpha_{sup})$ \\ \cline{7-8}
 &  &  &  &  &  & $\overline{\omega_{RSR}^{-2}}(\alpha_{inf}) > \overline{\omega_{RSR}^{-2}}(\alpha_{sup})$ & \begin{tabular}[c]{@{}c@{}}$\overline{\omega_{RSR}^{-2}}(\alpha_{inf}) \leq \phi_{-2,-1} < 2\pi$,\\ \text{ or } $0 \leq \phi_{-2,-1} \leq \overline{\omega_{RSR}^{-2}}(\alpha_{sup})$\end{tabular} \\ \hline
\multirow{4}{*}{$3$} & \multirow{4}{*}{$LSL$} & \multirow{2}{*}{$0$} & \multirow{4}{*}{$RSR$} & \multirow{2}{*}{$-1$} & \multirow{2}{*}{\begin{tabular}[c]{@{}c@{}}$\phi_{0,-1}=\atantwo \Big(\cos{\theta_f}-1 + w_y(\pi-\theta_f)$,\\  $-\sin{\theta_f} + w_x (\pi - \theta_f) \Big) \Mod{2\pi}$\end{tabular}} & $\overline{\omega_{LSL}^{0}}(\alpha_{sup}) < \overline{\omega_{LSL}^{0}}(\alpha_{inf})$ & $\overline{\omega_{LSL}^{0}}(\alpha_{sup}) \leq \phi_{0,-1} \leq \overline{\omega_{LSL}^{0}}(\alpha_{inf})$ \\ \cline{7-8}
 &  &  &  &  &  & $\overline{\omega_{LSL}^{0}}(\alpha_{sup}) > \overline{\omega_{LSL}^{0}}(\alpha_{inf})$ & \begin{tabular}[c]{@{}c@{}}$\overline{\omega_{LSL}^{0}}(\alpha_{sup}) \leq \phi_{0,-1} < 2\pi$,\\ \text{ or } $0 \leq \phi_{0,-1} \leq \overline{\omega_{LSL}^{0}}(\alpha_{inf})$\end{tabular} \\ \cline{3-3} \cline{5-8}
 &  & \multirow{2}{*}{$1$} &  & \multirow{2}{*}{$-2$} & \multirow{2}{*}{\begin{tabular}[c]{@{}c@{}}$\phi_{1,-2}=\atantwo \Big(\cos{\theta_f}-1 + w_y(\pi-\theta_f),$\\  $-\sin{\theta_f} + w_x (\pi - \theta_f) \Big) \Mod{2\pi}$\end{tabular}} & $\overline{\omega_{LSL}^{1}}(\alpha_{sup}) < \overline{\omega_{LSL}^{1}}(\alpha_{inf})$ & $\overline{\omega_{LSL}^{1}}(\alpha_{sup}) \leq \phi_{1,-2} \leq \overline{\omega_{LSL}^{1}}(\alpha_{inf})$ \\ \cline{7-8}
 &  &  &  &  &  & $\overline{\omega_{LSL}^{1}}(\alpha_{sup}) > \overline{\omega_{LSL}^{1}}(\alpha_{inf})$ & \begin{tabular}[c]{@{}c@{}}$\overline{\omega_{LSL}^{1}}(\alpha_{sup}) \leq \phi_{1,-2} < 2\pi$,\\ \text{ or } $0 \leq \phi_{1,-2} \leq \overline{\omega_{LSL}^{1}}(\alpha_{inf})$\end{tabular} \\ \hline
\multirow{4}{*}{$4$} & \multirow{4}{*}{$RSR$} & \multirow{2}{*}{$-1$} & \multirow{4}{*}{$LSL$} & \multirow{2}{*}{$0$} & \multirow{2}{*}{\begin{tabular}[c]{@{}c@{}}$\phi_{-1,0}=\atantwo \Big(1- \cos{\theta_f} - w_y (\pi-\theta_f)$,\\ $\sin{\theta_f} - w_x (\pi - \theta_f) \Big) \Mod{2\pi}$\end{tabular}} & $\overline{\omega_{RSR}^{-1}}(\alpha_{inf}) < \overline{\omega_{RSR}^{-1}}(\alpha_{sup})$ & $\overline{\omega_{RSR}^{-1}}(\alpha_{inf}) \leq \phi_{-1,0} \leq \overline{\omega_{RSR}^{-1}}(\alpha_{sup})$ \\ \cline{7-8}
 &  &  &  &  &  & $\overline{\omega_{RSR}^{-1}}(\alpha_{inf}) > \overline{\omega_{RSR}^{-1}}(\alpha_{sup})$ & \begin{tabular}[c]{@{}c@{}}$\overline{\omega_{RSR}^{-1}}(\alpha_{inf}) \leq \phi_{-1,0} < 2\pi$,\\ \text{ or } $0 \leq \phi_{-1,0} \leq \overline{\omega_{RSR}^{-1}}(\alpha_{sup})$\end{tabular} \\ \cline{3-3} \cline{5-8}
 &  & \multirow{2}{*}{$-2$} &  & \multirow{2}{*}{$1$} & \multirow{2}{*}{\begin{tabular}[c]{@{}c@{}}$\phi_{-2,1}=\atantwo \Big(1- \cos{\theta_f} - w_y (\pi-\theta_f)$,\\ $\sin{\theta_f} - w_x (\pi - \theta_f) \Big) \Mod{2\pi}$\end{tabular}} & $\overline{\omega_{RSR}^{-2}}(\alpha_{inf}) < \overline{\omega_{RSR}^{-2}}(\alpha_{sup})$ & $\overline{\omega_{RSR}^{-2}}(\alpha_{inf}) \leq \phi_{-2,1} \leq \overline{\omega_{RSR}^{-2}}(\alpha_{sup})$ \\ \cline{7-8}
 &  &  &  &  &  & $\overline{\omega_{RSR}^{-2}}(\alpha_{inf}) > \overline{\omega_{RSR}^{-2}}(\alpha_{sup})$ & \begin{tabular}[c]{@{}c@{}}$\overline{\omega_{RSR}^{-2}}(\alpha_{inf}) \leq \phi_{-2,1} < 2\pi$,\\ \text{ or } $0 \leq \phi_{-2,1} \leq \overline{\omega_{RSR}^{-2}}(\alpha_{sup})$\end{tabular} \\ \hline
\end{tabular}
\end{table*}

\vspace{6pt}
To implement this full reachability condition, we find the rotation of the line segment joining the centers $(p_{LSL}^0, q_{LSL}^0)$ and $(p_{LSL}^1, q_{LSL}^1)$ as
\begin{subequations}
\begin{align}
\phi_{0,1} & = \atantwo \big(q_{LSL}^1 - q_{LSL}^0, p_{LSL}^1 - p_{LSL}^0 \big) \Mod{2\pi} \\
& = \atantwo \big(w_y, w_x \big) \Mod{2\pi}, \label{eq:anglebetcenterscase1.1}
\end{align}
\end{subequations}
where (\ref{eq:anglebetcenterscase1.1}) is obtained using (\ref{eq:p_q_LSL}).

\vspace{6pt}
Then, based on the above discussion, we obtain the condition for full reachability as

\vspace{6pt}
\begin{itemize}
    \item If $\overline{\omega_{LSL}^0}(\alpha_{sup}) < \overline{\omega_{LSL}^0}(\alpha_{inf})$, then:
    \begin{equation}\label{eq:same_condition1}
        \boxed{\overline{\omega_{LSL}^0}(\alpha_{sup}) \leq \phi_{0,1} \leq \overline{\omega_{LSL}^0}(\alpha_{inf}).}
    \end{equation}
    \item If $\overline{\omega_{LSL}^0}(\alpha_{sup})$ > $\overline{\omega_{LSL}^0}(\alpha_{inf})$, then:
    \begin{equation}\label{eq:same_condition2}
        \boxed{\begin{split}
        \overline{\omega_{LSL}^0}(\alpha_{sup}) &\leq \phi_{0,1} < 2\pi, \\
        \text{ or \ \ \ \ \ \ \ } 0 &\leq \phi_{0,1} \leq \overline{\omega_{LSL}^0}(\alpha_{inf}).
        \end{split}}
    \end{equation}
\end{itemize}

\vspace{6pt}
\item [1.2] \textit{$k=1$ forms $LSL$ MaRA and $k=0$ forms $LSL$ MiRA}:

Since this subcase is similar to the first subcase of Case 1, we do not show the corresponding figure here.  Using the same logic as for the first subcase, we find the  rotation of the line segment joining the above two centers as
\begin{subequations}
\begin{align}
\phi_{1,0} & = \atantwo \big(q_{LSL}^0 - q_{LSL}^1, p_{LSL}^0 - p_{LSL}^1 \big) \Mod{2\pi}\\
& = \atantwo \big(-w_y, -w_x \big) \Mod{2\pi}, \label{eq:anglebetcenters2}
\end{align}
\end{subequations}
where (\ref{eq:anglebetcenters2}) is obtained using (\ref{eq:p_q_LSL}).

Then, we obtain the condition for full reachability as

\vspace{6pt}
\begin{itemize}
  \item
      If $\overline{\omega_{LSL}^1}(\alpha_{sup}) < \overline{\omega_{LSL}^1}(\alpha_{inf})$, then:\\
    \begin{equation}\label{eq:same_condition3}
       \boxed{
        \overline{\omega_{LSL}^1}(\alpha_{sup}) \leq \phi_{1,0} \leq \overline{\omega_{LSL}^1}(\alpha_{inf}).
        }
    \end{equation}

  \item If $\overline{\omega_{LSL}^1}(\alpha_{sup}) > \overline{\omega_{LSL}^1}(\alpha_{inf})$, then:
   \begin{equation}\label{eq:same_condition4}
     \boxed{
    \begin{split}
       \overline{\omega_{LSL}^1}(\alpha_{sup}) &\leq \phi_{1,0} < 2\pi, \\
       \text{ or \ \ \ \ \ \ \ } 0 &\leq \phi_{1,0} \leq  \overline{\omega_{LSL}^1}(\alpha_{inf}).
      \end{split}
       }
  \end{equation}
\end{itemize}
\end{itemize}

The reachability conditions for Cases $2-4$ can be derived in a similar fashion as Case $1$, and their illustrative examples are shown in Figs.~\ref{fig:same_RSR}, \ref{fig:different_LSL} and \ref{fig:different_RSR}, respectively. However, for Cases $3$ and $4$, the union of different path types is used. Therefore, to obtain reachability conditions for Cases $3$ and $4$, we need Lemma~\ref{lem:slopeproperty2} which connects the $k$ values associated with the MaRA and MiRA regions across different path types.

\begin{lem} \label{lem:slopeproperty2}
The following are true:
\begin{itemize}
\item [a)] If $k = 0$ forms $LSL$ MaRA (MiRA), then $k = -1$ forms $RSR$ MiRA (MaRA).

\item [b)] If $k = 1$ forms $LSL$ MaRA (MiRA), then $k = -2$ forms $RSR$ MiRA (MaRA).
\end{itemize}
\end{lem}
\begin{proof}
See Appendix~\ref{proof:slopeproperty2}
\end{proof}

Table~\ref{table:full_reachability} presents the  reachability conditions for all cases.

\begin{rem}
Besides Cases $1-4$, there are other cases that can be considered for reachability analysis. However, Lemma~\ref{lem:slopeproperty3} below negates those cases and shows that Cases $1-4$ are sufficient for full reachability analysis.
\end{rem}

\begin{lem} \label{lem:slopeproperty3}
The following are true:
\begin{itemize}
\item [a)] $LSL$ ($RSR$) MaRA alone cannot provide full reachability
\item [b)] Union of $LSL$ MaRA and $RSR$ MaRA cannot provide full reachability.
\item [c)] If Cases 1-4 do not provide full reachability, then the union of $LSL$ MaRA, $LSL$ MiRA, $RSR$ MaRA and $RSR$ MiRA cannot provide full reachability.

\end{itemize}
\end{lem}
\begin{proof}
See Appendix~\ref{proof:slopeproperty3}.
\end{proof}

\begin{cor}\label{lemma_cor}
Cases $1-4$ and the conditions therein are sufficient for full reachability analysis.
\end{cor}
\begin{proof}
Lemma~\ref{lem:slopeproperty3} discards all cases for full reachability analysis beyond Cases $1-4$. Hence proved.
\end{proof}

\subsection{Lemma proofs}
\label{app:lemma_proofs}
\vspace{6pt}
\subsubsection{\textbf{Proof of Lemma \ref{lem:swipe}}}\label{app:lemma1}
\begin{proof}
Lemma~\ref{lem:swipe} is proved in two steps. First, we show that as $\alpha$ varies within its feasible range as shown in Table~\ref{table:range_conventional}, the  rays (\ref{eq:LSL_reachability}) (corresponding to the $LSL$ path type) and (\ref{eq:RSR_reachability}) (corresponding to the $RSR$ path type) rotate, where the points $(p_{LSL}^k, q_{LSL}^k)$ and $(p_{RSR}^k, q_{RSR}^k)$ form their centers of rotation, respectively. Second, we show that as $\alpha$ increases, (\ref{eq:LSL_reachability}) rotates anticlockwise, while (\ref{eq:RSR_reachability}) rotates clockwise.

For $LSL$ path type, (\ref{eq:LSL_reachability}) can be re-written as

\begin{equation}\label{eq:LSL_reachability_revisit}
    a(\alpha) \cdot ( x_f - p_{LSL}^k \big) - c(\alpha) \cdot \big( y_f - q_{LSL}^k ) = 0.
\end{equation}

Thus, the slope of (\ref{eq:LSL_reachability_revisit}) varies when $\alpha$ changes, while the point $(p_{LSL}^k, q_{LSL}^k)$ always lies on  (\ref{eq:LSL_reachability_revisit}) for all  rotations. This indicates that $(p_{LSL}^k, q_{LSL}^k)$ is the center of rotation of (\ref{eq:LSL_reachability}). Moreover, for any given $\alpha$, one can determine the signs of $a(\alpha)$ and $c(\alpha)$, and the corresponding inequality constraint in (\ref{eq:LSL_reachability}), which in turn determines the quadrant of the coordinate system with center at $(p_{LSL}^k, q_{LSL}^k)$,  within which (\ref{eq:LSL_reachability}) falls in. Thus (\ref{eq:LSL_reachability}) represents a ray starting from the center $(p_{LSL}^k, q_{LSL}^k)$.

Now, we show that as $\alpha$ increases from $\alpha_{inf}^k$ to $\alpha^k_{sup}$, (\ref{eq:LSL_reachability}) rotates in the anticlockwise manner. To see this, denote the slope of (\ref{eq:LSL_reachability}) as $S_{LSL}(\alpha)=\frac{a(\alpha)}{c(\alpha)},  c(\alpha) \neq 0$. Note that $S_{LSL}(\alpha)$ is a continuous function of $\alpha$.

Taking the first-order derivative of $S_{LSL}(\alpha)$, we get

\begin{equation}
    \frac{\partial{S_{LSL}(\alpha)}}{\partial{\alpha}}=\frac{1+v_w\cos(\alpha-\theta_w)}{(\cos\alpha+v_w\cos\theta_w)^2}.
\end{equation}

Since, $v_w<1$ and $\cos(\alpha-\theta_w)\in[-1,1]$, we get $\frac{\partial{S_{LSL}(\alpha)}}{\partial{\alpha}} > 0$. Thus, as $\alpha$ grows, (\ref{eq:LSL_reachability}) rotates in the anticlockwise manner.

For $RSR$ path type,  (\ref{eq:RSR_reachability}) can be re-written as

\vspace{-3pt}
\begin{equation}\label{eq:RSR_reachability_revisit}
    b(\alpha) \cdot ( x_f - p_{RSR}^k \big) + c(\alpha) \cdot \big( y_f - q_{RSR}^k ) = 0.
\end{equation}

Thus, the point $(p_{RSR}^k, q_{RSR}^k)$ always lies on  (\ref{eq:RSR_reachability_revisit}) for all rotations. This indicates that $(p_{RSR}^k, q_{RSR}^k)$ is the center of rotation of (\ref{eq:RSR_reachability}). Moreover, for any given $\alpha$, one can determine the signs of $b(\alpha)$ and $c(\alpha)$, and the corresponding inequality constraint in (\ref{eq:RSR_reachability}), which in turn determines the quadrant of the coordinate system with center at $(p_{RSR}^k, q_{RSR}^k)$,  within which (\ref{eq:RSR_reachability}) falls in. This implies that (\ref{eq:RSR_reachability}) represents a ray starting from the center $(p_{RSR}^k, q_{RSR}^k)$.

Now, we show that as $\alpha$ increases from $\alpha_{inf}^k$ to $\alpha^k_{sup}$, (\ref{eq:RSR_reachability}) rotates in the clockwise manner. To see this, denote the slope of (\ref{eq:RSR_reachability}) as $S_{RSR}(\alpha) = -\frac{b(\alpha)}{c(\alpha)}, c(\alpha) \neq 0$. Note that $S_{RSR}(\alpha)$ is a continuous function of $\alpha$.

Taking the first-order derivative of $S_{RSR}(\alpha)$, we get

\begin{equation}
    \frac{\partial{S_{RSR}(\alpha)}}{\partial{\alpha}}=-\frac{1+ v_w\cos(\alpha+\theta_w)}{(\cos\alpha+v_w\cos\theta_w)^2}.
\end{equation}

Since $v_w < 1$ and $\cos(\alpha + \theta_w) \in [-1, 1]$, we get $\frac{\partial{S_{RSR}(\alpha)}}{\partial{\alpha}} < 0$. Thus, as $\alpha$ grows, (\ref{eq:RSR_reachability}) rotates in the clockwise manner.
\end{proof}

\vspace{6pt}

\subsubsection{\textbf{Proof of Lemma~\ref{lem:parallelslope}}}
\label{proof:parallelslope}
\begin{proof}
 First, consider $2\pi$-arc $LSL$ paths. From Table~\ref{table:range_conventional}, for $k = 0$: $\alpha_{inf} = 0$ and $\alpha_{sup} = \theta_f$; while for $k = 1$: $\alpha_{inf} = \theta_f$ and $\alpha_{sup} = 2\pi$. Then, using (\ref{eq:LSL_slope}) we get
\vspace{-3pt}
\begin{subequations}
\begin{align}
\omega_{LSL}^0(\alpha_{inf}) &= \omega_{LSL}^1(\alpha_{sup}) \nonumber \\
&= \atantwo(w_y,1+w_x) \Mod{2\pi}, \label{Lemma2:eq1}\\
\omega_{LSL}^0(\alpha_{sup}) &= \omega_{LSL}^1(\alpha_{inf}) \nonumber \\
&= \atantwo(\sin{\theta_f} + w_y, \cos{\theta_f} + w_x) \Mod{2\pi}. \label{Lemma2:eq2}
\end{align}
\end{subequations}

Now, consider $2\pi$-arc $RSR$ paths. From Table~\ref{table:range_conventional}, for $k = -1$: $\alpha_{inf} = 0$ and $\alpha_{sup} = 2\pi-\theta_f$; while for $k = -2$: $\alpha_{inf} = 2\pi-\theta_f$ and $\alpha_{sup} = 2\pi$. Then, using (\ref{eq:RSR_slope}) and  we get
\vspace{-3pt}
\begin{subequations}
\begin{align}
\omega_{RSR}^{-1}(\alpha_{inf}) &= \omega_{RSR}^{-2}(\alpha_{sup}) \nonumber \\
&= \atantwo(w_y,1+w_x) \Mod{2\pi}, \label{Lemma2:eq3}\\
\omega_{RSR}^{-1}(\alpha_{sup}) &= \omega_{RSR}^{-2}(\alpha_{inf}) \nonumber \\
&= \atantwo(\sin{\theta_f} + w_y, \cos{\theta_f} + w_x) \Mod{2\pi}. \label{Lemma2:eq4}
\end{align}
\end{subequations}

Therefore, from (\ref{Lemma2:eq1}) and (\ref{Lemma2:eq3}) we get:

$\omega_{LSL}^0(\alpha_{inf}) = \omega_{LSL}^1(\alpha_{sup}) = \omega_{RSR}^{-1}(\alpha_{inf}) = \omega_{RSR}^{-2}(\alpha_{sup})$.

\vspace{3pt}
And from (\ref{Lemma2:eq2}) and (\ref{Lemma2:eq4}) we get:

$\omega_{LSL}^0(\alpha_{sup}) = \omega_{LSL}^1(\alpha_{inf}) =  \omega_{RSR}^{-1}(\alpha_{sup})= \omega_{RSR}^{-2}(\alpha_{inf})$.
\end{proof}

\vspace{10pt}
\subsubsection{\textbf{Proof of Lemma~\ref{lem:slopeproperty2}}}
\label{proof:slopeproperty2}
\begin{proof}
From Lemma~\ref{lem:parallelslope}, we get:

\begin{itemize}
\item $\omega_{LSL}^0(\alpha_{inf}) = \omega_{RSR}^{-1}(\alpha_{inf})$ and
\item $\omega_{LSL}^0(\alpha_{sup}) = \omega_{RSR}^{-1}(\alpha_{sup})$.
\end{itemize}

Thus, the  rotations of the two boundaries of the region spanned by $k=0$ (i.e., ($\omega_{LSL}^0(\alpha_{inf})$ and $\omega_{LSL}^0(\alpha_{sup})$)) are the same as the rotations of the corresponding boundaries of the region spanned by $k=-1$ (i.e., $\omega_{RSR}^{-1}(\alpha_{inf})$ and $\omega_{RSR}^{-1}(\alpha_{sup})$), respectively. Therefore, the acute angles between the boundaries corresponding to $k=0$ and $k=1$ are the same. Note that the centres of these two regions could be different. However, from Lemma~\ref{lem:swipe}, the swiping direction for $k=0$ and $k=-1$ are opposite. Thus, if $k = 0$ forms $LSL$ MaRA (MiRA), then $k = -1$ forms $RSR$ MiRA (MaRA). This proves a). The proof of  b) follows similar logic and is omitted here.
\end{proof}

\subsubsection{\textbf{Proof of Lemma~\ref{lem:slopeproperty3}}}
\label{proof:slopeproperty3}
\begin{proof}
a) Consider $LSL$ MaRA formed by $k = 0$. From Table~\ref{table:range_conventional}, we have the feasible range of $\alpha$ as $[0,\theta_f]$, where $\theta_f \in [0,2\pi)$. Since $\omega^{0}_{LSL}(0) = \omega^{0}_{LSL}(2\pi) = \atantwo \big(w_y,1+w_x\big)$, then given any $\theta_f <2\pi$, $\omega^{0}_{LSL}(\alpha)$ cannot make a full rotation as $\alpha$ varies from $0$ to $\theta_f$. Thus, for $k=0$, $LSL$ MaRA cannot provide full reachability. Similarly, we can show that the MaRAs formed by $k=1,-1$ and $-2$ cannot provide full reachability.

b) First, we show that either $LSL$ MaRA completely covers the $RSR$ MaRA (i.e., $RSR$ MaRA is a subset of $LSL$ MaRA), or $RSR$ MaRA completely covers the $LSL$ MaRA (i.e., $LSL$ MaRA is a subset of $RSR$ MaRA).

Suppose $LSL$ MaRA is formed by $k=0$ (hence $LSL$ MiRA is formed by $k = 1$). By Lemma~\ref{lem:slopeproperty2}, $RSR$ MaRA is formed by $k=-2$. According to Lemma~\ref{lem:parallelslope}, $\omega^0_{LSL}(\alpha_{inf})= \omega_{RSR}^{-2}(\alpha_{sup})$ and $\omega^0_{LSL}(\alpha_{sup})= \omega_{RSR}^{-2}(\alpha_{inf})$. This implies that the boundaries of $LSL$ MaRA and $RSR$ MaRA are parallel to each other, and that they form the same acute angle, as shown in Fig. \ref{fig:reachability_condition_MARA}. Thus, $LSL$ MaRA can completely cover $RSR$ MaRA if the center $(p^{-2}_{RSR},q^{-2}_{RSR})$ falls inside the shadow region in Fig. \ref{fig:reachability_condition_MARA}, which is formed by the boundaries with angles $\overline{\omega^0_{LSL}}(\alpha_{sup})$ and $\overline{\omega^0_{LSL}}(\alpha_{inf})$. Similarly, we can also determine the other condition when $LSL$ MaRA is formed by $k=1$ and $RSR$ MaRA is formed by $k=-1$. Subsequently, we checked the trueness of both conditions for the full range of $\theta_f$ and $\theta_w$ from $0$ to $2\pi$, and the results are presented in Fig.~\ref{fig:reachability_MARA}. It is seen that for any given pair of $\theta_f$ and $\theta_w$, one of the above conditions is always true. Thus, either $LSL$ MaRA completely covers the $RSR$ MaRA or $RSR$ MaRA completely covers the $LSL$ MaRA. This indicates that the union of both MaRAs equals to the larger MaRA, then following part a) above, this in turn implies that their union cannot provide full reachability.

\begin{figure}[!t]
    \centering
    \subfloat[Illustration of the condition required for $2\pi$-arc $LSL$ MaRA with $k=0$ to completely cover the $2\pi$-arc $RSR$ MaRA with $k=-2$.]{
        \includegraphics[width=0.226\textwidth]{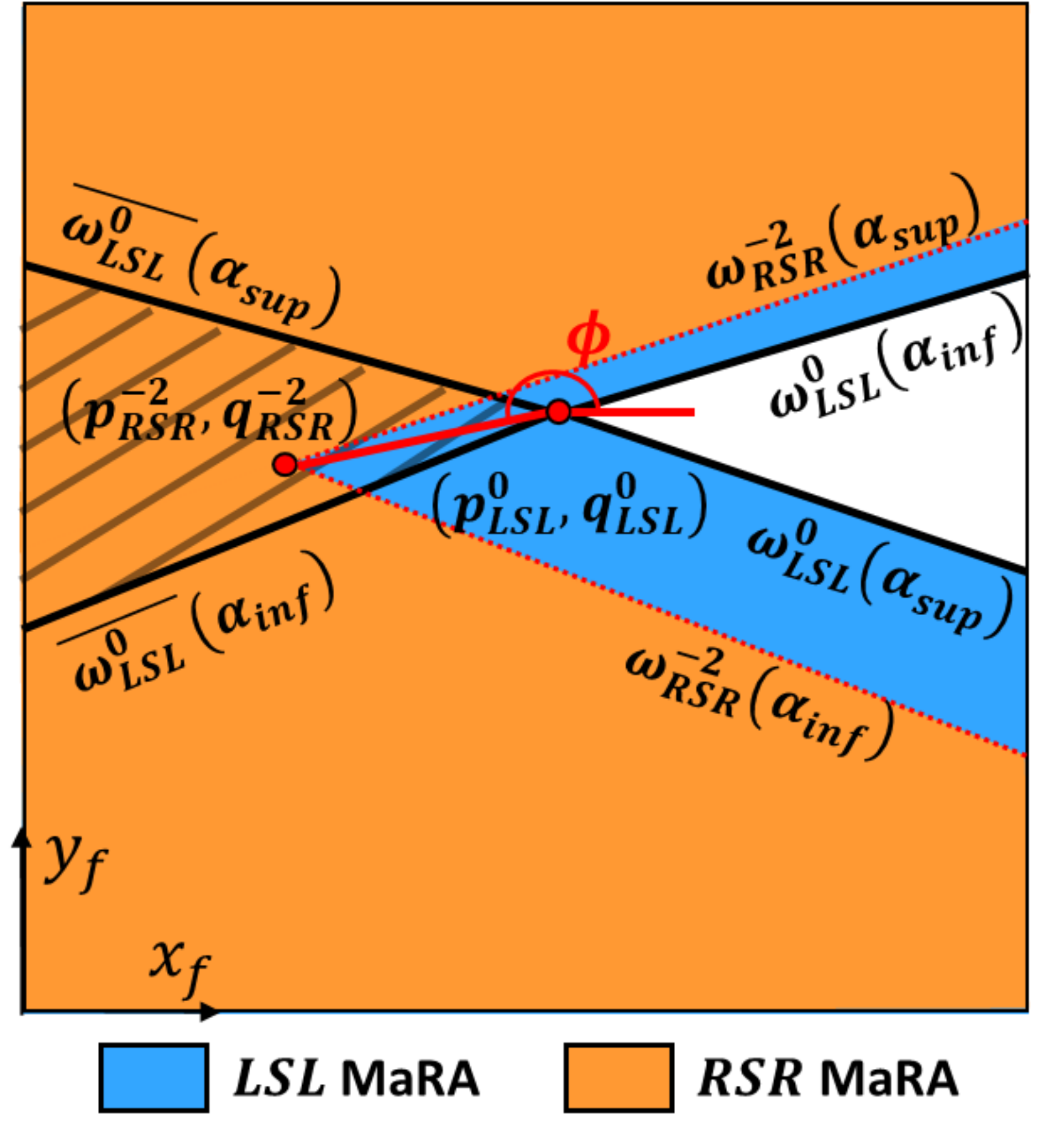}\label{fig:reachability_condition_MARA}} \hspace{-3pt}
    \subfloat[Numerical validation of the fact that either $2\pi$-arc $LSL$ MaRA covers $2\pi$-arc $RSR$ MaRA or vice versa over the full range of $\theta_f$ and $\theta_w$.]{
        \includegraphics[width=0.245\textwidth]{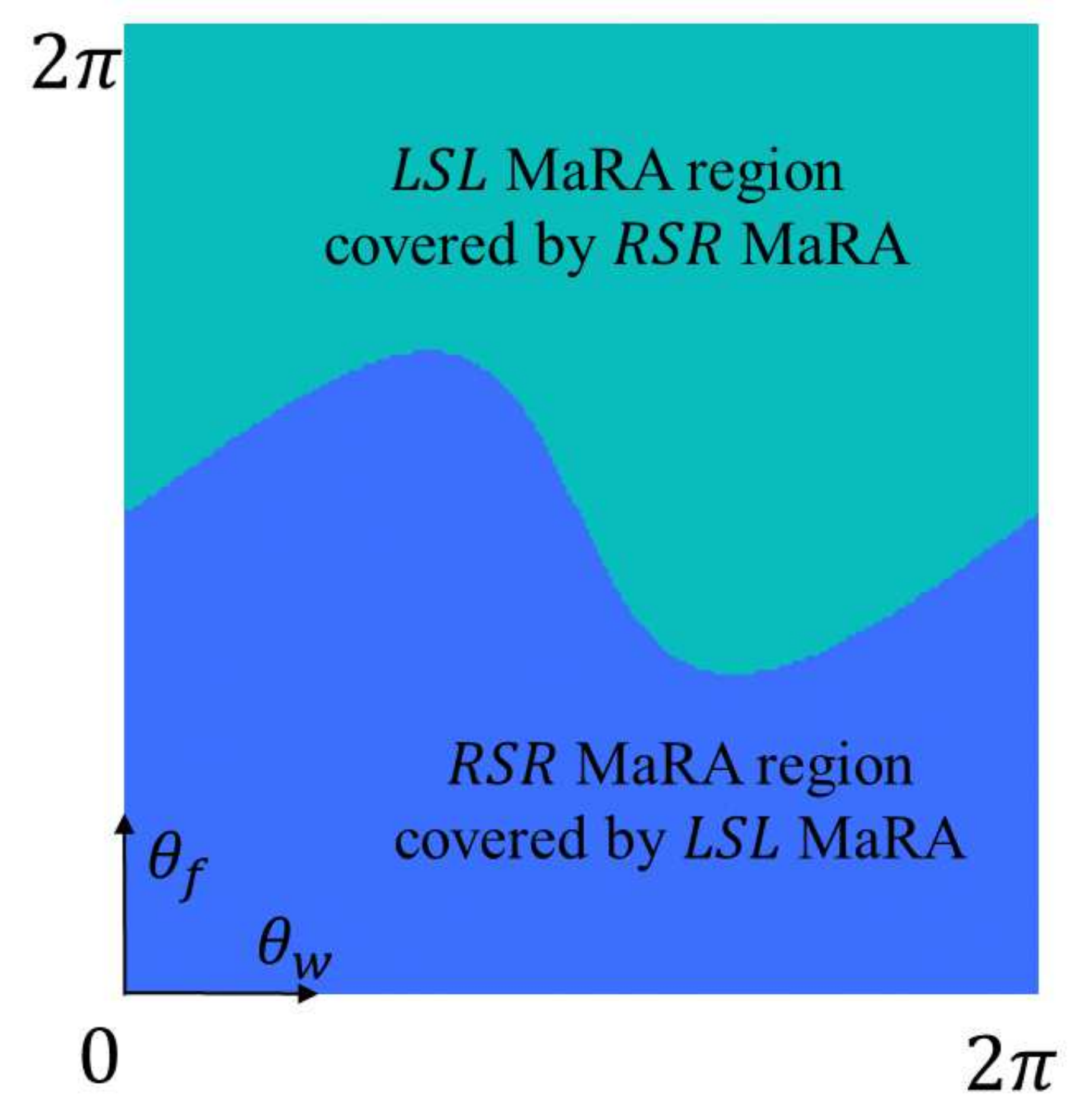}\label{fig:reachability_MARA}}
        \caption{Proof of Lemma~\ref{lem:slopeproperty3}b).}
     \label{fig:prooflemma4} \vspace{-6pt}
\end{figure}

c) According to part b) above, either $LSL$ MaRA completely covers the $RSR$ MaRA or $RSR$ MaRA completely covers the $LSL$ MaRA. First, suppose that $LSL$ MaRA is the larger of the two and covers the $RSR$ MaRA. If the condition of Case 1 is not satisfied, then the union of $LSL$ MaRA and $LSL$ MiRA cannot provide full reachability and there exists some region that is unreachable, say $\mathcal{R}_1$ (e.g., see the white region in Fig.~\ref{fig:same_LSL}). Thus, the center of $LSL$ MiRA is not in the shadow region formed by the $\pi$ rotations of $LSL$ MaRA boundaries. Also, if the condition of Case 3 is not satisfied, then the union of $LSL$ MaRA and $RSR$ MiRA cannot provide full reachability and there exists some region that is unreachable, say $\mathcal{R}_2$ (e.g., see the white region in Fig.~\ref{fig:different_LSL}). Thus, the center of $RSR$ MiRA is not in the shadow region formed by the $\pi$ rotations of $LSL$ MaRA boundaries. Since by Lemma~\ref{lem:parallelslope} the boundaries of $LSL$ MaRA, $LSL$ MiRA and $RSR$ MiRA are parallel to each other, as long as the centers of $LSL$ MiRA and $RSR$ MiRA are outside the shadow region of $LSL$ MaRA, there is no way they can together cover the reachability gaps $\mathcal{R}_1$ and $\mathcal{R}_2$ completely. Thus, in this case, because $RSR$ MaRA is a subset of $LSL$ MaRA, the union of $LSL$ MaRA, $LSL$ MiRA, $RSR$ MaRA and $RSR$ MiRA cannot provide full reachability.

Using a similar logic, when $RSR$ MaRA is the larger MaRA, one can show that if the conditions of Case 2 and Case 4 are not satisfied, then the union of $LSL$ MaRA, $LSL$ MiRA, $RSR$ MaRA and $RSR$ MiRA cannot provide full reachability.

\end{proof}

\vspace{-12pt}
\balance
\bibliographystyle{IEEEtran}
\bibliography{Dubins_Wind}

\end{document}